\documentclass[12pt]{article}
\usepackage{amsmath, amssymb, amsthm, amsfonts}
\usepackage{graphicx,psfrag,epsf, subfigure}
\usepackage{enumerate}
\usepackage{natbib}
\usepackage{fullpage}
\usepackage{url} 
\usepackage[ruled,vlined]{algorithm2e}
\usepackage[colorlinks=true, linkcolor=blue, urlcolor=blue, citecolor=blue]{hyperref}
\newcommand{\blind}{1}
\newcommand\blue{\textcolor{blue}}

\def\P{\mathbb{P}}
\def\I{\mathbf{I}}

\def\a{\alpha}
\def\R{\mathbb{R}}

\def\E{\mathbb{E}}
\def\N{\mathbb{N}}
\let\hat\widehat
\let\tilde\widetilde

\newtheorem{assumption}{Assumption}

\newtheorem{theorem}{Theorem}[section]
\newtheorem{proposition}[theorem]{Proposition}
\newtheorem{lemma}[theorem]{Lemma}
\newtheorem{corollary}[theorem]{Corollary}
\newtheorem{remark}{Remark}[section]

\begin{document}

\def\spacingset#1{\renewcommand{\baselinestretch}%
{#1}\small\normalsize} \spacingset{1}


\if1\blind
{
  \title{\bf Online Covariance Matrix Estimation in Stochastic Gradient Descent}
  \author{Wanrong Zhu\\
    Department of Statistics, University of Chicago\\
    and \\
    Xi Chen  \\
    Leonard N. Stern School of Business, New York University\\
and\\
Wei Biao Wu  \\
 Department of Statistics, University of Chicago\\
}
 \date{}
  \maketitle
} \fi

\if0\blind
{
  \bigskip
  \bigskip
  \bigskip
  \begin{center}
    {\LARGE\bf Online Covariance Matrix Estimation in Stochastic Gradient Descent\par} 
  \end{center}
  \medskip
} \fi

\bigskip
\begin{abstract}
	The stochastic gradient descent (SGD) algorithm is widely used for parameter estimation, especially for huge data sets and  online learning. While this recursive algorithm is popular for computation and memory efficiency,  quantifying variability and randomness of the solutions has been rarely studied. This paper aims at conducting statistical inference of SGD-based estimates in an online setting. In particular, we propose a fully online estimator for the covariance matrix of averaged SGD iterates (ASGD) only using the iterates from SGD. We formally establish our online estimator's consistency and show that the convergence rate is comparable to offline counterparts. Based on the classic asymptotic normality results of ASGD, we construct asymptotically valid confidence intervals for model parameters. Upon receiving new observations, we can quickly update the covariance matrix estimate and the confidence intervals.  This approach fits in an online setting and takes full advantage of SGD: efficiency in computation and memory. 
\end{abstract}

\noindent%


\spacingset{1.5} 
\section{Introduction}
\label{sec:intro}
Model parameter estimation through optimization of an objective function is a fundamental problem in statistics and machine learning. Here we consider the classic setting where the true model parameter $x^{*}\in \R^{d}$ can be characterized as the minimizer of a convex objective function $F: \R^{d}\rightarrow \R$, i.e.,
\begin{equation}\label{eq:obj}
x^{*} = \underset{x\in \R^{d}}{\arg\min} F(x).
\end{equation}
The objective function $F(x)$ is defined as $F(x) = \E_{\xi \sim \Pi}{f(x, \xi)}$, where $f(x, \xi)$ is a noisy measurement of $F(x)$ and $\xi$ is a random variable following the distribution $\Pi$. 

In recent years, huge data sets and streaming data arise frequently. Classic deterministic optimization methods that require storing all the data are not appealing due to expensive memory cost and computational inefficiency. To resolve these issues, one can apply the Robbins-Monro algorithm \citep{robbins1951stochastic, kiefer1952stochastic},  also known as Stochastic Gradient Descent (SGD), especially for online learning \citep{bottou1998online, mairal2010online, hoffman2010online}. Setting $x_{0}$ as the initial point, the $i$-th iteration of the SGD algorithm takes the following form
\begin{equation}\label{eq:SGDite}
x_{i} = x_{i-1} - \eta_{i}\nabla{f(x_{i-1}, \xi_{i})},\ i\ge 1,
\end{equation}
where $\{\xi_{i}\}_{i\ge 1}$ is a sequence of \emph{i.i.d} sample from the distribution $\Pi$, $\nabla{f}$ is the gradient of $f(x, \xi)$ with respect to the first argument $x$, and $\eta_{i}$ is the step size at the $i$-th step. This recursive adaptive algorithm performs one update at a time and does not need to remember outcomes in previous iterations. Therefore, it is computationally efficient,  memory friendly, and able to process data on the fly. 

Despite these advantages, SGD performs frequent updates with high variability, and the outcomes can fluctuate heavily. The crucial problem is to understand the variability and randomness of the solutions.
In this paper, we address the uncertainty quantification problem in the \emph{online setting} where data can arrive sequentially.
In particular, we propose a fully online approach to estimate the covariance matrix of SGD-based estimates only using the iterates from SGD. The efficient algorithm we propose is recursive. It performs an immediate update of the covariance estimate as new data arrives, which follows the spirit of SGD. 
We can then conduct statistical inference with the estimated covariance matrix and construct confidence intervals for model parameters in a fully online fashion.
  
Before discussing our method, we provide a brief review of the literature on SGD. The asymptotic convergence of SGD iterates has been studied extensively in the early years \citep{blum1954approximation, dvoretzky1956, robbins1971convergence, ljung1977analysis, sacks1958, fabian1968asymptotic, lai2003stochastic}. To further investigate the asymptotic distribution of SGD, \citet{polyak1992acceleration} and \citet{ruppert1988efficient}  introduced the averaged SGD (ASGD), a simple modification where iterates are averaged, and established the asymptotic normality of the obtained estimate. Moreover, it is known that ASGD estimates achieve the optimal central limit theorem rate $\mathcal{O}_P(1/\sqrt n)$ by running SGD for $n$ iterations under certain regularity conditions. For linear stochastic approximation, \citet{mou2020linear} modified the Polyak-Ruppert covariance with an additional correction term concerning the constant step size. Differently from the SGD algorithm, \citet{toulis2017asymptotic} introduced implicit SGD procedures and analyzed the asymptotic distribution of the averaged implicit SGD iterates.  Convergence in non-asymptotic fashion has also been studied recently for SGD and its variants with different objective functions 
 \citep{rakhlin2011making, moulines2011non, hazan2014beyond,
	bach2013non, 
	duchi2011adaptive, 
	kingma2014adam, shamir2013stochastic}. Our method and analysis rely on the averaged SGD and its asymptotic normality in later discussions. 

In addition to convergence and error bounds of SGD-based estimators, statistical inference problems based on SGD have recently started to gain more attention. Instead of focusing on point estimators, one is interested in assessing the uncertainty of the estimates through their confidence intervals/regions. \citet{chen2016statistical}  introduced the inference problem and proposed a batch-means method to construct asymptotically valid confidence intervals based on asymptotic normality of ASGD. \citet{fang2018online} and \citet{fang2019scalable} proposed bootstrap procedures for constructing confidence intervals through the perturbed-SGD. Meanwhile, variants of the SGD algorithm and corresponding inference in non-asymptotic fashion are studied in \citet{su2018uncertainty} and  \citet{liang2019statistical}. For online $l_{1}$ penalized problems, \citet{chao2019generalization} proposed a class of generalized regularized dual averaging and made uncertainty quantification possible for online sparse algorithms.

\subsection{Problem formulation} Our work in this paper is applicable to vanilla SGD, which is most widely used in practice. We use the ASGD iterate $$\bar{x}_{n} = n^{-1}\sum_{i=1}^{n}x_{i}$$ as the estimate for the model parameter at the $n$-th step. We set step size $\eta_{i} = \eta i^{-\alpha} (i\ge 1$) with $\eta>0$ and $\alpha\in(0.5, 1)$ as suggested by \citet{polyak1992acceleration}. Define 
\begin{equation}\label{sandwich}
A = \nabla^{2}F(x^{*}), \ S = \E\left([\nabla f(x^{*}, \xi)][\nabla f(x^{*}, \xi)]^{T}\right).
\end{equation}
From \citet{polyak1992acceleration}, under suitable conditions, $\bar{x}_{n}$ has the asymptotic normality:
\begin{equation}\label{eq:asym_norm}
\sqrt{n}(\bar{x}_{n} - x^{*}) \Rightarrow  N(0, \Sigma),
\end{equation}
where $\Sigma = A^{-1}SA^{-1}$,  which is known as the ``sandwich'' form of the covariance matrix. To leverage the asymptotic normality result for inference, it is critical to estimate the limiting covariance matrix $\Sigma$. Intuitively, one can estimate $S$ with a simple sample average $\hat S_{n} = n^{-1}\sum_{i=1}^{n}[\nabla f(x_{i-1}, \xi_{i})][\nabla f(x_{i-1}, \xi_{i})]^{T}$, and similarly estimate $A$ with $\hat A_{n} = n^{-1}\sum_{i=1}^{n}\nabla^{2}f(x_{i-1}, \xi_{i})$. Then the limiting covariance matrix $\Sigma$ can be estimated by the consistent plug-in estimator $\hat A_{n}^{-1}\hat S_{n}\hat A_{n}^{-1}$ (see \cite{chen2016statistical}). However, computation of the Hessian matrix of the loss function is not always available, e.g., 
	certain computations are not available in many existing codebases that only adopt SGD for optimization and in cases such as quantile regression, the Hessian matrix does not even exist. Also, the plug-in estimator may be computationally costly when $d$ is large since it involves matrix inversion with $O(d^3)$ time complexity in general. 

Our goal is to obtain an online estimate of the covariance matrix of $\sqrt{n}\bar{x}_{n}$, only through the SGD iterates $\{x_{1}, x_{2},..., x_{n}\}$. Our approach is attractive in situations where the computation for $A^{-1}$ and $S$ are difficult, which is quite typical in practice. Also, the approach is efficient in both computation and memory due to its recursive property, i.e., the estimate at $n$-th step $\hat\Sigma_{n}$ can be updated from $\hat\Sigma_{n-1}$ within $O(d^{2})$ computation. With the estimate, we can perform uncertainty quantification and statistical inference with desirable computation and memory efficiency. The approach is useful for online learning, where the data is constantly arriving over time, such as streaming data. 


For the time-homogeneous Markov chain, $\{x_{i}\}_{i\in\mathbb{Z}}$ is a stationary process. Under certain short-range dependence conditions, we have
$$
\sqrt{n}\left(\bar{x}_{n} - \E x_{i}\right)\Rightarrow N(0, \sigma^{2}),
$$
where
 $$\sigma^{2} = \lim_{n\rightarrow\infty}\text{Var}(\sqrt{n}\bar{x}_{n}) = \sum_{i=-\infty}^\infty {\rm cov}(x_0, x_i)$$ is the long-run variance, and it plays a fundamental role in the statistical inference of stationary processes. To estimate the long-run variance, one can apply the batch-means method \citep{Glynn:91, Flegal:10,subsampling, resampling, kitamura1997empirical}. Given $x_{1}, ..., x_{n}$, let $1\le l_{n}\le n$ be the batch size. Based on batch-means $\sum_{k=i}^{i+l_{n}}x_{k}/l_{n} - \bar{x}_{n}$ for $1\le i\le n - l_{n} + 1$, one can estimate $\sigma^{2}$ by 
$$
\sigma_{n}^{2} = \frac{l_{n}}{n - l_{n} + 1}\sum_{i = 1}^{n-l_{n}+1}\left(\sum_{k=i}^{i+l_{n}-1}x_{k}/l_{n} - \bar{x}_{n}\right)^{2}.
$$
As an alternative, one can use the non-overlapping batch-means $\sum_{k=i}^{i+l_{n}}x_{k}/l_{n} - \bar{x}_{n}$ for $i = 1, 1 + l_{n}, 1+2l_{n}, ...$, to construct a similar estimate.
Properties of overlapping and non-overlapping batch-means estimators are discussed in \citet{subsampling} and \citet{resampling}.
In our problem, estimation of $\Sigma$ in \eqref{eq:asym_norm} becomes more complicated since SGD iterates form a non-stationary Markov Chain. 

To apply to SGD, \citet{chen2016statistical} modified the classic non-overlapping batch-means by allowing increasing batch sizes and showed that the modified batch-means estimator is consistent.
 However, their approach is not in line along with the spirit of SGD, the fully online fashion. Their construction of covariance estimator $\hat\Sigma_{n}$ requires the information on the total number of iterations $n$ a priori. There is no simple algebraic relation between $\hat\Sigma_{n}$ and $\hat\Sigma_{n+1}$. In other words, when a new data point $x_{n+1}$ arrives later, their algorithm needs to re-compute their estimate from the beginning and cannot perform efficient sequential updating. So the approach is computationally expensive for online learning, where the dynamic training data is arriving over time, and the goal is to make sequential predictions; see Remark \ref{rm: BM} for a detailed discussion of \citet{chen2016statistical}. 

To address the above problems, we develop in this paper a fully \emph{online approach} for asymptotic covariance matrix estimation,  which we refer to as online batch means method. The construction does not require prior knowledge of the total sample size. Immediate updates from $\hat\Sigma_{n}$ to $\hat\Sigma_{n+1}$ can be performed recursively as new data is coming in, which fits our online setting. To achieve this goal, we design a novel construction of batches with time-varying size, which substantially extends the one in \citet{chen2016statistical}.  Similar to the recursive nature of SGD, our algorithm is also recursive and it updates the covariance matrix estimate once at a time only through the stochastic gradient 
within $O(d^2)$ computation. Note that since we are learning a $d\times d$ covariance matrix,  it requires at least $O(d^2)$ computation to update the covariance matrix estimates. In the important special case of marginal inference of each coordinate of the parameter vector, our online batch means estimator only needs to compute and store diagonals of the covariance matrix estimate, which only require $O(d)$ computation and $O(d)$ memory. The idea of online estimation is motivated by \citet{wu2009recursive},  who studied the estimation of long-run variances of stationary and ergodic processes. As mentioned above, the SGD iterates in (\ref{eq:SGDite}) form a non-homogeneous (non-stationary)  Markov Chain since the step size $\eta_{k}$ decays as $k$ increases, for example $\eta_{k} = \eta k^{-\a}$ for $\a \in (1/2, 1)$ as suggested by \citet{polyak1992acceleration}. Hence, the asymptotic behaviors of SGD and stationary processes are fundamentally different. The construction, which is associated with batch sizes, is novel and different for SGD iterates and stationary sequences. This non-stationarity also brings substantial difficulties in technical analysis. The convergence of our estimator is far from being trivial. We formally establish the consistency result and obtain the convergence rate of our online estimator in Section \ref{sec:theory}.

We summarize our contributions as follows. We propose a fully online approach to estimate the asymptotic covariance matrix of the ASGD solution and conduct statistical inference. 
 The fully online fashion allows efficient sequentially updating. It is important for online learning, where data comes in a stream and real-time update of predictions is needed before seeing future data. It has potential applications such as online advertisement placement and online web ranking \citep{richardson2007predicting, zhang2016bid}.  Our method is efficient in both computation and memory. In particular, the computational and memory complexity at the update step is $O(d^2)$, and the total computational cost only scales linearly in $n$. 
In terms of theoretical merits, the proposed estimator is the first fully online fashion estimator with rigorous convergence property for asymptotic covariance of ASGD. We show that the convergence rate of our online estimator is comparable to the offline counterparts.

\subsection{Organization and notation}
The remainder of this paper is organized as follows. In Section \ref{sec:meth},  we propose the online estimator (two versions) for the asymptotic covariance matrix of ASGD iterates and corresponding algorithms. In Section \ref{sec:theory}, we show that the online estimator is consistent and obtains the desired convergence rate. Also, confidence intervals/regions based on our online estimator are constructed for statistical inference.  Section \ref{sec:exper} provides a simulation study to demonstrate the convergence rate of the online estimator and the asymptotically valid coverage of the confidence intervals.  Further discussion and future work are presented in Section \ref{sec:futu}.

Through out the paper, for a vector $\mathbf{a}=(a_{1},...,a_{d})\in \mathbb{R}^{d}$, $\|\mathbf{a}\|_{2}$ is defined as the vector $l_{2}$ norm $\|\mathbf{a}\|_{2}=\left(\sum_{i=1}^{d}a_{i}^{2}\right)^{1/2}$. For a matrix $\mathbf{A}=(a_{ij})\in \mathbb{R}^{d\times d}$,  we use $\|\mathbf{A}\|_{F}$ to denote its Frobenius norm $\|\mathbf{A}\|_{F}=\left(\sum_{i=1}^{d}\sum_{j=1}^{d}a_{ij}^{2}\right)^{1/2}$, and $\|\mathbf{A}\|_{2}$ to denote its operator norm $\|\mathbf{A}\|_{2}=\max_{\|x\|_{2}\le1}\|\mathbf{A}x\|_{2}$. When $\mathbf{A}$ is positive semi-definite, $\lambda_{A}$ denotes the largest eigenvalue of $\mathbf{A}$ and $tr(\mathbf{A})$ denotes its trace.  We use $\mathbf{I}_{d}$ to denote a $d\times d$ identity matrix. For positive sequences $\left\{a_{n}\right\}_{n\in\N}$ and $\left\{b_{n}\right\}_{n\in\N}$, $a_{n}\lesssim b_{n}$ means there exists some constant $C$ such that $a_{n}\le Cb_{n}$ for all large $n$. And $a_{n}\asymp b_{n}$ if both $a_{n}\lesssim b_{n}$ and $b_{n}\lesssim a_{n}$ hold. For $t\in \R$, $\left\lfloor t \right\rfloor$ is the largest integer less than or equal to $t$. 
For notational simplicity, we use notation $C$ for constants which can take different values in different equations.
We define conditional expectation $\E_{n}(\cdot) = \E(\cdot|\mathcal{F}_{n})$, where $\mathcal{F}_{n}$ is $\sigma$-algebra generated by $\{\xi_{i}\}_{i\le n}$.  Moreover, we use $\Rightarrow$ to denote convergence in distribution.  

\section{Online Approach}
\label{sec:meth} 
 We first introduce a time varying batch scheme used in our online approach. Consider infinite sequentially arriving SGD iterates $\{x_{i}\}_{i= 1, 2, ...}$ in \eqref{eq:SGDite}. Let $\{a_{m}\}_{m\in \N}$ be a strictly increasing integer-valued sequence with $a_{1} = 1$.   For the $i$-th iterate $x_{i}$, we consider a data block $B_{i}$ including iterates from past iterations $t_i$ to $i$, i.e., $$B_{i} = \{x_{t_{i}}, ..., x_{i}\},$$
where $t_{i}$ is the index of iterate we trace back to at the $i$-th step. The value of $t_{i}$ is determined by the sequence $\{a_{m}\}_{m\in \N}$ through $t_{i} = a_{m}$ when $i\in [a_{m}, a_{m+1})$. For example, $t_{i} = \left\lfloor \sqrt{i} \right\rfloor^{2}$ if $a_{m} = m^{2}$.  In this case we have:\\ $B_{1} = \{x_{1}\}$, $B_{2} = \{x_{1}, x_{2}\}$, $B_{3} = \{x_{1}, x_{2}, x_{3}\}$, \\
$B_{4} = \{x_{4}\}$, $B_{5} = \{x_{4}, x_{5}\}$, $B_{6} = \{x_{4}, x_{5}, x_{6}\}$, $B_{7} =\{x_{4}, x_{5}, x_{6}, x_{7}\} $, $B_{8} =\{x_{4}, x_{5}, x_{6}, x_{7}, x_{8}\}, $  \\
$B_{9} = \{x_{9}\}$, $B_{10} = \{x_{9}, x_{10}\}$, $B_{11} = \{x_{9}, x_{10}, x_{11}\}, ...$ .\\
We can see that the batch sizes are time-varying. The blocks $\{B_{i}: {a_{m}\le i < a_{m+1}}\}$ can also be viewed as the so-called “forward scans” in block subsampling  \citep{mcelroy2007computer, nordman2013}. That is, given non-overlapping blocks $\{x_{a_{m}},....,  x_{a_{m+1}-1}\}$, the forward scans are overlapping blocks of sequentially increasing length starting from $x_{a_{m}}$.

\subsection{Online covariance matrix estimator based on batch means}
Based on blocks $\{B_{i}\}_{i\in \N}$, the covariance matrix estimator is defined as the sum of squared block sums (centered) divided by the sum of block lengths, i.e., at the $n$-th step
\begin{equation}\label{est}  
\hat{\Sigma}_{n} = \frac{\sum_{i=1}^{n}\left(\sum_{k=t_{i}}^{i}x_{k}-l_{i}\bar{x}_{n}\right)\left(\sum_{k=t_{i}}^{i}x_{k}-l_{i}\bar{x}_{n}\right)^{T}}{\sum_{i=1}^{n}l_{i}},
\end{equation}
where $l_{i} =|B_{i}|= i - t_{i} + 1$ denotes the length of $B_{i}$.

The novel idea of constructing data block $B_{i}$, which only includes past iterates, is the key to make the algorithm fully online. 
Next, we will show that the estimate $\hat{\Sigma}_{n}$ can be computed recursively.
Let $W_{i}$ denote the sum of the block $B_{i} = \{x_{t_{i}}, ..., x_{i}\}$, i.e., 
\begin{equation}
W_{i} = \sum_{k=t_{i}}^{i}x_{k}.
\end{equation} 
 When 
 $t_{i+1}=t_{i} = a_{m}$ for some $m$,  $ B_{i+1} = B_{i}\cup\{x_{i+1}\}$ and 
$$W_{i+1} = W_{i} + x_{i+1}, \ l_{i+1} = l_{i}+1.$$ 
When 
$t_{i+1}=a_{m+1}$ for some $m$, we start a new block  $B_{i+1} = \{x_{i+1}\}$ and 
$$W_{i+1} = x_{i+1}, \ l_{i+1} = 1.$$ 
We can see that both the batch sum $W_{i}$ and the batch length $l_{i}$ can be updated recursively. With the notation of $W_{i}$, the estimator in (\ref{est}) can be expressed as 
\begin{equation}\label{est2} 
\hat{\Sigma}_{n} = \frac{\sum_{i=1}^{n}W_{i}W_{i}^{T} +
	\sum_{i=1}^{n}l_{i}^{2}\bar{x}_{n}\bar{x}_{n}^{T}  -  \left(\sum_{i=1}^{n}l_{i}W_{i}\right)\bar{x}_{n}^{T} - \bar{x}_{n}\left(\sum_{i=1}^{n}l_{i}W_{i}\right)_{n}^{T}}{\sum_{i=1}^{n}l_{i}}   . 
\end{equation}
To further simplify the form, we introduce
\begin{equation}
\begin{split}
&V_{n} = \sum_{i=1}^{n}W_{i}W_{i}^{T}\ ,  \ P_{n} = \sum_{i = 1}^{n}l_{i}W_{i}.\\
&v_{n}=\sum_{i=1}^{n}l_{i},\ \ \ \text{and}\ \ q_{n} = \sum_{i=1}^{n}l_{i}^{2}.
\end{split}
\end{equation} 
They can be computed recursively since both $W_{i}$ and $l_{i}$ can be updated recursively. Now, $\hat{\Sigma}_{n}$ in \eqref{est} can be finally rewritten as
\begin{equation}\label{eq:estrec}
\hat{\Sigma}_{n} = \frac{V_{n}+q_{n}\bar{x}_{n}\bar{x}_{n}^{T}-P_{n}\bar{x}_{n}^{T} - \bar{x}_{n}P_{n}^{T} }{v_{n}}.
\end{equation} 
All five components in \eqref{eq:estrec}: $V_{n}, q_{n}, P_{n}, v_{n}, \bar{x}_{n}$ can be updated recursively. Thus,  $\hat{\Sigma}_{n}$ can be updated through  results in the $(n-1)$-th step and the new iterate $x_{n}$ within $O(d^{2})$ computation. 

To summarize, we propose Algorithm \ref{algo:recursive}. As shown in Algorithm \ref{algo:recursive}, the five components of $\hat\Sigma_{n+1}$ can be easily updated from their values in the $n$-th step. There is no need to store all the outcomes in the previous steps. The memory complexity is $O(d^2)$, independent of the sample size $n$. In the update step, the computational complexity is also $O(d^2)$. The total computational cost scales linearly in $n$. The algorithm is much more efficient compared to non-recursive methods and naturally fits online learning scenarios.

\begin{algorithm}[!t] 
	\textbf{Input:} function $f(\cdot)$, parameter $(\alpha, \eta)$, step size $\eta_{i}=\eta i^{-\a}$ for $i\ge 1$, pre-defined sequence $\{a_{m}\}_{m\in N}$.\\ 
	\textbf{Initialize:} $m_{0}= l_{0} = 0 , v_{0} = P_{0} = q_{0}= V_{0} =  W_{0} = \bar{x}_{0} = 0, x_{0}$\;
	
	\textbf{For} $n$ = 0, 1, 2, 3, ...\\
	\hspace{4.5ex} 
	\textbf{Receive:} new data $\xi_{n+1}$\\ 
	\hspace{4.5ex} \textbf{Do the following update:}\\
	\hspace{8.5ex}1. $x_{n+1} = x_{n} - \eta_{n+1}\nabla{f(x_{n}, \xi_{n+1})}$;\\
	\hspace{8.5ex}2. $\bar{x}_{n+1} = (n\bar{x}_{n}+x_{n+1})/(n+1)$;\\ 
	\hspace{8.5ex}3. \textbf{if} $n+1 =a_{m_{n}+1}$, \textbf{then}: \\
		\hspace{10.5ex}$m_{n+1} = m_{n} + 1$;  $l_{n+1} =1$; $W_{n+1} = x_{n+1}$;\\ 
		\hspace{8.5ex}\  \textbf{else:} \\
		\hspace{10.5ex}  $m_{n+1} = m_{n}$; $l_{n+1} =  l_{n} + 1$;  $W_{n+1} = W_{n} +  x_{n+1}$;\\
	\hspace{8.5ex}4. $q_{n+1} = q_{n}+l_{n+1}^{2}$;\\
	\hspace{8.5ex}5. $v_{n+1} = v_{n}+l_{n+1}$;\\ 
	\hspace{8.5ex}6. $V_{n+1} = V_{n} + W_{n+1}W_{n+1}^{T}$;\\
	\hspace{8.5ex}7. $P_{n+1} = P_{n} + l_{n+1}W_{n+1}$;\\
	\hspace{8.5ex}8. $S = V_{n+1}+q_{n+1}\bar{x}_{n+1}\bar{x}_{n+1}^{T} - P_{n+1}\bar{x}_{n+1}^{T} - \bar{x}_{n+1}P_{n+1}^{T} $;\\ 
	\hspace{4.5ex}  \textbf{Output:} ASGD estimator $\bar{x}_{n+1}$, estimated covariance $\hat{\Sigma}_{n+1} = S/v_{n+1}$
	\caption{Update ASGD iterate and covariance matrix estimate recursively}
	\label{algo:recursive}
\end{algorithm}

%
%
%

\subsubsection{An alternative version}
\label{sec: NOL}
 The estimate $\hat\Sigma_{n}$ in \eqref{est} includes squared block sums from all $n$ blocks $\{B_{i}\}_{i=1,2,..., n}$. Block $B_{i}$ and $B_{j}$ are overlapped when $a_{m}\le i < j <a_{m+1}$ for some $m$. So $\hat\Sigma_{n}$ in \eqref{est} is a full overlapping version of the online batch means estimator. We also introduce an alternative non-overlapping version with a slightly simpler form which has a comparable performance. As data arriving sequentially, we follow the same batch scheme above to construct $\{B_{i}\}_{i=1,2,..., }$ while only include a few squared block sums. At the $n$-th step, define set $S_{n} = \{n\}\bigcup\{a_{i}-1: i>1, a_{i}\le n\}$. Consider a set of non-overlapping blocks $\{B_{i}\}_{i\in S_{n}}$, i.e.,
\begin{alignat*}{4} 
\{\{x_{a_{1}}, &..., x_{a_{2}-1}\}, ..., \{x_{a_{m-1}},&&..., x_{a_{m}-1}\}, \{x_{a_{m}},&&..., x_{n}\}\} .\\
& B_{a_{2}-1} &&B_{a_{m}-1}   &&B_{n} . 
\end{alignat*}
The alternative non-overlapping estimate at the $n$-th step includes squared block sums of $\{B_{i}\}_{i\in S_{n}}$. It is then defined as
\begin{equation}\label{est_NO}   
\hat{\Sigma}_{n, NOL} = \frac{1}{n}\underset{i\in S_{n}}{\sum}\left(\sum_{k=t_{i}}^{i}x_{k}-l_{i}\bar{x}_{n}\right)\left(\sum_{k=t_{i}}^{i}x_{k}-l_{i}\bar{x}_{n}\right)^{T}. 
\end{equation}  
The non-overlapping version estimator is also recursive and can perform a real-time update. The algorithm is almost the same as the overlapping one with same computational and memory complexity. 
One can follow the derivation of Algorithm \ref{algo:recursive} to get Algorithm \ref{algo:non-overlapping}.

\begin{algorithm}[!t] 
	\textbf{Input:} function $f(\cdot)$, parameter $(\alpha, \eta)$, step size $\eta_{i}=\eta i^{\a}$ for $i\ge 1$, pre-defined sequence $\{a_{m}\}_{m\in \N}$.\\ 
	\textbf{Initialize:} $m_{0} =l_{0} = 0 , v_{0} = P_{0} = q_{0}= V_{0} =  W_{0} = \bar{x}_{0} = 0, x_{0}$\;
	
	\textbf{For} $n$ = 0, 1, 2, 3, ...\\
	\hspace{4.5ex} 
	\textbf{Receive:} new data $\xi_{n+1}$\\ 
	\hspace{4.5ex} \textbf{Do the following update:}\\
	\hspace{8.5ex}1. $x_{n+1} = x_{n} - \eta_{n+1}\nabla{f(x_{n}, \xi_{n+1})}$;\\
	\hspace{8.5ex}2. $\bar{x}_{n+1} = (n\bar{x}_{n}+x_{n+1})/(n+1)$;\\
	\hspace{8.5ex}4. \textbf{if} $n+1 = a_{m_{n}+1}$, \textbf{then}: \\
	\hspace{10.5ex}$m_{n+1} = m_{n} + 1$;  $l_{n+1} =  1$; $W_{n+1} = x_{n+1}$;\\
	\hspace{10.5ex} $q_{n+1} = q_{n}+l_{n}^{2}$;  $V_{n+1} = V_{n} + W_{n}W_{n}^{T}$; $P_{n+1} = P_{n} + l_{n}W_{n} $\\
	\hspace{8.5ex} \textbf{else:} \\
	\hspace{10.5ex}  $m_{n+1} = m_{n}$; $l_{n+1} = l_{n} + 1$;  $W_{n+1} =  W_{n} + x_{n+1}$;\\
	\hspace{10.5ex} $q_{n+1} = q_{n}$;  $V_{n+1} = V_{n}$;  $P_{n+1} = P_{n}$\\
	\hspace{8.5ex}5. $S' = W_{n+1}W_{n+1}^{T} + l_{n+1}^{2}\bar{x}_{n+1}\bar{x}_{n+1}^{T} - l_{n+1}W_{n+1}\bar{x}_{n+1}^{T} - l_{n+1}\bar{x}_{n+1}W_{n+1}^{T}$;\\	  
	\hspace{8.5ex}6. $S = V_{n+1}+q_{n+1}\bar{x}_{n+1}\bar{x}_{n+1}^{T}- P_{n+1}\bar{x}_{n+1}^{T} - \bar{x}_{n+1}P_{n+1}^{T}  + S'$;\\ 
	\hspace{4.5ex}  \textbf{Output:} ASGD estimator $\bar{x}_{n+1}$, estimated covariance $\hat{\Sigma}_{n+1, NOL} = S/(n+1)$
	\caption{Update ASGD estimator and covariance matrix estimate (non-overlapping version) recursively}
	\label{algo:non-overlapping}
\end{algorithm}

In the stationary process case, \citet{resampling, lahiri1999theoretical} showed that the mean squared error of the classic (non-recursive) non-overlapping batch-means estimate is 33\% larger than that of its overlapping version, while the convergence rates are the same. 
 The comparison between the full overlapping version and the non-overlapping version of our online estimators is more complicated in the non-stationary case. In Section \ref{sec:nlcov}, we provide upper bounds for estimation errors for both overlapping and non-overlapping estimators. The two upper bounds are of the same order. The non-overlapping version is easier to analyze theoretically, given its simpler structure.  In the mean estimation model, we can obtain the precise order of the mean squared error for the non-overlapping one; see Section \ref{sec:simple case}.  We also compare the empirical performance of the two versions in Section \ref{sec: simulation_online}. However, it is hard to tell which one is more efficient based on the simulation results. We leave the rigorous comparison as a future research problem by extending \cite{resampling} to non-stationary processes.

\begin{remark}[Comparison with the non-recursive batch-means covariance matrix estimator]\label{rm: BM}
	  The non-overlapping version \eqref{est_NO} appears similar to the batch-means estimator \citep{chen2016statistical}. However, the batch schemes of the two methods are fundamentally different.  \citet{chen2016statistical} split $n$ iterates of SGD into $M+1$ non-overlapping blocks,  
	where $M$ and batch sizes $b_{m, n}$ ($m = 0, ..., M)$ are chosen based on $n$ for desired convergence.  With $e_{m, n}$ denoting the ending index of the $k$-th block, the covariance matrix estimator at $n$-th iteration in \citet{chen2016statistical} is defined as 
	\begin{equation}\label{eq: BM}   
	\hat{\Sigma}_{n, BM} = \frac{1}{M}\sum_{m = 1}^{M}b_{m, n}\left(\sum_{k=e_{m-1, n}+1}^{e_{m, n}}x_{k}/b_{m, n}-  \bar{x}_{n}\right)\left(\sum_{k=e_{m-1,n}+1}^{e_{m,n}}x_{k}/b_{m, n} - \bar{x}_{n}\right)^{T}, 
\end{equation} 
	where $e_{M, n} = n$. The optimal batch size setting as suggested in \citet{chen2016statistical} is
	$e_{m, n} =\left((m+1)/(M+1)\right)^{1/(1-\alpha)}n$ with the number of batches $M = n^{(1-\alpha)/2}$.  Since $e_{m, n}$ must depend on $n$ to ensure the desired convergence rate at the $n$-th iteration, there is no simple algebraic relation between $\hat\Sigma_{n, BM}$ and $\hat\Sigma_{n+1, BM}$. So the batch-means estimator \citep{chen2016statistical} is only suitable for offline tasks requiring final prediction/inference given the \emph{pre-specified} total sample size $n$. In contrast, our fully online estimator can sequentially improve over each iteration. Also, $n$ does not need to be specified beforehand.
	
\end{remark}

\begin{remark}[Choice of batch-sizes when $n$ is unknown]\label{rm: BM_tol}
  \citet{chen2016statistical} also propose an approach based on a target error tolerance to apply the batch-means estimator when $n$ is unknown. In particular, given the pre-specified error $\epsilon$, \citet{chen2016statistical} propose to set the ending index of the $k$-th batch by $e_{k} = \left((k+1)C\epsilon^{-2}\right)^{1/(1-\alpha)}$, where $C$ is a constant. The approach indeed enables an online updating, thus achieve the goal of recursive processing. However, choosing the constant $C$ can be difficult or arbitrary in online settings. Moreover, there is a fundamental difference. The approach in \citet{chen2016statistical} only ensures that the expected spectrum norm loss of the covariance matrix is smaller than $\epsilon$ (up to a constant) for large $n$, rather than goes to $0$. In other words, the covariance matrix estimator is not necessarily \emph{consistent}. While our online method constantly improves the covariance matrix estimate as $n\rightarrow \infty$, and the estimation error goes to $0$.
	\end{remark}

	

\subsubsection{Choice of batch sizes}
The remaining question is to specify the sequence $\{a_{m}\}_{m\in\N}$. This pre-defined sequence does not depend on $n$. This ensures that we can construct batches even if the total number of data is unknown (which is a typical situation), and the incoming data will not affect the recursive estimation process. In Section \ref{sec:nlcov}, we show that $a_{m}$ is required to take a polynomial form so that the estimator is consistent. Next, we shall give some intuitive explanation and one example of choice. 

The formula in \eqref{est} bears a certain similarity to the sample covariance matrix $S_{n} = n^{-1}\sum_{i=1}^{n}(x_{i}-\bar{x}_{n})(x_{i}-\bar{x}_{n})^{T}$. On the other hand, in contrast to the standard sample covariance matrix where $\{x_{i}\}_{i\ge 1}$ are independent, our SGD iterates in \eqref{est} are highly correlated.  In other words,  we cannot ignore the covariance between data as in the construction of the sample covariance matrix. 
According to \eqref{eq:SGDite}, the correlation between $x_{i}$ and $x_{j}$ diminishes as the distance $|j-i|$ becomes larger, while the correlation between $x_{i}$ and $x_{i+1}$ becomes stronger as $i$ goes to infinity. The idea of online estimation is to choose sequence $(a_{m})_{m\in\N}$ and form non-overlapping blocks $\{B_{a_{m}-1}\}_{m>1}$ as mentioned above such that the correlation between $x_{i}$ and $x_{j}$ is sufficiently small when they are in different non-overlapping blocks. So when considering the effect of $x_{i}$, we trace back to the starting point of the non-overlapping block $x_{i}$ belongs to, i.e., construct data block $B_{i} = \{x_{t_{i}},..., x_{i}\}$. 
Recall that the $i$-th iterate $x_{i}$ through SGD takes the form $$x_{i} = x_{i-1} - \eta_{i}\nabla{f(x_{i-1}, \xi_{i})}.$$ Let $\delta_{i} = x_{i} - x^{*}$ be the error sequence, where $x^*$ is the minimizer in \eqref{eq:obj}. Then
\begin{equation}\label{eq:error}
\delta_{i} = \delta_{i-1}- \eta_{i}\nabla F(x_{i-1}) + \eta_{i}\epsilon_{i},
\end{equation}
where $\epsilon_{i}=\nabla F(x_{i-1})-\nabla f(x_{i-1}, \xi_{i})$. 
Note that $\nabla F(x^{*}) = 0$ since $x^{*}$ is the minimizer of $F(x)$. By Taylor's expansion of $\nabla F(x_{i-1})$ around $x^{*}$, we have $\nabla F(x_{i-1})\approx  \nabla A\delta_{i-1}$, where $A= \nabla^2 F(x^*)$.  Thus, by modifying equation (\ref{eq:error}) with $\nabla F(x_{i-1})$ approximated by $A\delta_{i-1}$, for large $i$ 
\begin{align} 
\delta_{i} \approx (I-\eta_{i}A)\delta_{i-1} + \eta_{i}\epsilon_{i}.
\end{align}
Then for the $i$-th iterate $x_{i}$ and the $j$-th iterate $x_{j}$ (assume $j<i$), the strength of correlation between them is roughly
\begin{equation}
\Pi_{k=j+1}^{i}\left\|I_{d} - \eta_{k}A\right\|_{2}\le (1-\eta\lambda_{A}i^{-\alpha})^{i-j},  
\end{equation} 
when $\eta_{k} = \eta k^{-\alpha}$.
To make the correlation small, one can choose $i-j \approx Ki^{(\alpha+1)/2}$, where $K$ is a constant. Then the correlation is less than $(1-\eta\lambda_{A}i^{-\alpha})^{Ki^{\alpha}i^{(1-\alpha)/2}}$, which goes to zero as $i$ goes to infinity. Combining the correlation between $x_{i}, x_{j}$ and the form of $i-j$, a reasonable setting is that the sequence $\{a_{m}\}_{m\in\N}$ satisfies 
\begin{equation}\label{eq:a}
a_{m} - a_{m-1} = Ka_{m}^{(\alpha+1)/2}.
\end{equation}
Let $a_{m}$ increase polynomially, i.e., $a_{m} =  Cm^{\beta}$ for some constant $C$.  We obtain $\beta = 2/(1-\alpha)$ by solving equation (\ref{eq:a}).  Thus a natural choice of $a_{m}$ is
\begin{equation}
a_{m} = \left\lfloor Cm^{2/(1-\a)}\right\rfloor.
\end{equation}  
This is also the best choice in the general setting, as discussed in Section \ref{sec:nlcov}. However, the best choice of $\beta$ may change considering specific objective functions.


\subsection{Statistical inference}\label{sec:inf}
 Now the limiting covariance matrix $\Sigma$ can be approximated through the online estimation proposed above. Let $0<q<1$. Based on the asymptotic normality of ASGD in \eqref{eq:asym_norm}, the $(1-q)100\%$ confidence interval for  $x_{i}^{*}$, the $i$-th coordinate of $x^{*}$, can be constructed as 
\begin{equation}\label{eq:ci_cor}
\left[\bar{x}_{n, i}- z_{1-q/2}\sqrt{\hat{\sigma}_{ii}/n},\ \bar{x}_{n, i}+ z_{1-q/2}\sqrt{\hat{\sigma}_{ii}/n}\right],
\end{equation}
where $\bar{x}_{n, i}$ is the $i$-th coordinate of $\bar{x}_{n}$, $z_{1-q/2}$ is the $(1-q/2)$-th percentile of the standard Gaussian distribution and  $\hat{\sigma}_{ii}$ is the $i$-th diagonal of the covariance matrix estimate. The confidence interval is constructed in a fully online fashion since both $\bar{x}_{n, i}$ and $\hat{\sigma}_{ii}$ can be computed recursively. Joint confidence regions and general form of confidence intervals are discussed in Section \ref{sec:CI}.

 \subsubsection{Relation to empirical likelihood}
\label{sec: EL}
As pointed out by a reviewer, the construction of the non-overlapping version estimator shares a similar spirit with the blocking scheme and covariance estimator by \citet{kim2013progressive}, who developed a progressive block empirical likelihood (PBEL) method. They consider a stationary, weakly dependent sequence $(X_1,...,X_n)$ with mean $\mu$ such that the CLT $\sqrt{n}(\bar{X}_{n} - \mu) \Rightarrow N(0, \sigma^{2})$ holds. The variance estimator $\hat{\sigma}_{n, NOL}^{2}$ in \citet{kim2013progressive} matches our scheme in Section \ref{sec: NOL} with $a_{m} = (m-1)m/2 + 1$ (or the $i$-th block has length $i$) and is shown to be a consistent variance estimator. The chi-squared limit of the log-likelihood ratio based on PBEL is established following the consistency of $\hat{\sigma}_{n, NOL}^{2}$. It would be interesting to see if one can obtain similar results as the PBEL ratio and establish a limiting distribution that can be used to calibrate confidence regions in the SGD case here. 

\section{Theoretical Results}
\label{sec:theory}

\subsection{Preamble: mean estimation model} 
\label{sec:simple case}
Before investigating the convergence property of the online batch means estimators in the general setting, we shall look at the simple mean estimation example. Taking advantage of the simpler structure of the non-overlapping version, we can obtain the exact order of convergence. Consider the  mean estimation model:
$$y = x^{*} + e,$$
where $x^{*}\in \R$ is the mean we want to estimate, $e$ is the random error with mean $0$. Let $\{y_{i}\}_{i\in\N}$ be a sequence of \emph{i.i.d} sample from the model. Consider the squared loss function at $x$, $F(x) = (y-x)^{2}/2$. The $i$-th SGD iterate takes the form 
\begin{equation}
\label{eq:sgd_mean}
x_{i} = x_{i-1} +\eta_{i}(y_{i} - x_{i-1}), i\ge 1,
\end{equation}
where we choose the step size $\eta_{i} = \eta i^{-\alpha}$, $\alpha\in(1/2, 1)$. Then the error $\delta_{i} = x_{i} - x^{*}$ takes the form 
$$\delta_{i} = (1-\eta_{i})\delta_{i-1} + \eta_{i}e_{i}.$$
 In this case, one can have an explicit form of  $\text{var}(\sqrt{n}{\bar x_{n}})$ and $\hat\Sigma_{n, NOL}$. Additionally, we can have an explicit form for the order of magnitude of the mean squared error of $\hat\Sigma_{n, NOL}$. Let the variance $\text{var}({\sqrt{n}\bar x_{n}}) = \sigma_{n}^{2}$. We have the following proposition. 
\begin{proposition}
	\label{prop:mean}
	For $m\ge 2$, let $a_{m} = \lfloor c m^{\beta} \rfloor$, where $\beta > 1$ and $c > 0$ are constants. Given the SGD iterates defined in \eqref{eq:sgd_mean}, we have
	\begin{equation}
	\E(\hat\Sigma_{n, NOL} - \sigma_{n}^{2})^{2}\asymp n^{-1/\beta}+ n^{2\alpha + 2/\beta-2}.
	\end{equation}
\end{proposition}
Choose $\beta = {3} / (2(1-\alpha))$. In the mean estimation model, the above proposition asserts that the convergence rate of the mean squared error of our recursive non-overlapping variance estimate is $n^{-{2(1-\alpha)}/{3}}$. For $\alpha$ close to $1/2$, the latter rate approaches $n^{-{1}/{3}}$. 
This rate is faster than that of the batch-means estimator in \citet{chen2016statistical}, which approaches $n^{-1/4}$. So, besides the advantage of the recursive property, our estimator may improve the convergence rate.

In the general setting, the analysis is much more complicated due to the nonlinearity. Upper bounds for the convergence rates of online estimators for both overlapping and non-overlapping versions are given in Section \ref{sec:nlcov}.

\subsection{Assumptions and existing convergence results}
In the work of \citet{polyak1992acceleration}, assumptions on the objective function $F(x)$ and the gradient difference are proposed to prove the asymptotic normality of ASGD estimate. Those assumptions are necessary for our problem since we adopt the ASGD as the point estimator and require the asymptotic normality for statistical inference. Those assumptions, as well as some error bounds, are also proposed in other literature. We impose similar assumptions and review some existing results in this section.   
\begin{assumption}\label{ass:reg} 
	Assume that the objective function $F(x)$ is continuously differentiable and strongly convex with parameter $\mu>0$. That is, for any $x_{1}$ and $x_{2}$,
	$$
	\begin{aligned}
	F(x_{2})\ge F(x_{1}) + \langle\nabla F(x_{1}), x_{2}-x_{1}\rangle + \frac{\mu}{2}\|x_{1}-x_{2}\|_{2}^{2}.
	\end{aligned}
	$$
	Furthermore, assume that $\nabla^{2}F(x^{*})$ exists and $\nabla F(x)$ is Lipschitz continuous in the sense that there exist $L>0$ such that,
	$$
	\left\|\nabla F(x_{1}) - \nabla F(x_{2})\right\|_{2}\le L\|x_{1}-x_{2}\|_{2}. 
	$$
\end{assumption}

\begin{assumption}\label{ass:martingale} 
	For the $n$-th iteration, define error $\delta_{n} = x_{n} - x^{*}$ and gradient difference $\epsilon_{n}=\nabla F(x_{n-1})-\nabla f(x_{n-1}, \xi_{n})$. Recall that $\E_{n}(\cdot) = \E(\cdot|\xi_{n}, \xi_{n-1}, ...)$. The following hold:
	\begin{description}
		\item[1).] The function $f(x, \xi)$ is continuously differentiable with respect to $x$ for any $\xi$ and $\|\nabla f(x, \xi)\|_{2}$ is uniformly integrable for any $x$. So $\E_{n-1}\left[\nabla f(x_{n-1}, \xi_{n})\right] = \nabla F(x_{n-1})$, which implies that $\E_{n-1}\left(\epsilon_{n}\right) = 0$.
		\item[2).] 
		
		The conditional covariance of $\epsilon_{n}$ has an expansion around $S$ which satisfies
		\begin{equation} 
		\begin{split}
		\left\|\E_{n-1}\left(\epsilon_{n}\epsilon_{n}^{T}\right) - S\right\|_{2}\le C\left(\|\delta_{n-1}\|_{2} + \|\delta_{n-1}\|_{2}^{2}\right),
		\end{split}
		\end{equation}
		where $C>0$ is some constant. Here $S$ is defined in  \eqref{sandwich}.
		\item[3).] There exists a constant $C>0$ such that the fourth conditional moment of $\epsilon_{n}$ is bounded by $$\E_{n-1}\left(\|\epsilon_{n}\|_{2}^{4}\right)\le C\left(1+\|\delta_{n-1}\|_{2}^{4}\right).$$
	\end{description}
	
\end{assumption}

Assumption \ref{ass:reg} imposes strong convexity of the objective function $F(x)$ and Lipschitz continuity of its gradient. Assumption \ref{ass:martingale} asserts the regularity and the bound of the noisy gradient. These assumptions are widely used in SGD literature \citep{ruppert1988efficient, polyak1992acceleration, moulines2011non, rakhlin2011making}. With these assumptions, we have the asymptotic normality for averaged SGD iterates by \citet{polyak1992acceleration} and \citet{ruppert1988efficient}. We also review the error bound for SGD iterates in Lemma \ref{lemma:1}.

\begin{lemma}\label{lemma:1}
	Under Assumptions \ref{ass:reg} and \ref{ass:martingale}, for some constant $C>0$ and $n_{0}\in\N$, we have for any $n>n_{0}$, the sequence of error $\delta_{n} = x_{n} - x^{*}$ satisfies 
	
	$$
	\begin{aligned}
	\E(\|\delta_{n}\|_{2})\le Cn^{-\alpha/2}(1+\|\delta_{0}\|_{2}),\\
	\E(\|\delta_{n}\|_{2}^{2})\le Cn^{-\alpha}(1+\|\delta_{0}\|_{2}^{2}),\\
	\E(\|\delta_{n}\|_{2}^{4})\le Cn^{-2\alpha}(1+\|\delta_{0}\|_{2}^{4}),\\
	\end{aligned}
	$$
	when the step size is chosen to be $\eta_{n} = \eta n^{-\alpha}$ with $1/2<\alpha<1$.
\end{lemma}

\subsection{Convergence properties for the online estimator} 
\label{sec:nlcov}
\begin{theorem}\label{theorem2} Under Assumptions  \ref{ass:reg} and \ref{ass:martingale},  let $a_{m}=\left\lfloor Cm^{\beta}\right\rfloor$, where $C>0$ is a constant, $\beta>(1-\alpha)^{-1}$. Set step size at the $i$-th  iteration as $\eta_{i} =   \eta  i^{-\alpha}$ with $1/2<\a<1$. Then for $\hat{\Sigma}_{n}$ defined in (\ref{est}) 
	\begin{equation}
	\E\left\|\hat{\Sigma}_{n} - \Sigma\right\|_{2}\lesssim   n^{-1/(2\beta)} + n^{(\alpha-1)/2 + 1/(2\beta)}.
	\end{equation}
\end{theorem}

Theorem \ref{theorem2} shows that as $n\rightarrow \infty$, the estimator $\hat\Sigma_{n}$ converges to the limiting covariance matrix of the averaged SGD iterates in terms of operator norm loss. The convergence rate is associated with the parameters $\alpha$ and $\beta$. We state the following Corollary \ref{corollary1} to suggest the best choice of $\beta$.

\begin{corollary}\label{corollary1}
	Under conditions in Theorem \ref{theorem2} and let $\beta = 2/(1-\a)$, we have 
	\begin{equation}
	\E\left\|\hat{\Sigma}_{n} - \Sigma\right\|_{2}\lesssim n^{-(1-\alpha)/4}.
	\end{equation}
\end{corollary}

\begin{remark}
	\label{rm:rates}
 This convergence rate is the same as that of the non-recursive batch-means estimator in \citet{chen2016statistical}. According to Corollary 4.5 in \citet{chen2016statistical}, the upper bound of the batch means estimator is also $O(n^{-(1-\alpha)/4})$ with the prior knowledge of the sample size $n$. So we make it possible that online estimation of covariance matrix achieves the same efficiency as offline methods. 
	The plug-in approach in \citet{chen2016statistical} achieves the rate of $O(n^{-\alpha/2})$ when the $i$-th step size is chosen to be $i^{-\alpha}$. As a tradeoff, the online estimator enjoys efficient computation without the necessity of accessing Hessian information but pays the price in terms of the slower convergence rate. 
\end{remark}

Next, we will show in Theorem \ref{theorem: nonoverlap} that the alternative version $\hat{\Sigma}_{n, NOL}$ shares the same upper bound. 

\begin{theorem}\label{theorem: nonoverlap} Under conditions in Theorem \ref{theorem2}, the alternative version  $\hat{\Sigma}_{n, NOL}$ defined in (\ref{est_NO}) satisfies
	\begin{equation}
		\E\left\|\hat{\Sigma}_{n, NOL} - \Sigma\right\|_{2}\lesssim  n^{-1/(2\beta)} + n^{(\alpha-1)/2+ 1/(2\beta)}.
	\end{equation}
\end{theorem}



\subsection{Asymptotically accurate confidence intervals/regions}\label{sec:CI}
The next corollary shows that the confidence interval/region based on the online estimator achieves asymptotically correct coverage level $1-q$ for a pre-specified $q$ with $0<q<1$.  
\begin{corollary}\label{corollary2}
	Under conditions in Theorem \ref{theorem2},  as $n$ goes to infinity
	\begin{equation}\label{eq:CI} 
	\P( x_{i}^{*}\in \text{CI}_{q, n, i})\rightarrow 1-q,
	\end{equation}
	where $$
	\text{CI}_{q, n, i} = \left[\bar{x}_{n, i} -  z_{1-q/2}\sqrt{\hat{\sigma}_{ii}/n},\  \bar{x}_{n, i} +  z_{1-q/2}\sqrt{\hat{\sigma}_{ii}/n}\right]
	$$
	and $\hat{\sigma}_{ii}$ is the $i$-th diagonal of the online batch-means estimator $\hat\Sigma_{n}$ (or $\hat\Sigma_{n, NOL}$). We can also construct joint confidence regions as follows:
	\begin{equation}\label{eq:CIregion} 
	\P\left(x^{*}\in C_{q, n} \right)\rightarrow 1-q,
	\end{equation}
where $$C_{q, n} = \left\{x\in \R^{d}: n\left(\bar{x}_{n}-x\right)^{T}\hat\Sigma_{n}^{-1}\left(\bar{x}_{n}-x\right)\le \chi^{2}_{d, 1-2/q}\right\}.$$

\end{corollary}

Corollary \ref{corollary2} constructs asymptotic valid confidence intervals for each coordinate of $x^{*}$ and joint confidence regions for $x^{*}\in\R^{d}$. More generally, for any unit length vector $w\in \mathbb{R}^{d}$ (i.e., $\|w\|_2=1$),  we have by Theorem \ref{theorem2} and Slutsky's theorem, 
\begin{equation}\label{multinor}
\frac{\sqrt{n}w^{T}(\bar{x}_{n} - x^{*})}{\sqrt{w^{T}\hat{\Sigma}_{n}w}} \Rightarrow  N(0,1).
\end{equation}
Therefore, the $(1-q)100\%$ confidence interval for $w^{T}x^{*}$ can be constructed as 
\begin{equation}\label{est:ci}
\begin{split}
\left[w^{T}\bar{x}_{n} - z_{1-q/2}\sqrt{w^{T}\hat{\Sigma}_{n}w/n},
w^{T}\bar{x}_{n} + z_{1-q/2}\sqrt{w^{T}\hat{\Sigma}_{n}w/n}\right].
\end{split}
\end{equation}

\subsubsection{Stopping rule}
	\label{sec:stopping_rule}
 In principle, SGD constantly improves the quality of $\bar{x}_n$, and our method constantly improves the covariance estimate $\hat\Sigma_{n}$ as $n$ grows. A natural questions is when can we stop updating $\bar{x}_{n}$ and $\hat{\Sigma}_{n}$? There are several heuristics of stopping rules widely used in machine learning. For example, an online algorithm can stop when the neighboring estimates become sufficiently close. Or a more widely used approach in stopping SGD is to evaluate the error on a separate validation dataset and stops the SGD when the error becomes stable.  
	
    We can better answer this question and assess the SGD error based on the inference results, inspired by stopping rules for Markov Chain Monte Carlo (MCMC) that rely on a Markov chain central limit theorem. Especially, one can apply the fixed-width sequential stopping rule in \citet{jones2006}, where the updating is terminated the first time when  the width of the confidence interval for each component is small enough. More formally, for a desired tolerance of $\epsilon_{i}$ for the $i$-th coordinate, the rule terminates updating the first time after the $n$-th iteration when the following condition is satisfied for all the coordinates $i=1,\ldots,d$,
   $$
   t_{*}\frac{\hat\sigma_{n, i}}{\sqrt{n}} + n^{-1}\le \epsilon_{i},
   $$
   where $\hat\sigma_{n, i}$ is the $i$-th diagonal of the online estimator $\hat\Sigma_{n}$ (or $\hat\Sigma_{n, NOL}$), and $t_{*}$ is an appropriate $t$-distribution quantile. For the joint inference, one may consider simplifying the relative standard deviation fixed-volume sequential stopping rule in \citet{vats2019multivariate}, where updating is terminated the first time when the volume of the confidence region $C_{n}$ \eqref{eq:CIregion} is small enough. For a desired tolerance of $\epsilon$, the rule terminates updating the first time after the $n$-th iteration when
   $$
   \text{Vol}(C_{n})^{1/d} + n^{-1}\le \epsilon,
   $$
   where $\text{Vol}(C_{n}) =  2\left(\pi\chi_{*}^{2}/n\right)^{d/2}|\hat\Sigma_{n}|^{1/2}/(d\Gamma(d/2))$, $|\cdot|$ denotes determinant, $\chi_{*}^{2}$ is an appropriate chi-square distribution quantile, and $\hat\Sigma_{n}$ is our online estimator.
	We also include a simple simulation study of the stopping rule in the last section of the Supplement.

   \begin{remark}
   	\label{rm:stopping_rule}
   The original stopping rule in \citet{vats2019multivariate} avoids the practical issue of choosing $\epsilon$ with the idea of effective sample size (ESS). 
   They consider an F-invariant Harris recurrent Markov chain and define a multivariate approach to ESS.  The stopping rule in \citet{vats2019multivariate} terminates the MCMC simulation the first time the estimated ESS is larger than a  pre-specified lower bound. However, we need to re-define ESS in the non-stationary case, which requires more careful considerations. We will leave it as a future research direction. 
	\end{remark}

\section{Simulation Studies}
\label{sec:exper}
 In this section, we evaluate the empirical performance of the proposed online approach.  We focus on two classes of examples: linear regression and logistic regression. 
Let $\{\xi_{i}\equiv(a_{i}, b_{i})\}_{i=1,2,...}$ denotes an \emph{i.i.d} sequence of pairs, and $x^{*}$ denote the true parameter in the models. In both linear regression and logistic regression cases, $a_{i}\in \R^{d}$ is generated from $N(0,\mathbf{I}_{d})$. In the former case,  $b_{i} = a_{i}^{T}x^{*} + \epsilon_{i}$, where $\epsilon_{i}$ is independently generated from $N(0, 1)$. In the latter case, $b_{i}|a_{i}\sim Bernoulli((1 + \exp(-a_{i}^{T}x^{*} ))^{-1})$. The loss function $f(\cdot)$ is defined as the negative log likelihood function, so we have
$$
f(x, a_{i}, b_{i}) = 
\left\{
\begin{aligned}
	&\frac{1}{2}(a_{i}^{T}x-b_{i})^{2}\ \  &\text{linear regression}\\
	&(1-b_{i})a_{i}^{T}x + \log(1+\exp(-a_{i}^{T}x))\ \ &\text{logistic regression}.
\end{aligned}
\right.
$$
The true coefficient $x^{*}$ is   a $d$-dimensional vector  linearly spaced between
0 and 1.  In the SGD procedure,  the step size $\eta_{j}$ is set to be $0.5j^{-\alpha}$ and the parameter $\alpha$ is chosen to be $0.505$. The sequence $\{a_{k}\}_{k\ge 1}$ in our online approach is chosen in the form of $a_{m} = \left\lfloor C m^{2/(1-\alpha)}\right\rfloor$, for some constant $C$. 
All the measurements in the following discussions are averaged over 200 independent runs.

\subsection{Empirical performance of the proposed online approach} 
\label{sec: simulation_online}
\noindent{\bf Convergence of the recursive estimator.\ \ }
 We focus on linear regression here since the true limiting covariance matrix is easy to compute. In the linear regression model described above, $$A = \E\left[\nabla^{2}f(x^{*})\right] =\E\left(aa^{T}\right) = \mathbf{I}_{d},$$ 
$$S = \E\left([\nabla f(x^{*}, \xi)][\nabla f(x^{*}, \xi)]^{T}\right) = \E(\epsilon^{2})\E\left(aa^{T}\right) = \mathbf{I}_{d}.$$ Then the limiting covariance matrix $$\Sigma=A^{-1}SA^{-1} = \mathbf{I}_{d}.$$   
 We check the convergence of our proposed online estimators, both the full overlapping and the non-overlapping versions, by computing the operator norm loss of the covariance matrix estimate, i.e., $\|\hat\Sigma_{n} - \Sigma\|_{2}$. Figure  
\ref{fig:linear_loss} shows that the log loss of the online estimators are approximately linear with the log number of steps and the slopes are about $-1/8$ for the large total number of steps. It suggests that both the full overlapping and the non-overlapping versions converge to the limiting covariance matrix with the same convergence rate, about $O(n^{-1/8})$. We also compute the relative efficiency (MSE of the full overlapping version \eqref{est} divided by MSE of the non-overlapping version \eqref{est_NO}); see Figure \ref{fig:efficiency}. 
Their performances are comparable. Also, the performance is relatively insensitive to the choice of $C$ in $a_{m} = \left\lfloor Cm^{2/(1-\alpha)}\right\rfloor$.  Therefore, we will implement the non-overlapping version and set $C=1$ in the subsequent simulations without any specification.
\begin{figure}[!t]
	\centering
	\subfigure[$d$=1]{  
		\includegraphics[width=0.42\textwidth]{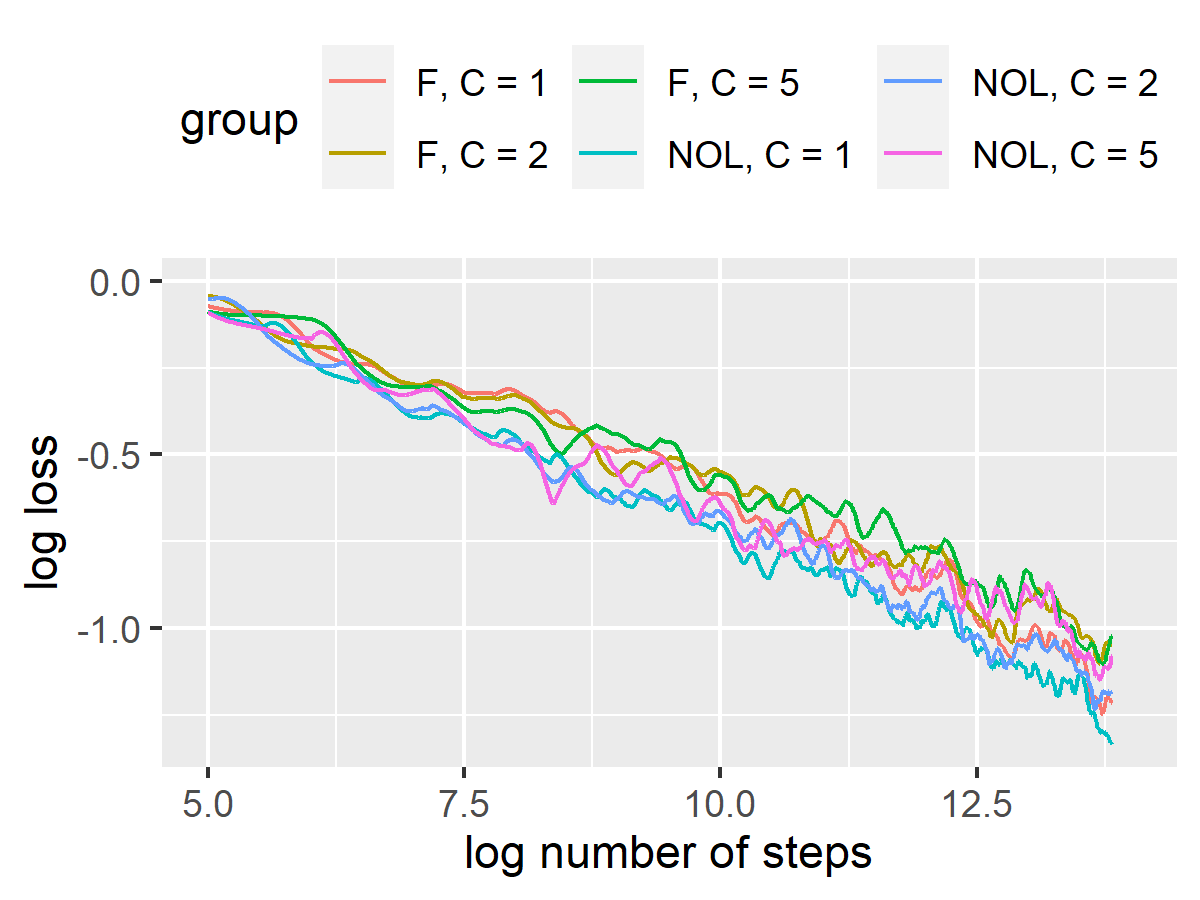}
		
	} 
	\subfigure[$d$=5]{ 
		\includegraphics[width=0.42\textwidth]{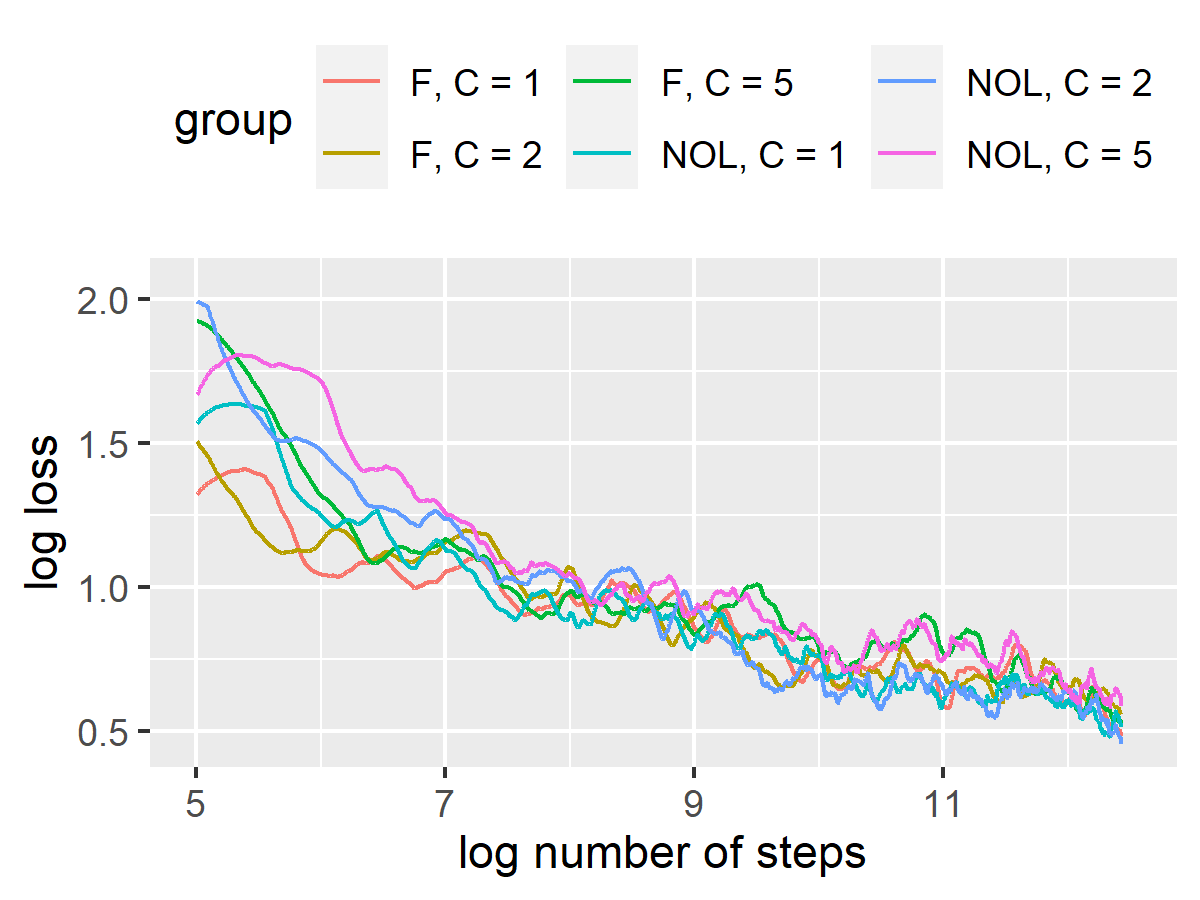}
	}
	\caption{Linear regression: Log loss (operator norm) of the estimated covariance matrix against the log of total number of steps. Here F denotes the full overlapping version \eqref{est}, NOL denotes the non-overlapping version \eqref{est_NO}, and $C$ denotes the constant in $a_{m} = \left\lfloor Cm^{2/(1-\alpha)}\right\rfloor$.}
	\label{fig:linear_loss}
\end{figure} 

\begin{figure}[!t]	 	
	\centering
	\includegraphics[width=0.5\textwidth]{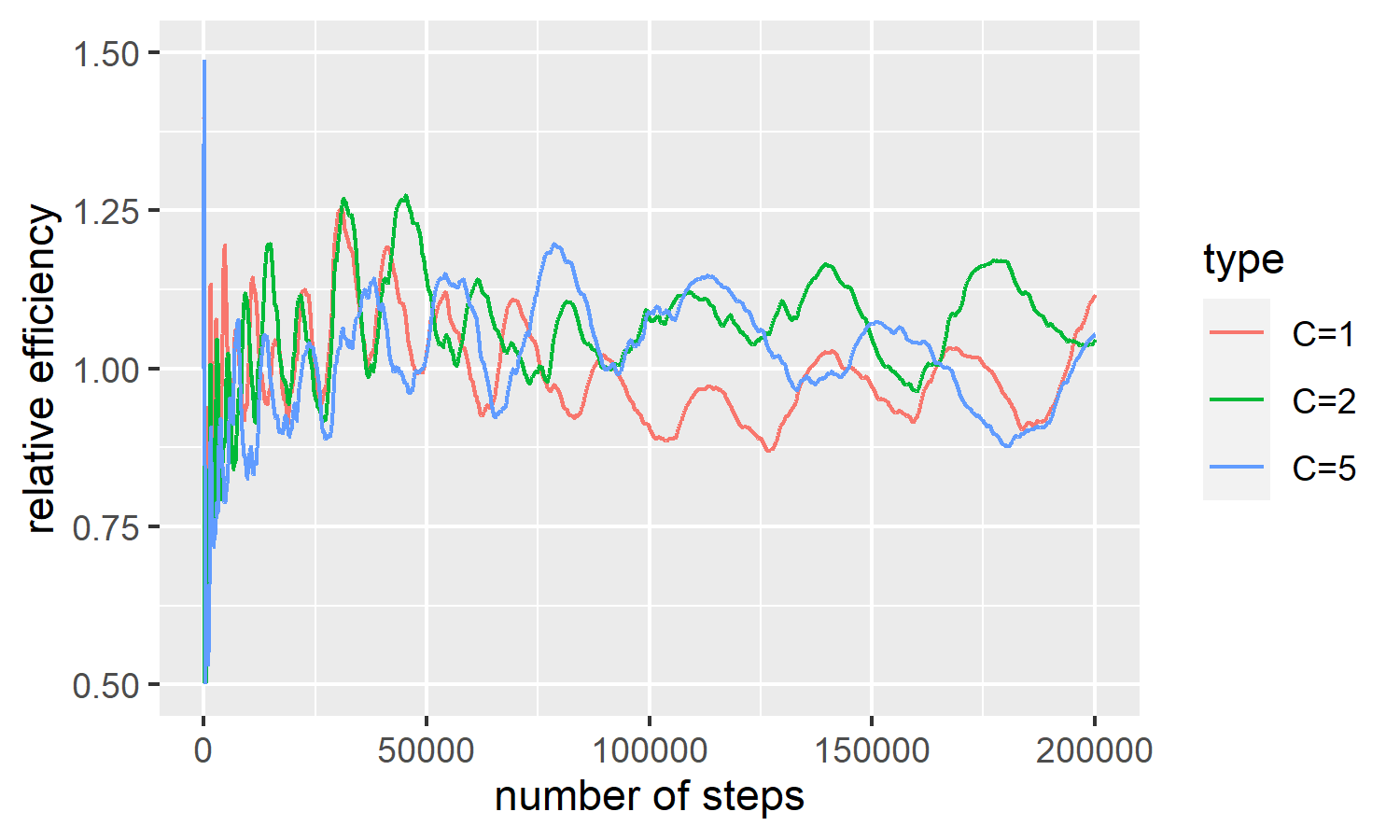}
	\caption{Relative efficiency (ratio of MSE) of the full overlapping version \eqref{est} and non-overlapping version \eqref{est_NO}. We set $d = 5$ in linear regression.  Here $C$ denotes the constant in $a_{m} = \left\lfloor Cm^{2/(1-\alpha)}\right\rfloor$. }
	\label{fig:efficiency}
\end{figure}


\noindent{\bf Asymptotic normality and CI coverage.\ \ } 
 With the covariance matrix estimates, we construct $95\%$ confidence intervals for the averaged coefficient $\mu = 1^{T}x^{*}$ according to \eqref{est:ci}, i.e., 
$$
\begin{aligned} 
\left[1^{T}\bar{x}_{n} - z_{1-q/2}\sqrt{1^{T}\hat{\Sigma}_{n}1/n},
1^{T}\bar{x}_{n} + z_{1-q/2}\sqrt{1^{T}\hat{\Sigma}_{n}1/n}\right].
\end{aligned}
$$
We also compute the oracle $95\%$ confidence intervals based on the true limiting covariance matrix.  Figure \ref{fig:5dlinear_ci} shows that for both overlapping and non-overlapping versions, the empirical coverage rate  converges to $95\%$, and the  standardized error $\sqrt{n}1^{T}(\hat{x} - x^{*})/\sqrt{1^{T}\hat{\Sigma}_{n}1}$ is approximately standard normal. Also, the estimated CI length converges to the oracle length.

\begin{figure}[!t]
	\centering
	\subfigure[Empirical cover rate]{  
			\includegraphics[width=0.7\textwidth]{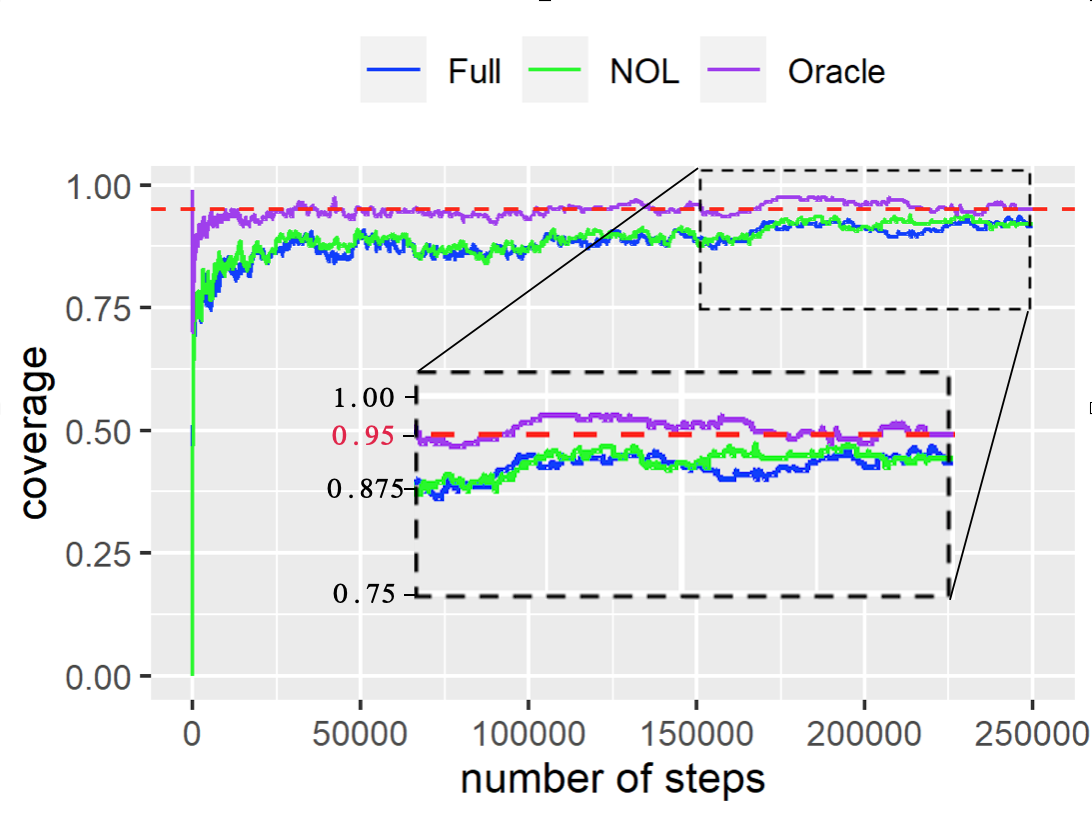}
	}\\
	\subfigure[CI length]{ 
		\includegraphics[width=0.4\textwidth]{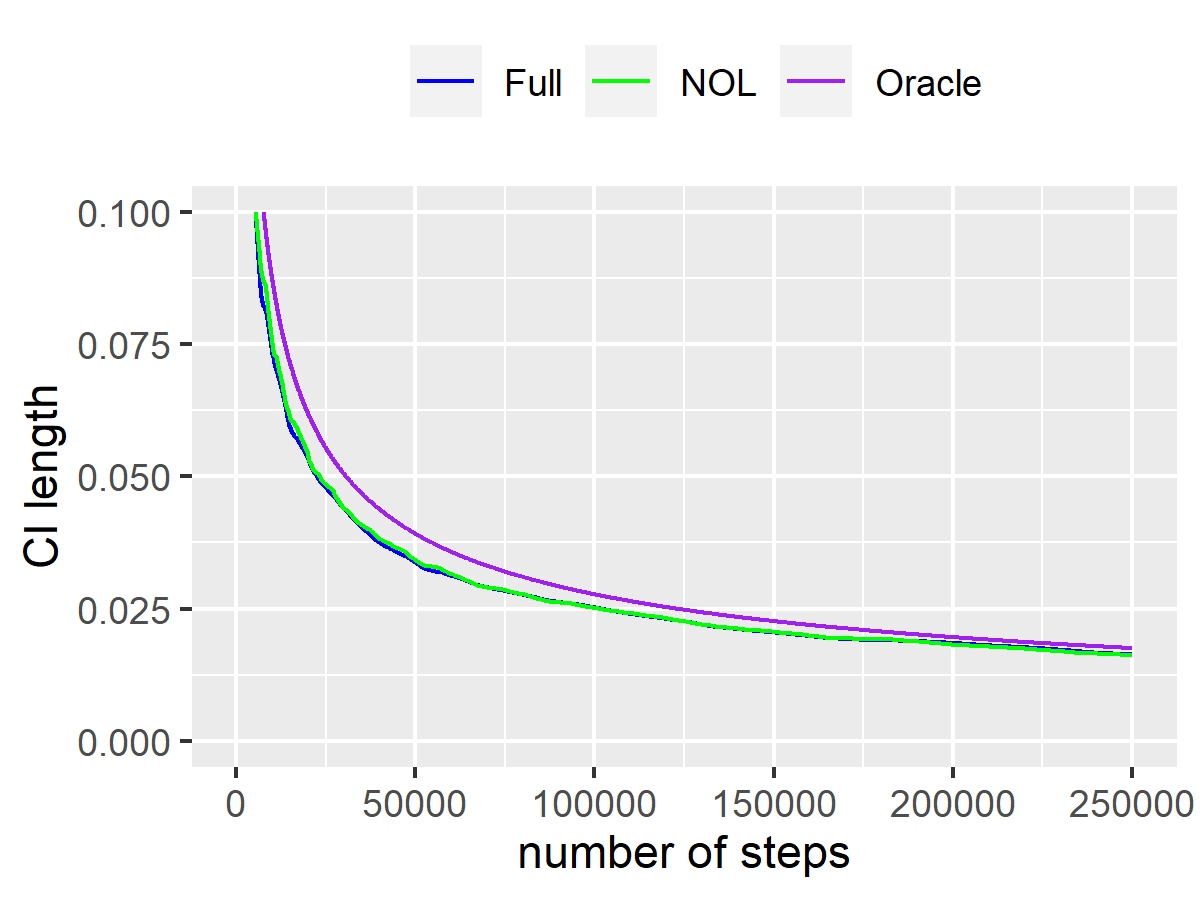}
	} 
	\subfigure[Normality]{  
		\includegraphics[width=0.4\textwidth]{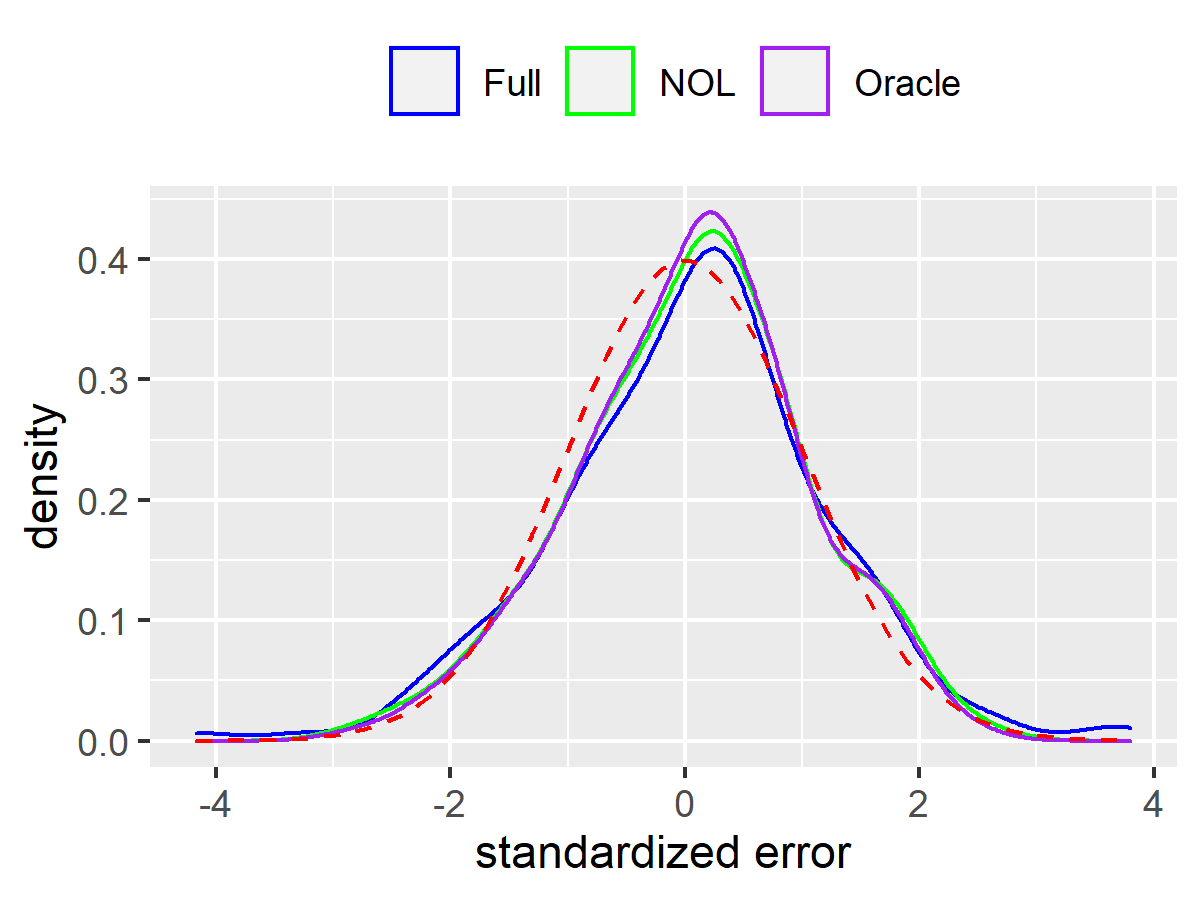} 
	}\\
	\caption{Linear regression with $d$ = 5: (a): Empirical coverage rate against the number of steps. Red dashed line denotes the nominal coverage rate of 0.95. (b): Length of confidence intervals. (c): Density plot for the standardized error. Red curve denotes the standard normal density. }
	\label{fig:5dlinear_ci}
\end{figure}


\subsection{Comparison with other methods}
\label{sec:simulation_comparision}
 In this section, we compare the performance of the proposed online estimator, which we refer to as online-BM in the subsequent numerical experiments, with other estimators for marginal inference of each individual regression coefficient.  We consider both linear and logistic regression examples. The nominal coverage probability is set to 95\%. 

 We first compare the empirical coverage rates of the proposed estimator with the plug-in estimator in \citet{chen2016statistical}. As we mentioned in the introduction, the plug-in estimator requires the computation of the Hessian matrix (of the loss function) and its inverse. Figure \ref{fig:compare_P_linear} shows that our online estimator (online-BM) has a comparable performance as the plug-in estimator when the number of iterations is large enough. Although the online-BM has a slower convergence rate, it has an advantage in computational efficiency  since it only uses the iterates from SGD. The online-BM is more desirable for practitioners when the computation is limited or only stochastic gradient information is available.
\begin{figure}
	\centering  
		\includegraphics[width=0.44\textwidth]{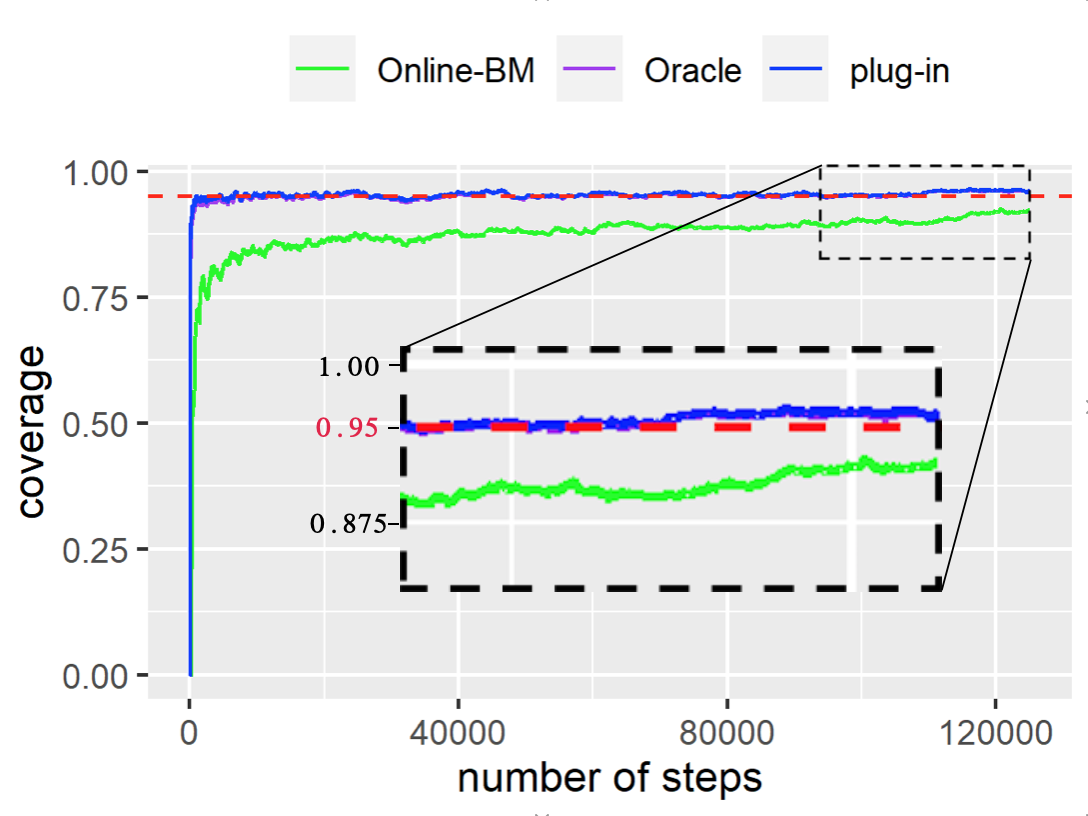} 
		   	\includegraphics[width=0.44\textwidth]{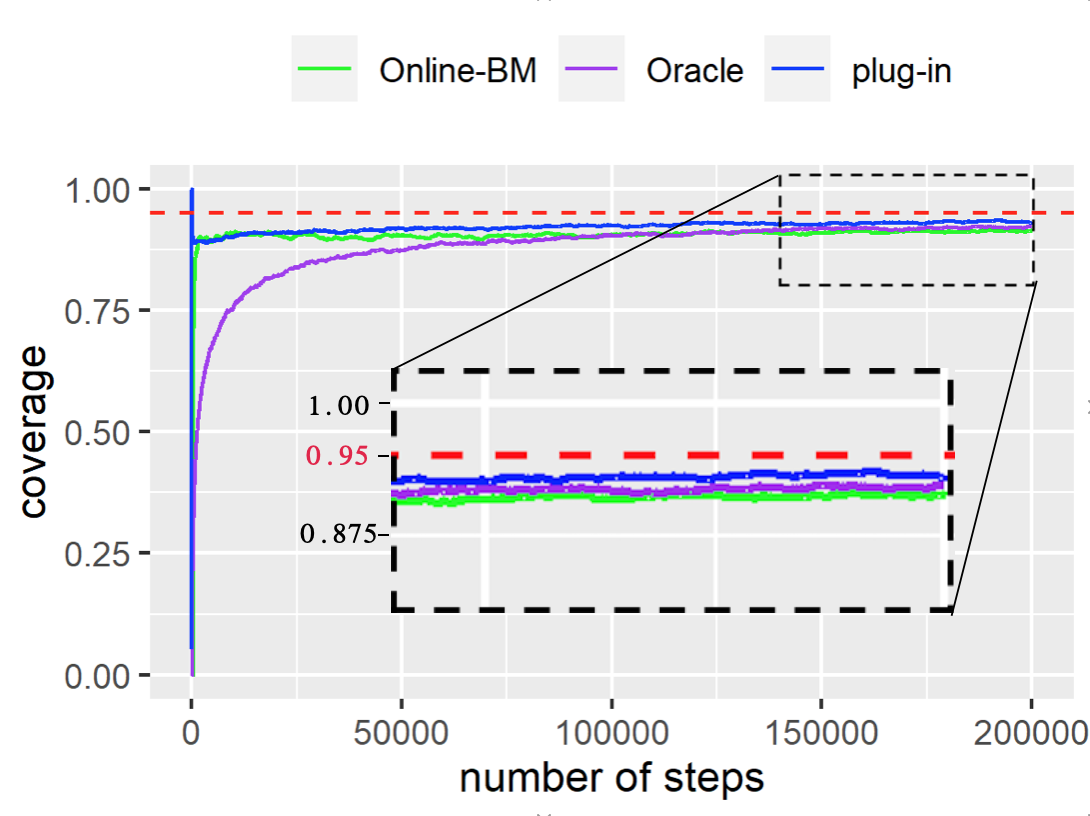}
	\\ 
\includegraphics[width=0.44\textwidth]{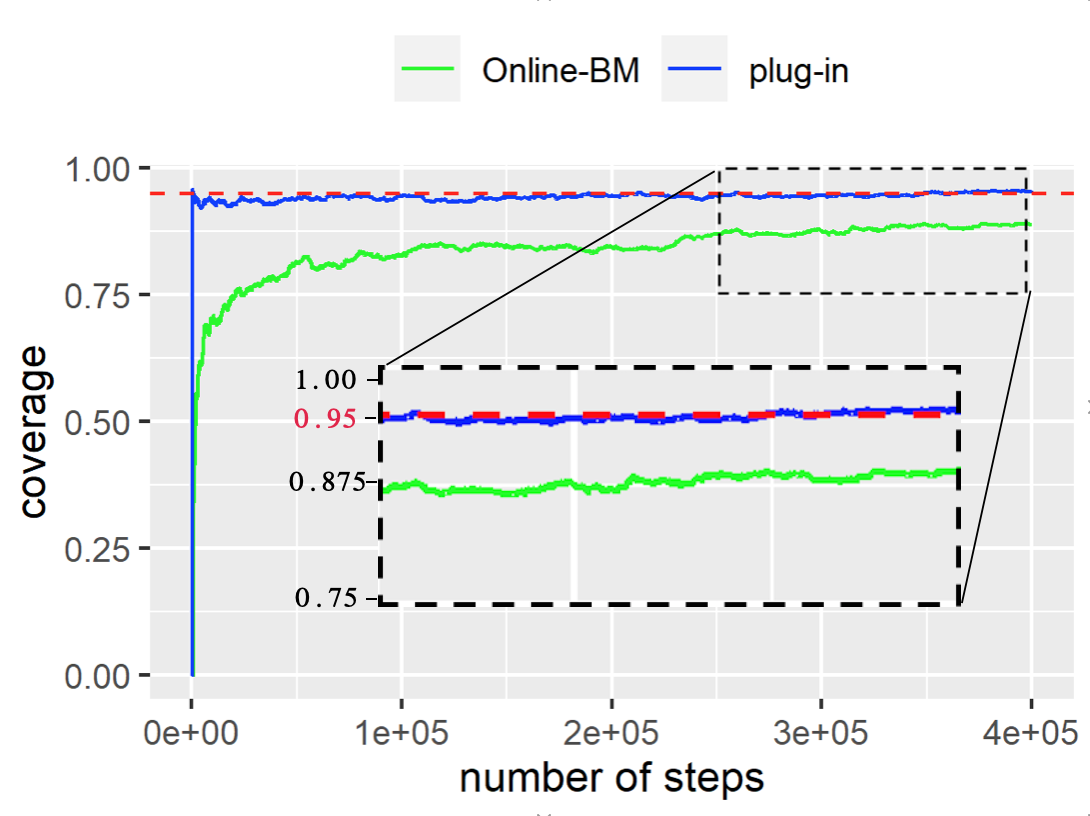} 
		\includegraphics[width = 0.44\textwidth]{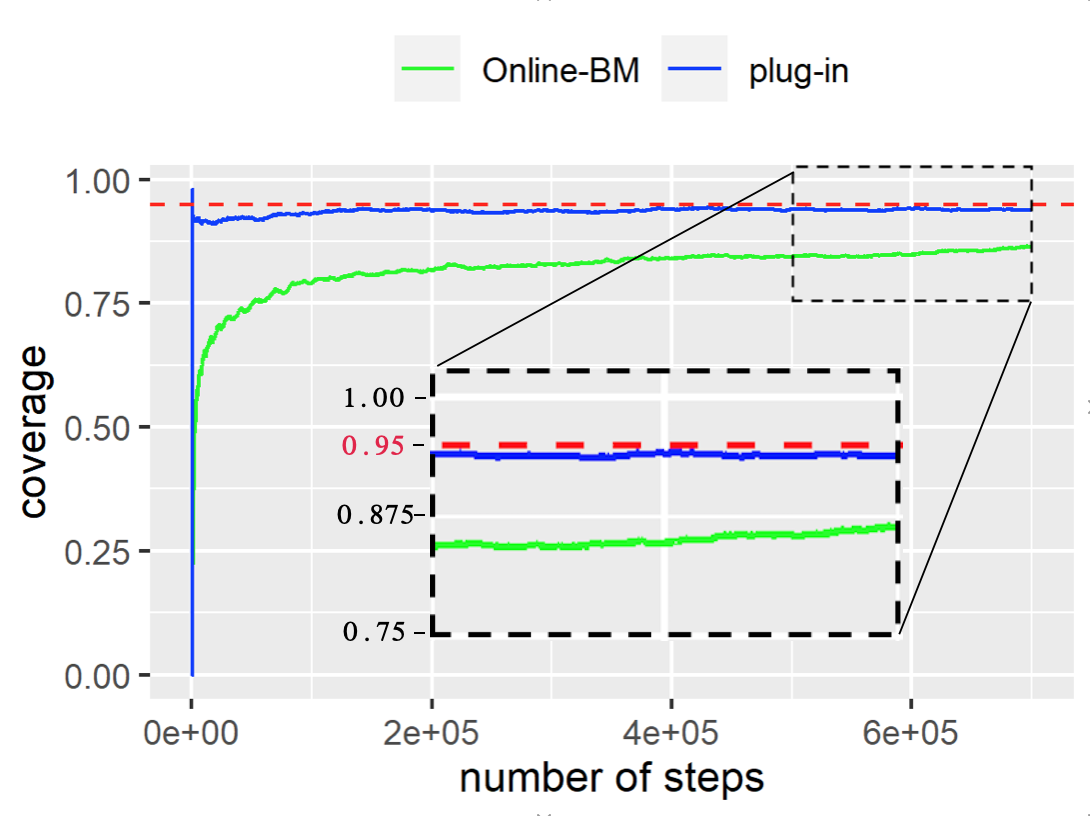} 
\\
	\subfigure[$d$ = 5]{ 	\includegraphics[width=0.44\textwidth]{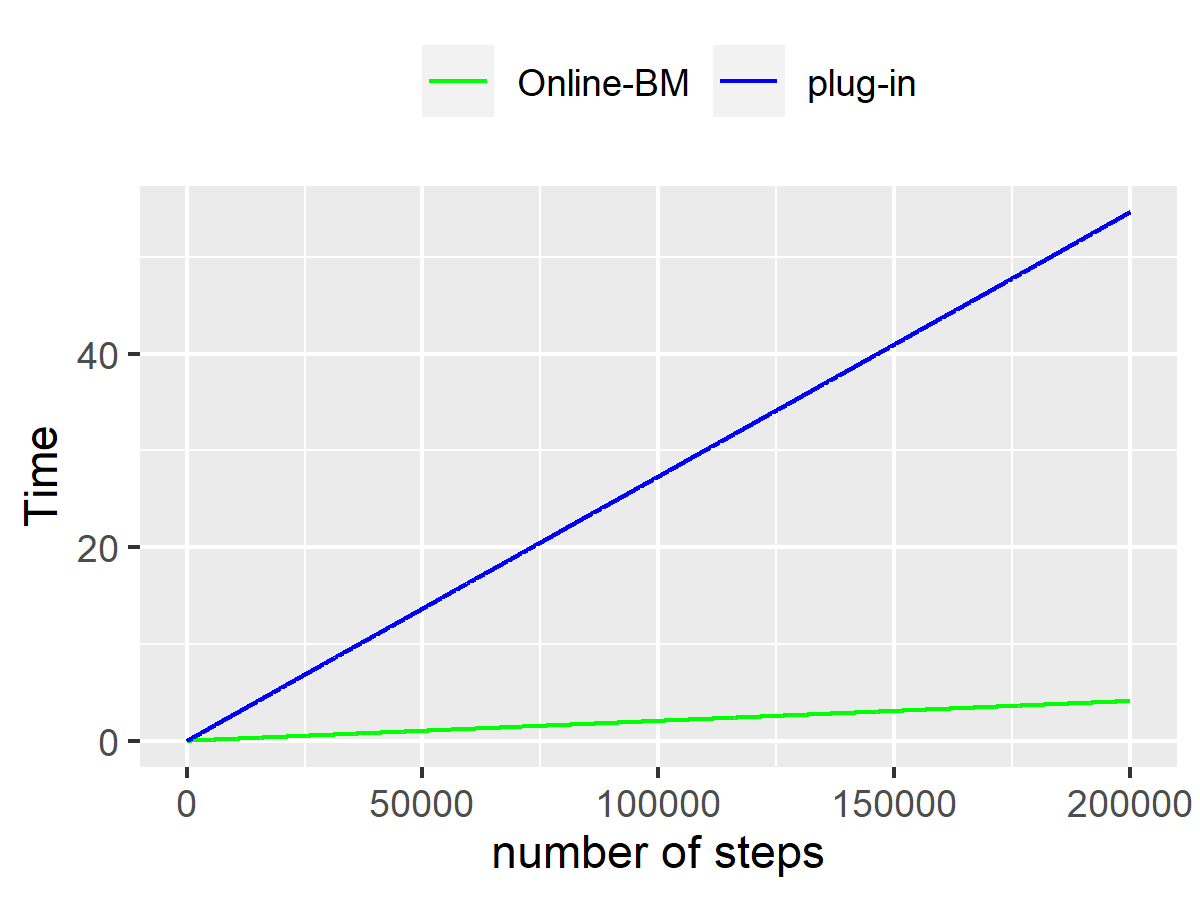}
	} 
	\subfigure[$d$ = 20]{  		\includegraphics[width= 0.44\textwidth]{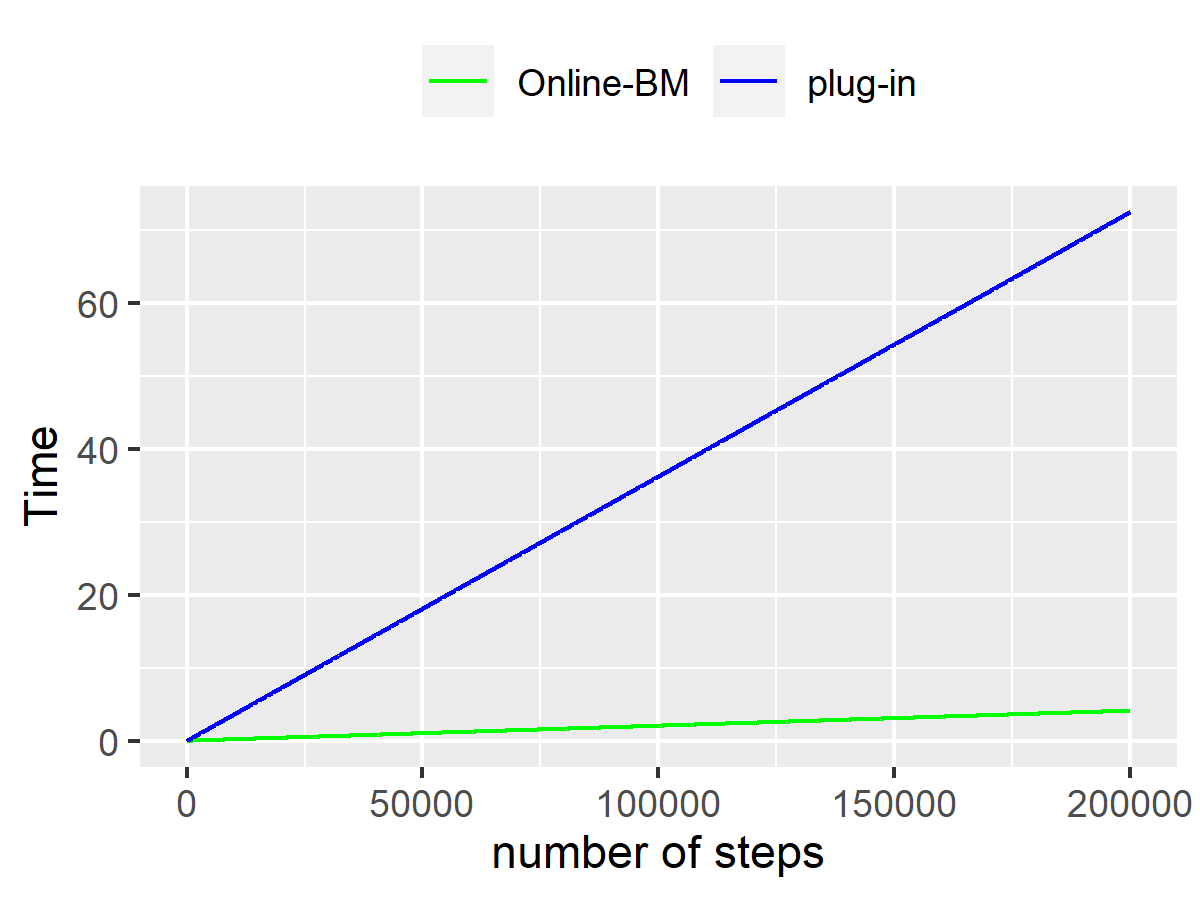} 	
	} 
	\caption{Comparison of online-BM and Plug-in estimators. First/Middle row: Empirical coverage rate against the number of steps in linear/logistic model. Red dashed line denotes the nominal coverage rate of 0.95. Third row: total computation time for updating covariance estimate and confidence intervals in SGD. }
	\label{fig:compare_P_linear}
\end{figure}

\begin{table}[!t]
	\centering	
	\caption{Empirical coverage rates: the average coverage rate for the nominal coverage probability $95\%$. Standard errors are reported in the brackets.}
	\label{table: compare_BM}
	\begin{tabular}{  l | c   c c c  }
		\hline
		\multicolumn{5}{c}{linear model }\\
		\hline	
		($d$ = 5)& $n$ = 50000 &  $n$ = 80000  & $n$ = 100000 &  $n$= 125000 \\
		\hline
		online-BM& 0.894 (0.02177)  & 0.901 (0.02114) &0.917 (0.01951)& 0.935 (0.01746)\\
		BM& 0.894 (0.02177)  & 0.904 (0.02085) & 0.910 (0.02022) & 0.928 (0.01831) \\
		\hline	
		($d$ = 20)& $n$= 50000 &  $n$= 100000  & $n$= 150000 &  $n$= 200000 \\
		\hline
		online-BM& 0.904 (0.02078) & 0.907 (0.02050)& 0.910 (0.02022)& 0.914 (0.01986) \\
		BM& 0.878 (0.02312)   & 0.901 (0.02121)  & 0.908  (0.02043) & 0.910 (0.02029) \\
		
		\hline
		\multicolumn{5}{c}{logistic model}\\
		\hline	
		($d$ = 5)& $n$ = 100000 &  $n$ = 200000  & $n$ = 300000 &  $n$ = 400000 \\
		\hline
		online-BM& 0.828 (0.01011)  & 0.844 (0.00933) &0.875 (0.00770)& 0.889 (0.00700) \\
		BM& 0.822 (0.01032)   & 0.847 (0.00919) & 0.875 (0.00771)  & 0.885 (0.00721) \\
		\hline 
		($d$ = 20)&$n$ = 100000 &  $n$ = 300000  & $n$ = 500000 &  $n$ = 700000 \\
		\hline
		online-BM& 0.791 (0.01167) & 0.829 (0.01004) &0.845 (0.00926) & 0.864 (0.00834)  \\
		BM& 0.787 (0.01188)    & 0.827 (0.01011)  & 0.839 (0.00955)   & 0.859 (0.00856) \\
		\hline 	    
	\end{tabular}
\end{table} 
Next, we compare the finite sample coverage rate of the proposed online-BM estimator and the batch means covariance matrix estimator from \citet{chen2016statistical}, which we refer to as BM.  Table \ref{table: compare_BM} shows that the finite sample coverage rates of the two estimators are close to each other in all cases, and the finite sample performance of our method slightly outperforms \cite{chen2016statistical} when $n$ is large. In fact, this is not a totally fair comparison for us since we implement the method in \cite{chen2016statistical} based on the prior knowledge of the exact sample size. 

\section{Conclusion and Future Work}
\label{sec:futu}
In this paper, we propose a fully online approach to estimate the asymptotic covariance matrix in SGD. The recursive algorithm to compute the covariance matrix estimate is computationally efficient. We demonstrate that the online batch means covariance matrix estimator (both full overlapping version and non-overlapping version) is consistent with the upper bound of convergence rate $O(n^{-(1-\alpha)/4})$ in the general case. 
Based on the estimated covariance matrix, we construct confidence intervals/regions with asymptotically correct coverage probabilities for the model parameters.  As for future directions, it would be of interest to develop a lower bound result on the online estimation of limiting covariance matrices. With such a result, we will be able to tell whether the proposed estimator is rate-optimal.  Also, as mentioned in Section \ref{sec: EL}, it would be interesting to see if one can obtain  statistics similar to the PBEL ratio based on the non-overlapping version online covariance estimator and establish a limiting distribution that can be used to calibrate confidence regions for SGD solutions without using the asymptotic normality results.

\section*{Acknowledgments}
We thank the anonymous reviewers and editors for the constructive feedbacks that significantly improved our paper. Wanrong Zhu and Wei Biao Wu would like to thank the support from NSF via NSF-DMS-1916351 and NSF-DMS-2027723.  Xi Chen would like to thank the support from NSF via IIS-1845444.

\appendix

 \section{Technical Lemmas}\label{app:1}
 \begin{lemma}\label{lemma:Y}
 	Assume that $A$ is a positive definite matrix. For any $i\in \N$, define the matrix sequence $\{Y_{i}^{j}\}$ with $Y_{i}^{i} = \I$ and for any $j>i$
 	$$Y_{i}^{j} = \prod_{k = i+1}^{j}(\I - \eta_{k}A),$$  where $\eta_{k}$ is chosen to be $\eta k^{-\alpha}$ for $\a \in (1/2, 1)$. Then we have
 	$$
 	\begin{aligned} 
 		\|Y_{i}^{j}\|_{2}\le \exp\left(-\eta\gamma\sum_{k=i+1}^{j}k^{-\alpha}\right)\le \exp\left[-\frac{\gamma\eta}{1-\alpha}\left(j^{1-\alpha} - (i+1)^{1-\alpha}\right)\right],
 	\end{aligned}
 	$$
 	where $\gamma = \min(\lambda_{\min}(A), 1/(2\eta))$. 
 \end{lemma}
 \begin{proof}
 	Since $A$ is positive definite, there exists an orthonormal matrix $Q$ and a diagonal matrix $\Lambda$ such that $A= Q\Lambda Q^{T}$. We have
 	$$
 	\begin{aligned}
 		\|Y_{i}^{j}\|_{2} &\le \prod_{k=i+1}^{j}\left\|(\I - \eta_{k}A)\right\|_{2}=\prod_{k=i+1}^{j}\left\|(\I - \eta_{k}\Lambda)\right\|_{2}
 		\le \prod_{k=i+1}^{j}\left(1- \gamma\eta k^{-\a}\right).
 	\end{aligned} 
 	$$
 	Note that $1-x\le \exp(-x) $ for any $x\in [0,1]$. So $\|Y_{i}^{j}\|_{2}$ can be further bounded as
 	$$\begin{aligned}
 		\|Y_{i}^{j}\|_{2}\le \exp\left(-\sum_{k=i+1}^{j}\gamma\eta k^{-\a}\right).
 	\end{aligned} $$
 	The lemma can be verified using the fact that $$
 	\sum_{k=i+1}^{j}k^{-\alpha}\ge\int_{i+1}^{j+1} k^{-\alpha}dk = \frac{1}{1-\alpha}\left((j+1)^{1-\alpha}-(i+1)^{1-\alpha}\right).
 	$$ 
 \end{proof}
 
 \begin{lemma}\label{lemma:S}With $Y_{i}^{j}$ defined in Lemma \ref{lemma:Y}, let $S_{i}^{j} = \sum_{k = i+1}^{j}Y_{i}^{k}$ for any $j>i$ and $S_{i}^{i} = 0$. Then we have
 	$$\begin{aligned}
 		\|S_{i}^{j}\|_{2}\lesssim (i+1)^{\alpha}.
 	\end{aligned}$$
 \end{lemma}
 \begin{proof}
 	Through triangle inequality and Lemma \ref{lemma:Y}, 
 	\begin{equation}\label{ap:1}
 		\|S_{i}^{j}\|_{2}\le\sum_{k = i+1}^{j}\|Y_{i}^{k}\|_{2}\le\sum_{k=i+1}^{j}\exp\left[-\frac{\gamma\eta}{1-\alpha}\left(k^{1-\alpha} - (i+1)^{1-\alpha}\right)\right].
 	\end{equation}
 	Note that $\exp\left(-\frac{\gamma\eta}{1-\alpha} k^{1-\alpha}\right)$ is decreasing with $k$, so	  
 	$$
 	\sum_{k=i+1}^{j}\exp\left(-\frac{\gamma\eta}{1-\alpha} k^{1-\alpha}\right)\le
 	\int_{i+1}^{j}\exp\left(-\frac{\gamma\eta}{1-\alpha}k^{1-\alpha}\right)dk\lesssim\int_{(i+1)^{1-\alpha}}^{k^{1-\alpha}}\exp\left(-\frac{\gamma\eta}{1-\alpha}t\right)t^{\alpha/(1-\alpha)}dt.
 	$$
 	For any $1\le a\le b$ and any $1 < \beta$, we have by elementary manipulation that
 	$$\begin{aligned}
 		\int_{a}^{b}e^{-x}x^{\beta}dx \le \int_{a}^{\infty}e^{-x}x^{\beta}dx\lesssim a^{\beta}e^{-a}C_{\beta},
 	\end{aligned}$$
 	where 
 	$C_{\beta}$ is a constant depending only on $\beta$. Then we have	
 	\begin{equation}\label{ap:2}
 		\sum_{k=i+1}^{j}\exp\left(-\frac{\gamma\eta}{1-\alpha} k^{1-\alpha}\right)\lesssim \exp\left(-\frac{\gamma\eta}{1-\alpha}(i+1)^{1-\alpha}\right)(i+1)^{\alpha}.
 	\end{equation}	
 	Combining (\ref{ap:1}) and (\ref{ap:2}),
 	$$\begin{aligned}
 		\|S_{i}^{j}\|_{2}&\le \exp\left(\frac{\gamma\eta}{1-\alpha}(i+1)^{1-\alpha}\right)\sum_{k=i+1}^{j}\exp\left(-\frac{\gamma\eta}{1-\alpha} k^{1-\alpha}\right) \lesssim (i+1)^{\alpha} .
 	\end{aligned}$$		
 \end{proof}
 
 \begin{lemma}\label{lemma:U}
 	With definition of $Y_{i}^{j}$ in Lemma \ref{lemma:Y}, sequence $U_{n}$ can be rewritten as
 	$$
 	U_{k} =  (\I - \eta_{k}A)U_{k-1} + \eta_{k}\epsilon_{k} 
 	= Y_{s}^{k}U_{s} + \sum_{p = s+1}^{k}Y_{p}^{k}\eta_{p}\epsilon_{p}.
 	$$
 	According to Lemma B.3 in \cite{chen2016statistical}, we have
 	$$
 	\E\|U_{k}\|_{2}^{2}\lesssim k^{-\alpha}.
 	$$
 \end{lemma}
 
 \begin{lemma}\label{lemma:l}
 	Let $a_{m} = \left\lfloor Cm^{\beta}\right\rfloor, m\ge 2$ ($a_{1} = 1$), for some constant $C>0$ and $\beta> 1/(1-\alpha)$. For $a_{M}\le n<a_{M+1}$, define $n_{m} = a_{m+1} - a_{m}, 1\le m<M$, and $n_{M} = n - a_{M} + 1$. We have
 	\begin{enumerate}
 		\item 
 		\begin{equation}
 			\lim_{M\rightarrow\infty}\frac{\sum_{i=1}^{n}l_{i}}{\sum_{i=1}^{a_{M+1}-1}l_{i}} = 1.
 		\end{equation}
 		\item  
 		\begin{equation}\label{eq:con}
 			\frac{(a_{M+1}-a_{M})^{2}}{\sum_{m = 1}^{M} (a_{m+1}-a_{m})^{2}}\lesssim M^{-1}, \ \text{and}\ \ \frac{a_{M}^{\alpha}}{n_{M}}\rightarrow 0.
 		\end{equation} 
 	\end{enumerate}
 	
 \end{lemma}
 \begin{proof}
 	Since $n\ge a_{M}$, we have
 	$$ 
 	\sum_{i=1}^{n}l_{i} \ge \sum_{i=1}^{a_{M}-1}l_{i} = \sum_{m=1}^{M-1}\sum_{i = a_{m}}^{a_{m+1}-1}(i - a_{m} +1)=\sum_{m=1}^{M-1}\frac{n_{m}(n_{m}+1)}{2}. 
 	$$
 	Also,
 	$$
 	\sum_{i=1}^{a_{M+1}-1}l_{i} = \sum_{m=1}^{M}\frac{n_{m}(n_{m}+1)}{2}.
 	$$
 	Then according to the choice of $a_{k}$, we have
 	\begin{equation}
 		\lim_{M\rightarrow\infty}\frac{\sum_{i=1}^{n}l_{i}}{\sum_{i=1}^{a_{M+1}-1}l_{i}} \ge 1 - \frac{n_{M}(n_{M}+1)}{\sum_{m=1}^{M}n_{m}(n_{m}+1)} = \lim_{M\rightarrow\infty}(1 - M^{-1}) = 1.
 	\end{equation}
 	Since $\sum_{i=1}^{n}l_{i}\le \sum_{i=1}^{a_{M+1}-1}l_{i}$, the limit is 1. Equation \eqref{eq:con} is easy to verify by using the form of $a_{k}$.
 \end{proof}


 \section{The Linear Case}\label{app:2}
 Recall that the error  $\delta_{n} = x_{n} - x^{*}$ takes the form:
 \begin{align}\label{eq:corr}
 	\delta_{n} = \delta_{n-1}- \eta_{n}\nabla F(x_{n-1}) + \eta_{n}\epsilon_{n},
 \end{align}
 where $\epsilon_{n}=\nabla F(x_{n-1})-\nabla f(x_{n-1}, \xi_{n})$. The sequence $\{\epsilon_{n}\}$ is a  martingale difference sequence since 
 \begin{equation}
 	\E_{n-1}\epsilon_{n}= \nabla F(x_{n-1}) - \E_{n-1}\nabla f(x_{n-1}, \xi_{n})=0.
 \end{equation}
 Note that $\nabla F(x^{*}) = 0$ since $x^{*}$ is the minimizer of $F(x)$. By Taylor's expansion of $\nabla F(x_{n-1})$ around $x^{*}$, we have $\nabla F(x_{n-1})\approx  \nabla A\delta_{n-1}$, where $A= \nabla^2 F(x^*)$.  Thus, modifying equation (\ref{eq:corr}) with $\nabla F(x_{n-1})$ approximated by $A\delta_{n-1}$, we have for large $n$ 
 \begin{align}\label{eq:linear}
 	\delta_{n} \approx (\I-\eta_{n}A)\delta_{n-1} + \eta_{n}\epsilon_{n}.
 \end{align} 
 Inspired by \eqref{eq:linear}, we define the linear sequence $(U_{n})_{n\in\N}$  as follows:
 \begin{align}
 	U_{n} = (\I-\eta_{n}A)U_{n-1} + \eta_{n}\epsilon_{n}, \ \ U_{0} = \delta_{0}.
 \end{align}	
 Now we define a new estimator $\tilde{\Sigma}_{n}$ based on $U_{n}$:
 \begin{equation} \tilde{\Sigma}_{n} = \frac{1}{\sum_{i=1}^{n}l_{i}}\sum_{i=1}^{n}\left(\sum_{k=t_{i}}^{i}U_{k}-l_{i}\bar{U}_{n}\right)\left(\sum_{k=t_{i}}^{i}U_{k}-l_{i}\bar{U}_{n}\right)^{T}.
 \end{equation}
 In certain cases when $\nabla F(x_{n-1}) =   \nabla A\delta_{n-1}$, such as  mean estimation model and linear regression model, error $\delta_{n}$ exactly takes the form of $U_{n}$. Then we have $\hat\Sigma_{n} = \tilde{\Sigma}_{n}$. In general cases,  we  can use $U_{n}$ to approximate $\delta_{n}$ since the difference between them is small. In other words, studying covariance matrix of $\bar{U}_{n}$ can give us insight into the covariance matrix of $\bar{x}_{n}$. Next lemma shows that the estimator $\tilde{\Sigma}_{n}$ is consistent. It can be viewed as  a special case of linear processes.

 \begin{lemma}\label{linearcase}  
 	Let $a_{m}=\left\lfloor Cm^{\beta}\right\rfloor$, where $C>0$ and $\beta>1/(1-\alpha)$.  Set step size at the $i$-th  iteration $\eta_{i} = \eta i^{-\alpha}$ with $\frac{1}{2}<\a<1$. Then under Assumptions  \ref{ass:reg} and \ref{ass:martingale},   
 	\begin{equation}
 		\E\left\|\tilde{\Sigma}_{n} - \Sigma\right\|_{2}\lesssim  M^{-\a\beta/2} + M^{-1/2} + M^{((\alpha -1)\beta +1)/2},
 	\end{equation}
 	where $M$ is the number of batches such that $a_{M}\le n<a_{M+1}$.
 \end{lemma}  
 \begin{proof} 
 	
 	Recall that 
 	$$\tilde{\Sigma}_{n} = \left(\sum_{i=1}^{n}l_{i} \right)^{-1}\sum_{i=1}^{n}\left(\sum_{k=t_{i}}^{i}U_{k}-l_{i}\bar{U}_{n}\right)\left(\sum_{k=t_{i}}^{i}U_{k}-l_{i}\bar{U}_{n}\right)^{T}.$$ Using triangle inequality we have
 	\begin{equation}\label{eq:tri1}
 		\begin{split}
 			\E\left\|\tilde{\Sigma}_{n} - \Sigma\right\|_{2}&\le \E \left\| \left(\sum_{i=1}^{n}l_{i} \right)^{-1}\sum_{i=1}^{n}\left(\sum_{k=t_{i}}^{i}U_{k}\right)\left(\sum_{k=t_{i}}^{i}U_{k}\right)^{T}- \Sigma
 			\right\|_{2}\\& + \E \left\| \left(\sum_{i=1}^{n}l_{i} \right)^{-1}\sum_{i=1}^{n}l_{i}^{2}\bar{U}_{n}\bar{U}_{n}^{T}
 			\right\|_{2} + 2\E\left\| \left(\sum_{i=1}^{n}l_{i} \right)^{-1}\sum_{i=1}^{n}\left(\sum_{k=t_{i}}^{i}U_{k}\right)\left(l_{i}\bar{U}_{n}\right)^{T}\right\|_{2}.
 		\end{split}
 	\end{equation}
 	By Lemmas \ref{lemma3}, \ref{lemma4} and \ref{lemma5} (proved in the rest of this section), all these three terms in \eqref{eq:tri1} are bounded, which implies Lemma \ref{linearcase}.
 \end{proof}
 
 Let $$\hat{S}_{n} = \left(\sum_{i=1}^{n}l_{i}\right)^{-1} \sum_{i=1}^{n}\left(\sum_{k = t_{i}}^{i}\epsilon_{k}\right)\left(\sum_{k=t_{i}}^{i}\epsilon_{k}\right)^{T}.$$ In Lemma \ref{lemma2}, we show that $\hat{S}_{n}$ converges to $S$, the covariance matrix of $\nabla{f(x^{*}, \xi)}$. Using this fact, we have Lemma \ref{lemma3}, which provides an upper bound for the first term in \eqref{eq:tri1}. The other two terms in \eqref{eq:tri1} are bounded by Lemma \ref{lemma4} and \ref{lemma5} respectively.

 \begin{lemma}\label{lemma2}
 	Let  $a_{M}\le n<a_{M+1}$. Under conditions in Lemma \ref{linearcase}, we have  
 	\begin{equation}
 		\E \left\|\hat{S}_{n}- S
 		\right\|_{2}
 		\lesssim
 		M^{-\a\beta/2} + M^{-1/2}.
 	\end{equation}

 \end{lemma} 
 \begin{proof}
 	Here we introduce sequence $\{\epsilon_{n}^{*}\}$ as follows 
 	$$\epsilon_{n}^{*} = \nabla F(x^{*})-\nabla f(x^{*},\xi_{n}) = -\nabla f(x^{*},\xi_{n}), n\ge 1.$$
 	Note that $\{\epsilon_{n}^{*}\}$ is  a sequence of  $\emph{i.i.d}$ variables with mean $0$, and therefore $\{\epsilon_{n} - \epsilon_{n}^{*}\}$ is still a martingale difference sequence. We further define
 	$$\hat{S}_{n}^{*} = \left(\sum_{i=1}^{n}l_{i}\right)^{-1} \sum_{i=1}^{n}\left(\sum_{k = t_{i}}^{i}\epsilon^{*}_{k}\right)\left(\sum_{k=t_{i}}^{i}\epsilon^{*}_{k}\right)^{T}.$$
 	Then we can bound $\E\|{\hat{S}_{n}} - S\|_{2}$ through triangle inequality 
 	\begin{equation}
 		\E\|{\hat{S}_{n}} - S\|_{2}\le \E\|\hat{S}^{*}_{n} - S\|_{2} + \E\|\hat{S}_{n} - \hat{S}^{*}_{n}\|_{2}.
 	\end{equation}
 	
 	\bigskip
 	\noindent{\bf Step 1:} Bound $\E\|\hat{S}^{*}_{n} - S\|_{2}$. 	\\
 	Since $\hat{S}_{n}^{*} - S$ is symmetric, 	
 	\begin{equation} 
 		\E\|\hat{S}_{n}^{*} - S\|_{2} = \E|\lambda_{max}(\hat{S}^{*}_{n} - S)| 
 		=  \E\sqrt{\lambda_{max}(\hat{S}_{n}^{*} - S)^{2}}. 
 	\end{equation}
 	Note that $(\hat{S}^{*}_{n} - S)^{2}$ is positive semidefinite. For any positive semidefinite matrix $C$ we have $\lambda_{max}(C)\le\text{tr}(C)\le d\|C\|_{2}$. So $\lambda_{max}(\hat{S}_{n}^{*} - S)^{2}\le \text{tr}(\hat{S}_{n}^{*} - S)^{2}$. Further using Jensen's inequality, we have
 	\begin{equation} 
 		\E\|\hat{S}_{n}^{*} - S\|_{2} 
 		\le  \E\sqrt{\text{tr}(\hat{S}_{n}^{*} - S)^{2}}
 		\le  \sqrt{\text{tr}\E(\hat{S}_{n}^{*} - S)^{2}}
 		\le  \sqrt{d\|\E(\hat{S}_{n}^{*}-S)^{2}\|_{2}}.
 	\end{equation}  
 	Note that by definition of $S$,
 	$$\E(\hat{S}_{n}^{*}) = \left(\sum_{i=1}^{n}l_{i}\right)^{-1} \sum_{i=1}^{n}\sum_{k = t_{i}}^{i}\E\epsilon^{*}_{k}\epsilon^{*T}_{k} = S.$$
 	Then 
 	$$
 	\|\E(\hat{S}_{n}^{*}-S)^{2}\|_{2} = \|\E\hat{S}_{n}^{*2}-S^{2}\|_{2}.
 	$$
 	Note that  $\E(\epsilon^{*}_{p_{1}}\epsilon_{p_{2}}^{*T}\epsilon_{p_{3}}^{*}\epsilon_{p_{4}}^{*T})$ is nonzero if and only if for any $r$ there exist $r'\ne r$ such that $p_{r} = p_{r'}$, $r, r' \in \{1,2,3,4\}$. There are two cases we can consider. The first case is $p_{1}=p_{3}\ne p_{2}=p_{4}$ or $p_{1}=p_{4}\ne p_{2}=p_{3}$. This requires $i$ and $j$ in the same block. The second case is $p_{1}=p_{2}$ and $p_{3}=p_{4}$. We can expand $\E{\hat{S}_{n}}^{*2}$ and rewrite it into two parts,
 	\begin{equation}\label{ap:6}
 		\begin{split}
 			\E{\hat{S}_{n}}^{*2} = & \E \left(\sum_{i=1}^{n}l_{i}\right)^{-2}  \sum_{1\le i,j \le n}\left(\sum_{k = t_{i}}^{i}\epsilon_{k}^{*}\right)\left(\sum_{k=t_{i}}^{i}\epsilon_{k}^{*}\right)^{T}\left(\sum_{k = t_{j}}^{j}\epsilon_{k}^{*}\right)\left(\sum_{k=t_{j}}^{j}\epsilon_{k}^{*}\right)^{T}\\
 			=&\left(\sum_{i=1}^{n}l_{i}\right)^{-2}I + \left(\sum_{i=1}^{n}l_{i}\right)^{-2} II ,
 		\end{split}
 	\end{equation}
 	where 
 	$$
 	\begin{aligned}
 		I = &\E\sum_{m = 1}^{M-1}\sum_{i = a_{m}}^{a_{m+1 }-1}\left[2\sum_{j = a_{m}}^{i-1}\sum_{a_{m}\le p_{1}\ne p_{2}\le j}\left(\epsilon_{p_{1}}^{*}\epsilon_{p_{2}}^{*T}\epsilon_{p_{1}}^{*}\epsilon_{p_{2}}^{*T} + \epsilon_{p_{1}}^{*}\epsilon_{p_{2}}^{*T}\epsilon_{p_{2}}^{*}\epsilon_{p_{1}}^{*T} \right) + \sum_{a_{m}\le p_{1}\ne p_{2}\le i}\left(\epsilon_{p_{1}}^{*}\epsilon_{p_{2}}^{*T}\epsilon_{p_{1}}^{*}\epsilon_{p_{2}}^{*T} + \epsilon_{p_{1}}^{*}\epsilon_{p_{2}}^{*T}\epsilon_{p_{2}}^{*}\epsilon_{p_{1}}^{*T}\right)\right]\\
 		+& \E \sum_{i = a_{M}}^{n}\left[2\sum_{j = a_{M}}^{i-1}\sum_{a_{M}\le p_{1}\ne p_{2}\le j}\left(\epsilon_{p_{1}}^{*}\epsilon_{p_{2}}^{*T}\epsilon_{p_{1}}^{*}\epsilon_{p_{2}}^{*T} + \epsilon_{p_{1}}^{*}\epsilon_{p_{2}}^{*T}\epsilon_{p_{2}}^{*}\epsilon_{p_{1}}^{*T}\right) + \sum_{a_{M}\le p_{1}\ne p_{2}\le i}\left(\epsilon_{p_{1}}^{*}\epsilon_{p_{2}}^{*T}\epsilon_{p_{1}}^{*}\epsilon_{p_{2}}^{*T} + \epsilon_{p_{1}}^{*}\epsilon_{p_{2}}^{*T}\epsilon_{p_{2}}^{*}\epsilon_{p_{1}}^{*T}\right)\right],
 	\end{aligned}$$
 	and
 	$$
 	\begin{aligned}
 		II =  \sum_{i = 1}^{n}\sum_{j = 1}^{n}\sum_{p = t_{i}}^{i}\sum_{q = t_{j}}^{j}\E(\epsilon_{p}^{*}\epsilon_{p}^{*T}\epsilon_{q}^{*}\epsilon_{q}^{*T})  .
 	\end{aligned}$$
 	Let  $\|\E(\epsilon_{p_{1}}^{*}\epsilon_{p_{2}}^{*T}\epsilon_{p_{3}}^{*}\epsilon_{p_{4}}^{*T})\|_{2}$ be bounded by constant $C$ for any $p_{r}$,$r\in \{1,2,3,4\}$. Then we can bound $I$ as follows, 
 	\begin{equation}
 		\begin{split}
 			\|I\|_{2}\le &\sum_{m = 1}^{M}\sum_{i = a_{m}}^{a_{m+1 }-1}\left[2\sum_{j = a_{m}}^{i-1}\sum_{a_{m}\le p_{1}\ne p_{2}\le j}\left(C + C\right) + \sum_{a_{m}\le p_{1}\ne p_{2}\le i}\left(C+C\right)\right]\\
 			\lesssim &\sum_{m = 1}^{M}\sum_{i = a_{m}}^{a_{m+1 }-1}\left(1\times2 + 2\times3 +...+ (l_{i}-1)\times l_{i}\right)\\
 			\lesssim &\sum_{m = 1}^{M}\sum_{i = a_{m}}^{a_{m+1 }-1} l_{i}^{3}
 			\lesssim \sum_{m = 1}^{M} n_{m}^{4}.
 		\end{split}
 	\end{equation}
 	Since $\sum_{i=1}^{n}l_{i}\asymp \sum_{m = 1}^{M}\sum_{i = a_{m}}^{a_{m+1 }-1} l_{i}\asymp \sum_{m = 1}^{M}n_{m}^{2}$ and $ n_{M}^{2}/\sum_{m = 1}^{M} n_{m}^{2} \lesssim M^{-1}$ ,  we have
 	\begin{equation}\label{ap:7}
 		\left(\sum_{i=1}^{n}l_{i}\right)^{-2}\|I\|_{2}\lesssim \frac{\sum_{m = 1}^{M} n_{m}^{4}}{(\sum_{m = 1}^{M} n_{m}^{2})^{2}}\lesssim \frac{\max_{1\le m\le M} n_{m}^{2}}{\sum_{m = 1}^{M} n_{m}^{2}} \lesssim M^{-1}.
 	\end{equation}
 	Next, note that $\sum_{i = 1}^{n}\sum_{j = 1}^{n}\sum_{p = t_{i}}^{i}\sum_{q = t_{j}}^{j}1 = \left(\sum_{i = 1}^{n}l_{i}\right)^{2}$. Then,
 	\begin{equation}\label{ap:5}
 		\begin{split}
 			\left\|\left(\sum_{i = 1}^{n}l_{i}\right)^{-2}II - S^{2}\right\|_{2}\le & \left(\sum_{i = 1}^{n}l_{i}\right)^{-2} \sum_{i = 1}^{n}\sum_{j = 1}^{n}\sum_{p = t_{i}}^{i}\sum_{q = t_{j}}^{j}\left\|\E(\epsilon_{p}^{*}\epsilon_{p}^{*T}\epsilon_{q}^{*}\epsilon_{q}^{*T})-S^{2}\right\|_{2}\\
 			\lesssim& \left(\sum_{i = 1}^{a_{M + 1}-1}l_{i}\right)^{-2}\sum_{m = 1}^{M}\sum_{k = 1}^{M}\sum_{i = a_{m}}^{a_{m+1}-1}\sum_{j = a_{k}}^{a_{k+1}-1}\sum_{p = a_{m}}^{i}\sum_{q = a_{k}}^{j}\left\|\E(\epsilon_{p}^{*}\epsilon_{p}^{*T}\epsilon_{q}^{*}\epsilon_{q}^{*T}) - S^{2}\right\|_{2}. 
 		\end{split}
 	\end{equation}				
 	We consider two cases here. One is when $p$ and $q$ are in the same block. Let 
 	$$
 	III = \sum_{m = 1}^{M}\sum_{i = a_{m}}^{a_{m+1}-1}\sum_{j = a_{m}}^{a_{m+1}-1}\sum_{p = a_{m}}^{i}\sum_{q = a_{m}}^{j}\left\|\E(\epsilon_{p}^{*}\epsilon_{p}^{*T}\epsilon_{q}^{*}\epsilon_{q}^{*T}) - S^{2}\right\|_{2}.
 	$$ 
 	Here $\|\E(\epsilon_{p}^{*}\epsilon_{p}^{*T}\epsilon_{q}^{*}\epsilon_{q}^{*T})\|_{2}$ is still bounded by constant $C$. Then we have
 	\begin{equation}\label{ap:3}
 		\begin{split}
 			\left(\sum_{i = 1}^{a_{M + 1}-1}l_{i}\right)^{-2} III  \le& \left(\sum_{i = 1}^{a_{M + 1}-1}l_{i}\right)^{-2}\sum_{m = 1}^{M}\sum_{i = a_{m}}^{a_{m+1}-1}\sum_{j = a_{m}}^{a_{m+1}-1}\sum_{p = a_{m}}^{i}\sum_{q = a_{m}}^{j}\left(C +\left\| S^{2}\right\|_{2}\right)\\
 			\lesssim&\left(\sum_{i = 1}^{a_{M + 1}-1}l_{i}\right)^{-2}\sum_{m=1}^{M}\left(\sum_{i=a_{m}}^{a_{m+1}-1}l_{i}\right)^{2}\\
 			\lesssim &\frac{\sum_{m = 1}^{M} n_{m}^{4}}{(\sum_{m = 1}^{M} n_{m}^{2})^{2}}\lesssim \frac{\max_{1\le m\le M} n_{m}^{2}}{\sum_{m = 1}^{M} n_{m}^{2}} \lesssim  M^{-1}.
 		\end{split}
 	\end{equation} 
 	The other case is when $p$ and $q$ are in different blocks. Let 
 	$$
 	IV = \sum_{m\ne k}\sum_{j=a_{k}}^{a_{k+1}-1}\sum_{i=a_{m}}^{a_{m+1}-1}\sum_{q=a_{k}}^{j}\sum_{p=a_{m}}^{i}\left\|\E(\epsilon_{p}^{*}\epsilon_{p}^{*T}\epsilon_{q}^{*}\epsilon_{q}^{*T}) - S^{2}\right\|_{2}.
 	$$
 	Note that $\E(\epsilon_{n}^{*}\epsilon_{n}^{*T}) = S$ by definition of $S$ and $\epsilon_{p}^{*}\epsilon_{p}^{*T}$ is independent of $\epsilon_{q}^{*}\epsilon_{q}^{*T}$, $\forall p>q$. Then for $p>q$,  
 	\begin{equation}
 		\left\|\E(\epsilon_{p}^{*}\epsilon_{p}^{*T}\epsilon_{q}^{*}\epsilon_{q}^{*T}) - S^{2}\right\|_{2} = 	\left\|\E(\epsilon_{p}^{*}\epsilon_{p}^{*T})\E(\epsilon_{q}^{*}\epsilon_{q}^{*T}) - S^{2}\right\|_{2} = 0.
 	\end{equation}
 	Then we have
 	\begin{equation}\label{ap:4}
 		\left(\sum_{i = 1}^{a_{M + 1}-1}l_{i}\right)^{-2}IV=0
 	\end{equation}		
 	Combining (\ref{ap:5}), (\ref{ap:3}) and (\ref{ap:4}), we have
 	\begin{equation}\label{ap:8}
 		\left\|\left(\sum_{i = 1}^{n}l_{i}\right)^{-2}II - S^{2}\right\|_{2}\lesssim \left(\sum_{i = 1}^{a_{M + 1}-1}l_{i}\right)^{-2}III  + \left(\sum_{i = 1}^{a_{M + 1}-1}l_{i}\right)^{-2}IV\lesssim M^{-1}.
 	\end{equation}
 	Further combining (\ref{ap:6}), (\ref{ap:7}) and (\ref{ap:8}), we have
 	\begin{equation} 
 		\|\E\hat{S}_{n}^{*2} - S^{2}\|\le 	\left\|\left(\sum_{i=1}^{n}l_{i}\right)^{-2}II - S^{2}\right\|_{2} + 	\left(\sum_{i=1}^{n}l_{i}\right)^{-2}\left\| I\right\|_{2}
 		\lesssim M^{-1}. 
 	\end{equation}
 	Therefore 
 	$$
 	\E\|\hat{S}_{n}^{*} - S\|_{2} 	\le  \sqrt{d\|\E(\hat{S}_{n}^{*}-S)^{2}\|_{2}}= \sqrt{d\|\E\hat{S}_{n}^{*2}-S^{2}\|_{2}} \lesssim M^{-1/2}.$$
 	
 	\bigskip
 	\noindent{\bf Step 2:} Bound $\E\|\hat{S}_{n} - \hat{S}^{*}_{n}\|_{2}$. 	\\ 
 	Let $v_{k} = \epsilon_{k} - \epsilon_{k}^{*}, k\ge 1$.  We can expand $\E\|\hat{S}_{n} - \hat{S}^{*}_{n}\|_{2}$ as 		
 	\begin{equation} 
 		\begin{split}
 			&\E\|\hat{S}_{n} - \hat{S}^{*}_{n}\|_{2} = \E\left\|\left(\sum_{i=1}^{n}l_{i}\right)^{-1} \sum_{i=1}^{n}\left[\left(\sum_{k = t_{i}}^{i}\epsilon_{k}\right)\left(\sum_{k=t_{i}}^{i}\epsilon_{k}\right)^{T} - \left(\sum_{k = t_{i}}^{i}\epsilon^{*}_{k}\right)\left(\sum_{k=t_{i}}^{i}\epsilon^{*}_{k}\right)^{T}\right]\right\|\\
 			\le &2\E\left\|\left(\sum_{i=1}^{n}l_{i} \right)^{-1}\sum_{i=1}^{n}\left(\sum_{k=t_{i}}^{i}v_{k}\right)\left(\sum_{k=t_{i}}^{i}\epsilon_{k}^{*}\right)^{T}\right\|_{2} +\E\left\|\left(\sum_{i=1}^{n}l_{i} \right)^{-1}\sum_{i=1}^{n}\left(\sum_{k=t_{i}}^{i}v_{k}\right)\left(\sum_{k=t_{i}}^{i}v_{k}\right)^{T}\right\|_{2}.
 		\end{split}
 	\end{equation} 
 	Apply Cauchy's inequality   
 	\begin{equation}
 		\begin{split}
 			&\E\left\|\left(\sum_{i=1}^{n}l_{i} \right)^{-1}\sum_{i=1}^{n}\left(\sum_{k=t_{i}}^{i}v_{k}\right)\left(\sum_{k=t_{i}}^{i}\epsilon_{k}^{*}\right)^{T}\right\|_{2}\\
 			&\le  \sqrt{\E\|\hat{S}^{*}_{n}\|_{2}}\sqrt{\E\left\|\left(\sum_{i=1}^{n}l_{i} \right)^{-1}\sum_{i=1}^{n}\left(\sum_{k=t_{i}}^{i}v_{k}\right)\left(\sum_{k=t_{i}}^{i}v_{k}\right)^{T}\right\|_{2}}.
 		\end{split}
 	\end{equation}   
 	Then we only need to bound $\E\left\|\left(\sum_{i=1}^{n}l_{i} \right)^{-1}\sum_{i=1}^{n}\left(\sum_{k=t_{i}}^{i}v_{k}\right)\left(\sum_{k=t_{i}}^{i}v_{k}\right)^{T}\right\|_{2}$.
 	By triangle inequality and the fact $\|C\|_{2}\le\text{tr}(C)$ for any positive semi-definite matrix $C$,
 	\begin{equation} 
 		\begin{split}
 			\E\left\|\left(\sum_{i=1}^{n}l_{i} \right)^{-1}\sum_{i=1}^{n}\left(\sum_{k=t_{i}}^{i}v_{k}\right)\left(\sum_{k=t_{i}}^{i}v_{k}\right)^{T}\right\|_{2} 
 			\le&\left(\sum_{i=1}^{n}l_{i}\right)^{-1}\sum_{i=1}^{n}\E\text{tr}\left(\left(\sum_{k=t_{i}}^{i}v_{k}\right)\left(\sum_{k=t_{i}}^{i}v_{k}\right)^{T} \right)\\
 			=&\left(\sum_{i=1}^{n}l_{i}\right)^{-1}\sum_{i=1}^{n}\E\left\|\sum_{k=t_{i}}^{i}v_{k}\right\|_{2}^{2}. 
 		\end{split}
 	\end{equation} 
 	Note that the sequence $\{v_{k}\}$ is still a martingale difference sequence since $$\E_{k-1}v_{k} = \E_{k-1}\epsilon_{k} - \E_{k-1}\epsilon_{k}^{*} = 0.$$ 
 	Then we have
 	$$
 	\E \|\sum_{k=t_{i}}^{i}v_{k} \|_{2}^{2}= \sum_{k=t_{i}}^{i}\E\|v_{k}\|_{2}^{2}.
 	$$
 	We also have
 	\begin{equation}
 		\begin{split} 
 			\E\|v_{k}\|_{2}^{2} &= \E\|\epsilon_{k} - \epsilon_{k}^{*}\|_{2}^{2} 
 			= \E\|\nabla F(x_{k-1}) - \nabla F(x^{*}) - (\nabla f(x_{k-1}, \xi_{k}) - \nabla f(x^{*}, \xi_{k}))\|_{2}^{2}\\
 			&\le  2\E\|\nabla F(x_{k-1}) - \nabla F(x^{*})\|_{2}^{2} + 2 \E\|\nabla f(x_{k-1}, \xi_{k}) - \nabla f(x^{*}, \xi_{k})\|_{2}^{2}\\
 			&\lesssim \E\|x_{k-1} - x^{*}\|_{2}^{2} \lesssim (k-1)^{-\alpha}.
 		\end{split}
 	\end{equation} 
 	The second last inequality comes from Lipschitz continuity of objective function (here we also assume $f(x, \xi)$ is Lipschitz continuous with respect to the first argument $x$). Last inequality comes from Lemma \ref{lemma:1}. Then we have
 	\begin{equation} 
 		\begin{split}
 			&\E\left\|\left(\sum_{i=1}^{n}l_{i} \right)^{-1}\sum_{i=1}^{n}\left(\sum_{k=t_{i}}^{i}v_{k}\right)\left(\sum_{k=t_{i}}^{i}v_{k}\right)^{T}\right\|_{2} 
 			\le \left(\sum_{i=1}^{n}l_{i}\right)^{-1}\sum_{i=1}^{n}\E\left\|\sum_{k=t_{i}}^{i}v_{k}\right\|_{2}^{2}\\
 			\le&  \left(\sum_{i=1}^{n}l_{i}\right)^{-1}\sum_{i=1}^{n}\sum_{k=t_{i}}^{i}\E\|v_{k}\|_{2}^{2}
 			\lesssim  \left(\sum_{i=1}^{n}l_{i}\right)^{-1}\sum_{i=1}^{n}\sum_{k=t_{i}}^{i}(k-1)^{-\alpha}\\
 			\le & \left(\sum_{i = 1}^{n}l_{i}\right)^{-1} \sum_{m = 1}^{M}\sum_{i=a_{m}}^{a_{m + 1}-1}l_{i}(a_{m}-1)^{-\alpha}.
 		\end{split}
 	\end{equation} 
 	Since $$\sum_{i = 1}^{n}l_{i} \asymp \sum_{m = 1}^{M}n_{m}^{2}, \sum_{m = 1}^{M}\sum_{i=a_{m}}^{a_{m + 1}-1}l_{i}(a_{m}-1)^{-\alpha}\asymp \sum_{m = 1}^{M}n_{m}^{2}a_{m}^{-\alpha},$$
 	we have 
 	$$
 	\E\left\|\left(\sum_{i=1}^{n}l_{i} \right)^{-1}\sum_{i=1}^{n}\left(\sum_{k=t_{i}}^{i}v_{k}\right)\left(\sum_{k=t_{i}}^{i}v_{k}\right)^{T}\right\|_{2} \lesssim M^{-\alpha\beta}.
 	$$	 
 	Then
 	$$
 	\E\|\hat{S}_{n} - \hat{S}^{*}_{n}\|_{2}\lesssim M^{-\a\beta/2}.
 	$$
 	
 	Finally,  we  reach the result
 	\begin{equation}
 		\begin{split}
 			\E\|\hat{S}_{n}-S\|_{2}\lesssim \E\|\hat{S}_{n}^{*} - S\|_{2} +  \E\|\hat{S}_{n} - \hat{S}^{*}_{n}\|_{2}\lesssim M^{-\a\beta/2} + M^{-1/2}.
 		\end{split}
 	\end{equation}
 \end{proof}
 
 \begin{lemma}\label{lemma3}
 	Under conditions in  Lemma \ref{linearcase},  we have
 	\begin{equation}
 		\E \left\|  \left(\sum_{i=1}^{n}l_{i}\right)^{-1}\sum_{i=1}^{n}\left(\sum_{k=t_{i}}^{i}U_{k}\right)\left(\sum_{k=t_{i}}^{i}U_{k}\right)^{T}- \Sigma
 		\right\|_{2}
 		\le
 		M^{-\a\beta/2} + M^{-1/2} + M^{((\a - 1)\beta+1)/2},
 	\end{equation}
 	where $a_{M}\le n < a_{M+1}$.
 \end{lemma}
 \begin{proof}
 	With the formula of $U_{k}$ in Lemma \ref{lemma:U}, for $k\in [t_{i}, i]$ we have
 	$$
 	U_{k} = Y_{t_{i}-1}^{k}U_{t_{i}-1} + \sum_{p = t_{i}}^{k}Y_{p}^{k}\eta_{p}\epsilon_{p}.
 	$$
 	With definition of $S_{j}^{k}$, we have
 	$$
 	\sum_{k=t_{i}}^{i}U_{k} = \sum_{k=t_{i}}^{i}\left(Y_{t_{i}-1}^{k}U_{t_{i}-1} + \sum_{p = t_{i}}^{k}Y_{p}^{k}\eta_{p}\epsilon_{p}\right) = S_{t_{i}-1}^{i}U_{t_{i}-1} + \sum_{p =t_{i}}^{i}(\I + S_{p}^{i})\eta_{p}\epsilon_{p}.
 	$$
 	Then we have the following expansion:
 	\begin{equation}
 		\begin{split}
 			&\left(\sum_{i=1}^{n}l_{i}\right)^{-1}\sum_{i=1}^{n}\left(\sum_{k=t_{i}}^{i}U_{k}\right)\left(\sum_{k=t_{i}}^{i}U_{k}\right)^{T}\\ 
 			= & \left(\sum_{i=1}^{n}l_{i}\right)^{-1}\sum_{i=1}^{n}\left(S_{t_{i}-1}^{i}U_{t_{i}-1} + \sum_{p =t_{i}}^{i}(\I + S_{p}^{i})\eta_{p}\epsilon_{p}\right)\left(S_{t_{i}-1}^{i}U_{t_{i}-1} + \sum_{p =t_{i}}^{i}(\I + S_{p}^{i})\eta_{p}\epsilon_{p}\right)^{T}\\
 			= & \left(\sum_{i=1}^{n}l_{i}\right)^{-1} \sum_{i=1}^{n}\left( A^{-1}\left(\sum_{p = t_{i}}^{i}\epsilon_{p}\right)\left(\sum_{p = t_{i}}^{i}\epsilon_{p}\right)^{T}A^{-1} + B_{i}A_{i}^{T} + A_{i}B_{i}^{T} + B_{i}B_{i}^{T}\right), 
 		\end{split}
 	\end{equation}		
 	where $A_{i} = \sum_{p = t_{i}}^{i}A^{-1}\epsilon_{p}$ and $B_{i} = S_{t_{i}-1}^{i}U_{t_{i}-1} + \sum_{p = t_{i}}^{i}(\eta_{p}S_{p}^{i}+\eta_{p}\I-A^{-1})\epsilon_{p}$. We then have
 	\begin{equation}
 		\begin{split}
 			&\E \left\| \left(\sum_{i=1}^{n}l_{i}\right)^{-1}\sum_{i=1}^{n}\left(\sum_{k=t_{i}}^{i}U_{k}\right)\left(\sum_{k=t_{i}}^{i}U_{k}\right)^{T}- \Sigma
 			\right\|_{2}\\
 			\lesssim & E\left\|\left(\sum_{i=1}^{n}l_{i}\right)^{-1}\sum_{i=1}^{n}\left( A^{-1}\left(\sum_{p = t_{i}}^{i}\epsilon_{p}\right)\left(\sum_{p = t_{i}}^{i}\epsilon_{p}\right)^{T}A^{-1} -\Sigma\right)\right\|_{2}\\& + \E\left\|\left(\sum_{i=1}^{n}l_{i}\right)^{-1}\sum_{i = 1}^{n}B_{i}A_{i}^{T}\right\|_{2}  +\E\left\|\left(\sum_{i=1}^{n}l_{i}\right)^{-1}\sum_{i = 1}^{n}  B_{i}B_{i}^{T}\right\|_{2}\\
 			= & I + II + III.
 		\end{split}
 	\end{equation}
 	It is suffices to show that all three parts above can be bounded. 
 	Recall that $$\hat{S}_{n} = \left(\sum_{i=1}^{n}l_{i}\right)^{-1} \sum_{i=1}^{n}\left(\sum_{k = t_{i}}^{i}\epsilon_{k}\right)\left(\sum_{k=t_{i}}^{i}\epsilon_{k}\right)^{T},$$ and $\Sigma = A^{-1}SA^{-1}$. We can  bound $I$ using Lemma \ref{lemma2}.		
 	\begin{equation}
 		I \le \|A^{-1}\|_{2}^{2}\E\|\hat{S}_{n}-S\|_{2}
 		\lesssim M^{-\a\beta/2} + M^{-1/2}.
 	\end{equation}
 	For the third part $III$, since $B_{i}B_{i}^{T}$ is positive semi-definite, we have 
 	$$\begin{aligned}
 		\E\|B_{i}B_{i}^{T}\|_{2}\le\E\text{tr}(B_{i}B_{i}^{T})=\text{tr}(\E(B_{i}B_{i}^{T}))\le d\|\E(B_{i}B_{i}^{T})\|_{2}.
 	\end{aligned}$$
 	Since $\epsilon_{p}$ are martingale differences, we have $\E(U_{a_{m}-1} \epsilon_{p}^{T})=0$ for any $p\ge a_{m}$ and $\E(\epsilon_{p_{1}}\epsilon_{p_{2}}^{T})=0$ for any $p_{1}\ne p_{2}$. So, 
 	\begin{equation}\label{ap:19}
 		\begin{split}
 			\left\|\E(B_{i}B_{i}^{T})\right\|_{2} = & \left\|S_{a_{m}-1}^{i}\E (U_{a_{m}-1}{U_{a_{m}-1}}^{T}){S_{a_{m}-1}^{i}}^{T} + \sum_{p = a_{m}}^{i}(\eta_{p}S_{p}^{i}+\eta_{p}\I-A^{-1})\E(\epsilon_{p}\epsilon_{p}^{T})(\eta_{p}S_{p}^{i}+\eta_{p}\I-A^{-1})^{T}\right\|_{2}\\
 			\le&\left\|S_{a_{m}-1}^{i}\right\|_{2}^{2}\left\|\E(U_{a_{m}-1}U_{a_{m}-1}^{T})\right\|_{2} + \sum_{p = a_{m}}^{i}\left\|\eta_{p}S_{p}^{i}+\eta_{p}\I-A^{-1}\right\|_{2}^{2}\left\|\E(\epsilon_{p}\epsilon_{p}^{T})\right\|_{2}.
 		\end{split}
 	\end{equation}	
 	From Lemmas \ref{lemma:S} and \ref{lemma:U}, we can see that $\left\|S_{a_{m}-1}^{i}\right\|_{2}^{2}\lesssim a_{m}^{2\alpha}$ and
 	$$\begin{aligned}
 		\|\E(U_{a_{m}-1}U_{a_{m}-1}^{T})\|_{2}\lesssim \text{tr}\E(U_{a_{m}-1}U_{a_{m}-1}^{T})\lesssim \E\text{tr}(U_{a_{m}-1}U_{a_{m}-1}^{T})\lesssim E\|U_{a_{m}-1}\|_{2}^{2}\lesssim (a_{m}-1)^{-\alpha}.
 	\end{aligned}$$
 	So we have 
 	$$
 	\left\|S_{a_{m}-1}^{i}\right\|_{2}^{2}\left\|\E(U_{a_{m}-1}U_{a_{m}-1}^{T})\right\|_{2} \lesssim a_{m}^{\alpha}.
 	$$
 	For the remaining part in (\ref{ap:19}), $\left\|\E(\epsilon_{p}\epsilon_{p}^{T})\right\|_{2}$ is bounded and	
 	\begin{equation} 
 		\sum_{p=a_{m}}^{i}\|\eta_{p}S_{p}^{i}+\eta_{p}\I-A^{-1}\|_{2}^{2}\lesssim \sum_{p = a_{m}}^{i}\left(\|\eta_{p}S_{p}^{i}-A^{-1}\|_{2}^{2} + \|\eta_{p}\I\|_{2}^{2}\right) .
 	\end{equation}
 	
 	Next, we need to bound $\|\eta_{p}S_{p}^{i}-A^{-1}\|_{2}^{2}$.  
 	When $\eta_{j} = \eta j^{-\alpha}$ and $a_{m}\le p \le i< a_{m+1}$, based on Lemma D.2 (3) in \cite{chen2016statistical}, we have
 	$$
 	\|\eta_{p}S_{p}^{i}-A^{-1}\|_{2}^{2}\lesssim  p^{2\alpha-2}+\exp\left(-2\gamma\sum_{j=p}^{i}\eta_{j}\right). 
 	$$ 
 	Also,
 	$$
 	\sum_{p=a_{m}}^{i}\exp\left(-2\gamma\sum_{j=p}^{i}\eta_{j}\right)\le \sum_{p=a_{m}}^{i}\exp\left(-2\gamma\eta(i-p)i^{-\alpha}\right) \le \sum_{k=0}^{\infty}\exp\left(-2\gamma\eta i^{-\alpha}k\right).
 	$$
 	Note that $\int_{0}^{\infty}e^{-ax}dx = a^{-1}$. Then we can use integration to bound the summation above as 
 	$$
 	\sum_{k=0}^{\infty}\exp\left(-2\gamma\eta i^{-\alpha}k\right) \le\int_{0}^{\infty}\exp\left(-2\gamma\eta i^{-\alpha}k\right) \lesssim i^{\alpha}.
 	$$
 	Furthermore, $p^{2\alpha-2}\ge p^{-2\alpha}$ since $\alpha > 1/2$. So
 	$$
 	\sum_{p = a_{m}}^{i}\left\|\eta_{p}S_{p}^{i}+\eta_{p}\I-A^{-1}\right\|_{2}^{2}\lesssim l_{i}a_{m}^{2\alpha-2} + i^{\alpha}.
 	$$	
 	Recall the definition of $B_{i}$, when $t_{i}= a_{m}$	
 	\begin{equation}
 		\|\E(B_{i}B_{i}^{T})\|_{2} \lesssim i^{\alpha} + l_{i}a_{m}^{2\alpha-2} .
 	\end{equation}	
 	Now since $\sum_{i=1}^{n}l_{i}\asymp \sum_{m=1}^{M}n_{m}^{2}$, we can bound $III$ as follows:		
 	\begin{equation}
 		\begin{split}
 			III
 			&\lesssim \left(\sum_{m=1}^{M}n_{m}^{2}\right)^{-1}\sum_{m = 1}^{M}\sum_{i=a_{m}}^{a_{m+1} - 1}\E\left\| B_{i}B_{i}^{T}\right\|_{2}\\
 			&\le \left(\sum_{m=1}^{M}n_{m}^{2}\right)^{-1}\sum_{m = 1}^{M}\sum_{i=a_{m}}^{a_{m+1} - 1}\left(i^{\alpha} + l_{i}a_{m}^{2\alpha-2}\right)\\
 			&\lesssim \left(\sum_{m=1}^{M}n_{m}^{2}\right)^{-1}\sum_{m=1}^{M}\left(n_{m}^{2}a_{m}^{2\alpha-2}+n_{m}a_{m}^{\alpha}\right).
 		\end{split}
 	\end{equation} 
 	Recall that $a_{m}\asymp m^{\beta}$ and $n_{m}\asymp m^{\beta-1}$, we then have  
 	\begin{equation}
 		III\lesssim a_{M}^{2\alpha-2} + \frac{a_{M}^{\alpha}}{n_{M}}\lesssim M^{(\a-1)\beta+1}.
 	\end{equation}
 	For the second part, using Cauchy's inequality we have
 	\begin{equation} 
 		II\le\text{\footnotemark}\sqrt{\frac{\sum_{i = 1}^{n} \E\left\|A_{i}A_{i}^{T}\right\|_{2}}{\sum_{i=1}^{n}l_{i}}\frac{\sum_{i = 1}^{n} \E\left\|B_{i}B_{i}^{T}\right\|_{2}}{\sum_{i=1}^{n}l_{i}}} .
 	\end{equation}
 	\footnotetext{Apply Cauchy's inequality twice: $\E|\sum_{i=1}^{n}x_{i}y_{i}|\le\E\sqrt{(\sum_{i=1}^{n}x_{i}^{2})(\sum_{i=1}^{n}y_{i}^{2})}\le\sqrt{\sum_{i=1}^{n}\E x_{i}^{2}\sum_{i=1}^{n}\E y_{i}^{2}}$.}	
 	We already have the bound for $\left(\sum_{i=1}^{n}l_{i}\right)^{-1}\sum_{i = 1}^{n} \E\left\|B_{i}B_{i}^{T}\right\|_{2}$. To finish the proof, the only term remained to bound is $\left(\sum_{i=1}^{n}l_{i}\right)^{-1}\sum_{i = 1}^{n} \E\left\|A_{i}A_{i}^{T}\right\|_{2}$. Recall the definition of $A_{i} = \sum_{p = a_{m}}^{i}A^{-1}\epsilon_{p}$  when $a_{m}\le i<a_{m+1}$. Since $A_{i}A_{i}^{T}$ is positive semi-definite, we have 
 	\begin{equation}
 		\E\|A_{i}A_{i}^{T}\|_{2}\le\E\text{tr}(A_{i}A_{i}^{T})=\text{tr}\left(\E(A_{i}A_{i}^{T})\right) = \text{tr}\left(A^{-1}\E\left(\left(\sum_{p = a_{m}}^{i}\epsilon_{p}\right)\left(\sum_{p = a_{m}}^{i}\epsilon_{p}^{T}\right)\right)A^{-T}\right).
 	\end{equation}	
 	When $q\ne q$, we have $\E(\epsilon_{p}\epsilon_{q}^{T}) = 0$. Furthermore,  Let $\E_{n-1}(\epsilon_{n}\epsilon_{n}^{T}) - S = \Sigma_{1}(\delta_{n-1})$ . Then,
 	\begin{equation}
 		\begin{split}
 			\E\|A_{i}A_{i}^{T}\|_{2}& \le  \text{tr}\left(A^{-1}\left(\sum_{p = a_{m}}^{i}S + \E\Sigma_{1}(\delta_{p-1})\right)A^{-T}\right)\\
 			&= \E\text{tr}\left(A^{-1}\left(l_{i}S + \sum_{p = a_{m}}^{i} \Sigma_{1}(\delta_{p-1})\right)A^{-T}\right)\\
 			& \lesssim \E \left\|A^{-1}\left(l_{i}S + \sum_{p = a_{m}}^{i} \Sigma_{1}(\delta_{p-1})\right)A^{-T}\right\|_{2}\\
 			&\lesssim l_{i}\|S\|_{2} + \sum_{p = a_{m}}^{i}\E \left\|\Sigma_{1}(\delta_{p-1})\right\|_{2}.
 		\end{split}
 	\end{equation}	
 	In Assumption \ref{ass:martingale}, we have $\|\Sigma_{1}(\delta)\|_{2}\le C(\|\delta\|_{2}+\|\delta\|_{2}^{2})$ for any $\delta$.  Also Lemma \ref{lemma:1} shows that $\E\|\delta_{n}\|_{2}\le n^{-\alpha/2}(1+\|\delta_{0}\|_{2})$ and $
 	\E\|\delta_{n}\|_{2}^{2}\le n^{-\alpha}(1+\|\delta_{0}\|_{2}^{2})$. Then we can further bound $\E\|A_{i}A_{i}^{T}\|_{2}$ as
 	\begin{equation}
 		\begin{split}
 			\E\|A_{i}A_{i}^{T}\|_{2} &\lesssim l_{i}\|S\|_{2} + \sum_{p = a_{m}}^{i}\left(\E \left\|\delta_{p-1}\right\|_{2} +\E \left\|\delta_{p-1}\right\|_{2}^{2}\right)\\
 			& \lesssim l_{i} +  \sum_{p = a_{m}}^{i}(p-1)^{-\a/2}\\
 			&\lesssim l_{i} +  l_{i}(a_{m}-1)^{-\a/2}\lesssim l_{i}. 
 		\end{split}
 	\end{equation}
 	Then we can bound the remaining term as
 	\begin{equation}\label{ap:9}
 		\left(\sum_{i=1}^{n}l_{i}\right)^{-1}\sum_{i = 1}^{n} \E\left\|A_{i}A_{i}^{T}\right\|_{2}\lesssim O(1).
 	\end{equation}
 	Combining (\ref{ap:9}) and the bound of $III$, we have $II\lesssim M^{((\a-1)\beta+1)/2}$.
 	
 	Now, all three parts $I, II, III$ are bounded by $M^{-\a\beta/2} + M^{-1/2} + M^{((\a-1)\beta+1)/2}$.
 \end{proof}
 
 \begin{lemma}\label{lemma4}
 	Under the same conditions in Lemma \ref{linearcase}, we have
 	\begin{equation}
 		\E \left\| \left(\sum_{i=1}^{n}l_{i}\right)^{-1}\sum_{i=1}^{n}l_{i}^{2}\bar{U}_{n}\bar{U}_{n}^{T}
 		\right\|_{2}
 		\lesssim M^{-1} ,
 	\end{equation}
 	where $a_{M}\le n < a_{M+1}$.
 \end{lemma}
 \begin{proof}	
 	Since $\bar{U}_{n}\bar{U}_{n}^{T}$ is positive semi-definite, we have
 	\begin{equation}
 		\E\left\|\bar{U}_{n}\bar{U}_{n}^{T}\right\|_{2}\le \E\text{tr}\left(\bar{U}_{n}\bar{U}_{n}^{T}\right) =n^{-2}\text{tr}\left(\E\left(\sum_{i=1}^{n}U_{i}\right)\left(\sum_{i=1}^{n}U_{i}\right)^{T}\right).
 	\end{equation}
 	Recall that $ U_{i} =  Y_{0}^{i}U_{0} + \sum_{p = 1}^{i}Y_{p}^{i}\eta_{p}\epsilon_{p}$, then
 	$$ 
 	\sum_{i=1}^{n}U_{i} = \sum_{i=1}^{n}\left( Y_{0}^{i}U_{0} + \sum_{p = 1}^{i}Y_{p}^{i}\eta_{p}\epsilon_{p}\right)
 	= S_{0}^{n}U_{0} + \sum_{p=1}^{n}\left(\I + S_{p}^{n}\right)\eta_{p}\epsilon_{p}.
 	$$
 	Note that $\epsilon_{p}$ are martingale differences. We have the following inequality after plugging in the expansion above:
 	\begin{equation}
 		\begin{split}
 			\E\left\|\bar{U}_{n}\bar{U}_{n}^{T}\right\|_{2}&\le n^{-2}\text{tr}\left(\E\left(\left(S_{0}^{n}U_{0} + \sum_{p=1}^{n}\left(\I + S_{p}^{n}\right)\eta_{p}\epsilon_{p}\right)\left(S_{0}^{n}U_{0} + \sum_{p=1}^{n}\left(\I + S_{p}^{n}\right)\eta_{p}\epsilon_{p}\right)^{T}\right)\right)\\
 			&=n^{-2}\text{tr}\left(S_{0}^{n}\E\left(U_{0}U_{0}^{T}\right){S_{0}^{n}}^{T} + \sum_{p=1}^{n}\left(\I + S_{p}^{n}\right)\eta_{p}^{2}\E(\epsilon_{p}\epsilon_{p}^{T})\left(\I + S_{p}^{n}\right)^{T}\right)\\
 			& = n^{-2}\left(\|S_{0}^{n}\|_{2}^{2}\E\|U_{0}\|_{2}^{2}+ \sum_{p=1}^{n}\left\|\left(\I + S_{p}^{n}\right)\right\|_{2}^{2}\eta_{p}^{2}\E\|\epsilon_{p}\|_{2}^{2}\right).
 		\end{split}
 	\end{equation}
 	In Lemma \ref{lemma:S} we show that $\|S_{i}^{j}\|_{2} \lesssim (i+1)^{\a}$. So here we have $\|S_{0}^{n}\|_{2}^{2}=O(1)$ and
 	$$
 	\sum_{p=1}^{n}\left\|\left(\I + S_{p}^{n}\right)\right\|_{2}^{2}\eta_{p}^{2}\lesssim O(n).
 	$$
 	Since $\E\|U_{0}\|_{2}^{2}$ and $\E\|\epsilon_{p}\|_{2}^{2}$ are bounded, we have 
 	\begin{equation}
 		\E\left\|\bar{U}_{n}\bar{U}_{n}^{T}\right\|_{2}\lesssim O(n^{-1}).
 	\end{equation} 
 	Note that 
 	\begin{equation}
 		\frac{\sum_{i=1}^{n}l_{i}^{2}}{\sum_{i=1}^{n}l_{i}}\le \frac{\sum_{i=1}^{n}l_{i}\max_{k\le M}(a_{k+1}-a_{k})}{\sum_{i=1}^{n}l_{i}} \le n_{M} .
 	\end{equation}
 	Since $n_{M} = M^{\beta-1}$ and $n\asymp M^{1/\beta}$, we have 
 	$$\E \left\| \left(\sum_{i=1}^{n}l_{i}\right)^{-1}\sum_{i=1}^{n}l_{i}^{2}\bar{U}_{n}\bar{U}_{n}^{T}
 	\right\|_{2}  \le  \frac{\sum_{i=1}^{n}l_{i}^{2}}{\sum_{i=1}^{n}l_{i}}\E \left\|\bar{U}_{n}\bar{U}_{n}^{T}
 	\right\|_{2}\lesssim n_{M}n^{-1}\asymp M^{-1}.$$
 \end{proof}
 
 \begin{lemma}\label{lemma5}Under conditions in Lemma \ref{linearcase}, for $a_{M}\le n < a_{M+1}$, we have
 	\begin{equation}
 		\E \left\| \left(\sum_{i=1}^{n}l_{i}\right)^{-1}\sum_{i = 1}^{n}\left(\sum_{k=t_{i}}^{i}U_{k}\right)\left(l_{i}\bar{U}_{n}\right)^{T}\right\|_{2}\lesssim M^{-1/2}.
 	\end{equation}
 \end{lemma}
 \begin{proof}
 	Apply Cauchy's inequality twice we have
 	\begin{equation}
 		\begin{split}
 			& \E\left\|\left(\sum_{i=1}^{n}l_{i}\right)^{-1}\sum_{i = 1}^{n}\left(\sum_{k=t_{i}}^{i}U_{k}\right)\left(l_{i}\bar{U}_{n}\right)^{T}\right\|_{2}\\
 			\le &\text{\footnotemark} \sqrt{\frac{\E\left\|\sum_{i = 1}^{n}\left(\sum_{k=t_{i}}^{i}U_{k}\right)\left(\sum_{k=t_{i}}^{i}U_{k}\right)^{T}\right\|_{2}}{\sum_{i=1}^{n}l_{i}}\frac{\E\left\|\sum_{i = 1}^{n}l_{i}^{2}\bar{U}_{n}\bar{U}_{n}^{T}\right\|_{2}}{\sum_{i=1}^{n}l_{i}}}.
 		\end{split}
 	\end{equation}
 	\footnotetext{Apply Cauchy's inequality twice: $\E|\sum_{i=1}^{n}x_{i}y_{i}|\le\E\sqrt{(\sum_{i=1}^{n}x_{i}^{2})(\sum_{i=1}^{n}y_{i}^{2})}\le\sqrt{\sum_{i=1}^{n}\E x_{i}^{2}\sum_{i=1}^{n}\E y_{i}^{2}}$.}
 	In Lemma \ref{lemma4}, we already have $\E\| \left(\sum_{i=1}^{n}l_{i}\right)^{-1}\sum_{i=1}^{n}l_{i}^{2}\bar{U}_{n}\bar{U}_{n}^{T} \|_{2}\lesssim M^{-1}$. Moreover, the $L_{2}$ norm of $(\sum_{k=a_{m}}^{i}U_{k})(\sum_{k=a_{m}}^{i}U_{k})^{T}$ is less than or equal to its trace since it is positive semi-definite. Then we have LHS of the above equation bounded by
 	$$
 	O(M^{-\frac{1}{2}})\sqrt{\left(\sum_{i=1}^{n}l_{i}\right)^{-1} \sum_{i = 1}^{n} \E\text{tr}\left(\left(\sum_{k=t_{i}}^{i}U_{k}\right)\left(\sum_{k=t_{i}}^{i}U_{k}\right)^{T}\right)}. 
 	$$ 
 	Let $$I = \left(\sum_{i=1}^{n}l_{i}\right)^{-1}\sum_{i = 1}^{n} \E\text{tr}\left(\left(\sum_{k=t_{i}}^{i}U_{k}\right)\left(\sum_{k=t_{i}}^{i}U_{k}\right)^{T}\right).$$ To show Lemma \ref{lemma5}, it is suffices to show $I\lesssim O(1)$. Note that $\lim_{M\rightarrow\infty}\sum_{i=1}^{n}l_{i}/\sum_{i=1}^{a_{M}-1}l_{i}=1$ and $\text{tr}((\sum_{k=t_{i}}^{i}U_{k})(\sum_{k=t_{i}}^{i}U_{k})^{T})\ge 0$. Plug  $U_{k} = Y_{s}^{k}U_{s} + \sum_{p = s+1}^{k}Y_{p}^{k}\eta_{p}\epsilon_{p}$ into $I$, we have
 	\begin{equation}
 		\begin{split}
 			I 
 			\lesssim &  \left(\sum_{i=1}^{a_{M+1} - 1}l_{i}\right)^{-1}\sum_{m = 1}^{M}\sum_{i=a_{m}}^{a_{m+1}-1}\E\text{tr}\left(\left[\sum_{k=a_{m}}^{i}\left(Y_{0}^{k}U_{0} + \sum_{p = 1}^{k}Y_{p}^{k}\eta_{p}\epsilon_{p}\right)\right]\left[\sum_{k=a_{m}}^{i}\left(Y_{0}^{k}U_{0} + \sum_{p = 1}^{k}Y_{p}^{k}\eta_{p}\epsilon_{p}\right)\right]^{T}\right)\\
 			=&\left(\sum_{i=1}^{a_{M+1} - 1}l_{i}\right)^{-1}\sum_{m = 1}^{M}\sum_{i=a_{m}}^{a_{m+1} - 1}\text{tr}\left(\left(\sum_{k=a_{m}}^{i}Y_{0}^{k}\right)\E(U_{0}U_{0}^{T})\left(\sum_{k=a_{m}}^{i}Y_{0}^{k}\right)^{T} \right)\\
 			& + \left(\sum_{i=1}^{a_{M+1} - 1}l_{i}\right)^{-1}\sum_{m = 1}^{M}\sum_{i=a_{m}}^{a_{m+1} - 1}\sum_{p =1}^{i}\text{tr}\left(\left(\sum_{k = \max(a_{m},p)}^{i}Y_{p}^{k}\right)\E(\epsilon_{p}\epsilon_{p}^{T})\left(\sum_{k = \max(a_{m},p)}^{i}Y_{p}^{k}\right)^{T}\right)\eta_{p}^{2}\\
 			=& II + III.
 		\end{split}
 	\end{equation}
 	
 	Next we shall show that both $II$ and $III$ are bounded by $O(1)$. The first term 
 	$$II = \left(\sum_{i=1}^{a_{M+1} - 1}l_{i}\right)^{-1}\sum_{m = 1}^{M}\sum_{i=a_{m}}^{a_{m+1} - 1}\text{tr}\left(\left(\sum_{k=a_{m}}^{i}Y_{0}^{k}\right)\E(U_{0}U_{0}^{T})\left(\sum_{k=a_{m}}^{i}Y_{0}^{k}\right)^{T} \right).$$ 
 	It can be bounded using $\text{tr}(C)\le d\|C\|_{2}$ as follows
 	\begin{equation}
 		II \lesssim  \left(\sum_{i=1}^{a_{M+1} - 1}l_{i}\right)^{-1}\sum_{m = 1}^{M}\sum_{i=a_{m}}^{a_{m+1} - 1}\left\|\left(\sum_{k=a_{m}}^{i}Y_{0}^{k}\right)\right\|_{2}^{2}\left\|\E(U_{0}U_{0}^{T})\right\|_{2}.
 	\end{equation}
 	From Lemma \ref{lemma:S},
 	$$\left\|\left(\sum_{k=a_{m}}^{i}Y_{0}^{k}\right)\right\|_{2}^{2} = \left\|S_{0}^{i}- S_{0}^{a_{m}}\right\|_{2}^{2} \lesssim \left\|S_{0}^{i}\right\|_{2}^{2} + \left\| S_{0}^{a_{m}}\right\|_{2}^{2}\lesssim O(1).$$
 	Also note that $\left\|\E(U_{0}U_{0}^{T})\right\|_{2}\lesssim O(1)$. Then
 	\begin{equation} 
 		II\lesssim\left(\sum_{i=1}^{a_{M+1} - 1}l_{i}\right)^{-1}\sum_{m = 1}^{M}\sum_{i=a_{m}}^{a_{m+1} - 1}O(1) =  O(1). 
 	\end{equation}

 	The  term $III$ can be bounded as:
 	\begin{equation}
 		\begin{split}
 			III = &\left(\sum_{i=1}^{a_{M+1} - 1}l_{i}\right)^{-1}\sum_{m = 1}^{M}\sum_{i=a_{m}}^{a_{m+1} - 1}\sum_{p =1}^{i}\text{tr}\left(\left(\sum_{k = \max(a_{m},p)}^{i}Y_{p}^{k}\right)\E(\epsilon_{p}\epsilon_{p}^{T})\left(\sum_{k = \max(a_{m},p)}^{i}Y_{p}^{k}\right)^{T}\right)\eta_{p}^{2}\\
 			\lesssim &\left(\sum_{i=1}^{a_{M+1} - 1}l_{i}\right)^{-1}\sum_{m = 1}^{M}\sum_{i=a_{m}}^{a_{m+1} - 1}\sum_{p=1}^{i}\left\|\sum_{k = \max(a_{m},p)}^{i}Y_{p}^{k}\right\|_{2}^{2}\eta_{p}^{2}\\
 			\le& \left(\sum_{i=1}^{a_{M+1} - 1}l_{i}\right)^{-1}\sum_{m = 1}^{M}\sum_{i=a_{m}}^{a_{m+1} - 1}\sum_{p=1}^{i}\left(\sum_{k = \max(a_{m},p)}^{i} \|Y_{p}^{k}\|_{2}\right)^{2}\eta_{p}^{2}.
 		\end{split}
 	\end{equation}
 	Let $IV = \sum_{p=1}^{i}\left(\sum_{k = \max(a_{m},p)}^{i} \|Y_{p}^{k}\|_{2}\right)^{2}\eta_{p}^{2}$. From Lemma \ref{lemma:Y}, $$\|Y_{i}^{j}\|_{2} \le \exp\left[-\frac{\gamma\eta}{1-\alpha}\left(j^{1-\alpha} - (i+1)^{1-\alpha}\right)\right].$$ Then for $a_{m}\le i < a_{m+1}$, we have
 	\begin{equation}
 		IV \le \sum_{p=1}^{i}\left(\sum_{k = \max(a_{m},p)}^{i}\exp\left(-\eta\gamma\frac{k^{1-\alpha}}{1-\alpha}\right) \right)^{2}\eta_{p}^{2}e^{ \frac{2\eta\gamma}{1-\alpha}p^{1-\alpha}}.
 	\end{equation}
 	Using the integration, we can further bound $IV$ as
 	\begin{equation}
 		\begin{split}
 			IV\lesssim& \sum_{p=1}^{i}\left(\int_{\max(a_{m},p)}^{i}\exp\left(-\eta\gamma\frac{k^{1-\alpha}}{1-\alpha}\right) dk\right)^{2}p^{-2\alpha}e^{ \frac{2\eta\gamma}{1-\alpha}p^{1-\alpha}}\\
 			\lesssim&  \sum_{p=1}^{i}\left(\int_{\max(a_{m},p)^{1-\alpha}}^{i^{1-\alpha}}e^{-\frac{\eta\gamma}{1-\alpha}t}t^{\frac{\alpha}{1-\alpha}}dt\right)^{2}p^{-2\alpha}e^{ \frac{2\eta\gamma}{1-\alpha}p^{1-\alpha}}\\
 			\lesssim&  \sum_{p=1}^{i}e^{-\frac{2\eta\gamma}{1-\alpha}\max(a_{m},p)^{1-\alpha}}\max(a_{m},p)^{2\alpha} p^{-2\alpha}e^{ \frac{2\eta\gamma}{1-\alpha}p^{1-\alpha}}\\
 			\lesssim&  \sum_{p=1}^{a_{m}-1}e^{-\frac{2\eta\gamma}{1-\alpha}(a_{m}^{1-\alpha}-p^{1-\alpha})}\left(\frac{a_{m}}{p}\right)^{2\alpha} + l_{i}.
 		\end{split}
 	\end{equation}  	
 	Then $III$ is bounded by 
 	\begin{equation}
 		III \lesssim \left(\sum_{i=1}^{a_{M+1} - 1}l_{i}\right)^{-1}\sum_{m = 1}^{M}\sum_{i=a_{m}}^{a_{m+1} - 1}\sum_{p=1}^{a_{m}-1}e^{-\frac{2\eta\gamma}{1-\alpha}(a_{m}^{1-\alpha}-p^{1-\alpha})}\left(\frac{a_{m}}{p}\right)^{2\alpha} + 1.
 	\end{equation}
 	To show $III$ is also bounded by $O(1)$, it is suffices to show that 
 	\begin{equation}
 		\sum_{m = 1}^{M}\sum_{i=a_{m}}^{a_{m+1} - 1}\sum_{p=1}^{a_{m}-1}e^{-\frac{2\eta\gamma}{1-\alpha}(a_{m}^{1-\alpha}-p^{1-\alpha})}\left(\frac{a_{m}}{p}\right)^{2\alpha}\lesssim \sum_{i=1}^{a_{M+1} - 1}l_{i}.
 	\end{equation}
 	Using partial integration we have the following:
 	$$\int_{1}^{a_{m}-1}e^{\frac{2\eta\gamma}{1-\alpha}p^{1-\alpha}}p^{-2\alpha}dp = \int_{\frac{2\eta\gamma}{1-\alpha}}^{\frac{2\eta\gamma}{1-\alpha}(a_{m}-1)^{1-\alpha}}e^{u}u^{-\frac{\alpha}{1-\alpha}}du\lesssim e^{\frac{2\eta\gamma}{1-\alpha}a_{m-1}^{1-\alpha}}(a_{m}-1)^{-\alpha}.$$
 	Then we have
 	\begin{equation}
 		\begin{split}
 			& \sum_{m = 1}^{M}\sum_{i=a_{m}}^{a_{m+1} - 1}\sum_{p=1}^{a_{m}-1}e^{-\frac{2\eta\gamma}{1-\alpha}(a_{m}^{1-\alpha}-p^{1-\alpha})}\left(\frac{a_{m}}{p}\right)^{2\alpha}\\
 			\lesssim& \sum_{m = 1}^{M}\sum_{i=a_{m}}^{a_{m+1} - 1}e^{-\frac{2\eta\gamma}{1-\alpha}a_{m}^{1-\alpha}}a_{m}^{2\alpha}\int_{1}^{a_{m}-1}e^{\frac{2\eta\gamma}{1-\alpha}p^{1-\alpha}}p^{-2\alpha}dp\\ 
 			\lesssim&\sum_{m=1}^{M}n_{m}a_{m}^{\alpha}.
 		\end{split}
 	\end{equation}
 	Note that $ a_{m}^{\alpha}\lesssim n_{m}$ since $\beta >1/(1-\alpha)$, so we have the following
 	\begin{equation}
 		\sum_{m = 1}^{M}\sum_{i=a_{m}}^{a_{m+1} - 1}\sum_{p=1}^{a_{m}-1}e^{-\frac{2\eta\gamma}{1-\alpha}(a_{m}^{1-\alpha}-p^{1-\alpha})}\left(\frac{a_{m}}{p}\right)^{2\alpha}\lesssim \sum_{m=1}^{M}n_{m}a_{m}^{\alpha}\lesssim \sum_{m=1}^{M}n_{m}^{2}
 		\asymp \sum_{i=1}^{a_{M+1}-1}l_{i}.
 	\end{equation} 
 	
 \end{proof}

 \section{Proof of Main Theorems}\label{app:3}
 \subsection{Proof of Theorem \ref{theorem2}}
 \begin{proof}	In Lemma \ref{linearcase}, we demonstrate the convergence property of the estimator $\tilde{\Sigma}$, which is constructed based on linear process $\{U_{n}\}_{n\in\N}$. Let $s_{n} =  \delta_{n} - U_{n}$ be the difference between the error sequence $\delta_{n}$ and the linear sequence $U_{n}$. It has the following recursion form:
 	\begin{equation} 
 		\begin{split}
 			s_{n} & = \delta_{n-1} -  \eta_{n}\nabla F(x_{n-1}) -  (\I-\eta_{n}A)U_{n-1}\\
 			&= (\I - \eta_{n}A)(\delta_{n-1} - U_{n-1}) - \eta_{n}(\nabla F(x_{n-1})- A\delta_{n-1})\\
 			& = (\I - \eta_{n}A)s_{n-1} - \eta_{n}(\nabla F(x_{n-1})- A\delta_{n-1}).
 		\end{split}
 	\end{equation}
 	When $n$ is big enough, $x_{n-1}$ is close to the minimizer $x^{*}$. Based on Taylor's expansion around $x^{*}$, $\nabla F(x_{n-1})\approx A\delta_{n-1}$ since $\nabla F(x^{*})$ is zero. So
 	\begin{equation} 
 		\begin{split}
 			s_{n} \approx (\I - \eta_{n}A)s_{n-1}.
 		\end{split}
 	\end{equation}
 	It takes a similar linear form as $U_{n}$ and its value is small especially when $\delta_{n}$ is small. So the difference between $U_{n}$ and $\delta_{n}$, i.e., $x_{n} - x^{*}$, decays quickly as $n\rightarrow \infty$. We expect that the covariance matrix estimator $\tilde{\Sigma}$ and the recursive estimator $\hat{\Sigma}$ are asymptotically close. 
 	
 	To show Theorem \ref{theorem2}, it is suffices to show that $\E\|\tilde{\Sigma}_{n} - \hat{\Sigma}_{n}\|_{2}$ can be bounded with the same order as  $\E\|\tilde{\Sigma}_{n} - \Sigma\|_{2}$. 
 	Note that $\delta_{n}=x_{n} - x^{*}$ and $\hat{\Sigma}_{n}$ can be rewritten as 
 	$$
 	\hat{\Sigma}_{n} =  \left(\sum_{i=1}^{n}l_{i} \right)^{-1}\sum_{i=1}^{n}\left(\sum_{k=t_{i}}^{i}\delta_{k}-l_{i}\bar{\delta}_{n}\right)\left(\sum_{k=t_{i}}^{i}\delta_{k}-l_{i}\bar{\delta}_{n}\right)^{T}.
 	$$
 	Plug in the difference $s_{n }=\delta_{n} - U_{n}$, we can expand $\E\|\tilde{\Sigma}_{n} - \hat{\Sigma}_{n}\|_{2}$ as		
 	\begin{equation}\label{ap:12}
 		\begin{split}
 			\E\|\tilde{\Sigma}_{n} - \hat{\Sigma}_{n}\|_{2} \le 2\E\left\|\left(\sum_{i=1}^{n}l_{i} \right)^{-1}\sum_{i=1}^{n}\left(\sum_{k=t_{i}}^{i}U_{k}-l_{i}\bar{U}_{n}\right)\left(\sum_{k=t_{i}}^{i}s_{k}-l_{i}\bar{s}_{n}\right)^{T}\right\|_{2} \\+\E\left\|\left(\sum_{i=1}^{n}l_{i} \right)^{-1}\sum_{i=1}^{n}\left(\sum_{k=t_{i}}^{i}s_{k}-l_{i}\bar{s}_{n}\right)\left(\sum_{k=t_{i}}^{i}s_{k}-l_{i}\bar{s}_{n}\right)^{T}\right\|_{2}.
 		\end{split}
 	\end{equation}
 	We further claim that 
 	\begin{equation}\label{eq:claim_s}
 		\E\left\|\left(\sum_{i=1}^{n}l_{i} \right)^{-1}\sum_{i=1}^{n}\left(\sum_{k=t_{i}}^{i}s_{k}-l_{i}\bar{s}_{n}\right)\left(\sum_{k=t_{i}}^{i}s_{k}-l_{i}\bar{s}_{n}\right)^{T}\right\|_{2}\lesssim M^{-1}.
 	\end{equation}
 	Apply Cauchy's inequality twice, the first part in LHS of (\ref{ap:12}) can be bounded as following:
 	\begin{equation}
 		\begin{split}
 			&\E\left\|\left(\sum_{i=1}^{n}l_{i} \right)^{-1}\sum_{i=1}^{n}\left(\sum_{k=t_{i}}^{i}U_{k}-l_{i}\bar{U}_{n}\right)\left(\sum_{k=t_{i}}^{i}s_{k}-l_{i}\bar{s}_{n}\right)^{T}\right\|_{2}\\
 			&\le \sqrt{\E\|\tilde{\Sigma}_{n}\|_{2}}\sqrt{\E\left\|\left(\sum_{i=1}^{n}l_{i} \right)^{-1}\sum_{i=1}^{n}\left(\sum_{k=t_{i}}^{i}s_{k}-l_{i}\bar{s}_{n}\right)\left(\sum_{k=t_{i}}^{i}s_{k}-l_{i}\bar{s}_{n}\right)^{T}\right\|_{2}}\\
 			&\lesssim M^{-1/2},
 		\end{split}
 	\end{equation}
 	since  $\sqrt{\E\|\tilde{\Sigma}_{n}\|_{2}}$ is bounded by some constant.  
 	Then  $\E\|\tilde{\Sigma}_{n} - \hat{\Sigma}_{n}\|_{2} \lesssim M^{-1/2}$ and we have Theorem \ref{theorem2}. All we need to prove now is the claim in \eqref{eq:claim_s}.
 	By triangle inequality and the fact $\|C\|_{2}\le\text{tr}(C)$ for any positive semi-definite matrix $C$,
 	\begin{equation}\label{ap:15}
 		\begin{split}
 			&\E\left\|\left(\sum_{i=1}^{n}l_{i}\right)^{-1}\sum_{i=1}^{n}\left(\sum_{k=t_{i}}^{i}s_{k}-l_{i}\bar{s}_{n}\right)\left(\sum_{k=t_{i}}^{i}s_{k}-l_{i}\bar{s}_{n}\right)^{T}\right\|_{2} \\
 			\le&\left(\sum_{i=1}^{n}l_{i}\right)^{-1}\sum_{i=1}^{n}\E\text{tr}\left(\left(\sum_{k=t_{i}}^{i}s_{k}-l_{i}\bar{s}_{n}\right)\left(\sum_{k=t_{i}}^{i}s_{k}-l_{i}\bar{s}_{n}\right)^{T} \right)\\
 			=&\left(\sum_{i=1}^{n}l_{i}\right)^{-1}\sum_{i=1}^{n}\E\left\|\sum_{k=t_{i}}^{i}s_{k}-l_{i}\bar{s}_{n}\right\|_{2}^{2}\\
 			\lesssim& \left(\sum_{i=1}^{n}l_{i}\right)^{-1}\sum_{i=1}^{n}\E\left\|\sum_{k=t_{i}}^{i}s_{k}\right\|_{2}^{2}+\left(\sum_{i=1}^{n}l_{i}\right)^{-1}\sum_{i=1}^{n}l_{i}^{2}\E\left\|\bar{s}_{n}\right\|_{2}^{2}.
 		\end{split}
 	\end{equation} 
 	Note that $s_{n}$ takes the form 
 	$$s_{n} = (\I - \eta_{n}A)s_{n-1} - \eta_{n}(\nabla F(x_{n-1})- A\delta_{n-1}), \delta_{0} = 0.$$ 
 	
 	First, we shall prove that $$\left(\sum_{i=1}^{n}l_{i}\right)^{-1}\sum_{i=1}^{n}l_{i}^{2}\E\left\|\bar{s}_{n}\right\|_{2}^{2}=O(M^{-1}).$$ Based on the definition of $Y_{p}^{k}$ we have  
 	$$
 	s_{k} = \sum_{p=1}^{k}Y_{p}^{k}\eta_{p}[A\delta_{p-1}-\nabla F(x_{p-1})],
 	$$
 	and
 	$$
 	\begin{aligned}
 		\bar{s}_{n} &= n^{-1}\sum_{k=1}^{n}\sum_{p=1}^{k}Y_{p}^{k}\eta_{p}[A\delta_{p-1}-\nabla F(x_{p-1})]\\
 		& = n^{-1}\sum_{p=1}^{n}\left(\I +S_{p}^{n}\right)\eta_{p}[A\delta_{p-1}-\nabla F(x_{p-1})].
 	\end{aligned}
 	$$
 	By Cauchy's inequality 
 	\begin{equation}
 		\begin{split}
 			\E\left\|\bar{s}_{n}\right\|_{2}^{2}& = n^{-2}\E\left\|\sum_{p=1}^{n}\left(\I +S_{p}^{n}\right)\eta_{p}[A\delta_{p-1}-\nabla F(x_{p-1})]\right\|_{2}^{2}\\
 			&\le n^{-2}\E\left(\sum_{p=1}^{n}\left\|\I +S_{p}^{n}\right\|_{2}\eta_{p}\left\|A\delta_{p-1}-\nabla F(x_{p-1})\right\|_{2}\right)^{2}\\
 			&\le n^{-2}\left(\sum_{p=1}^{n}\left\|\I +S_{p}^{n}\right\|_{2}^{2}\eta_{p}^{2}\right)\left(\sum_{p=1}^{n}\E \left\|A\delta_{p-1}-\nabla F(x_{p-1})\right\|_{2}^{2}\right).
 		\end{split}
 	\end{equation}
 	From Lemma \ref{lemma:S}, $\|S_{p}^{n}\|_{2}\lesssim (p+1)^{\a}$, and therefore $\sum_{p=1}^{n}\left\|\I +S_{p}^{n}\right\|_{2}^{2}\eta_{p}^{2}\lesssim O(n)$. By Taylor's expansion around $x^{*}$, $\left\|A\delta_{p}-\nabla F(x_{p})\right\|_{2} =  O(\|\delta_{p}\|_{2}^{2})$. Then using Lemma \ref{lemma:1}
 	\begin{equation}
 		\sum_{p=1}^{n}\E \left\|A\delta_{p-1}-\nabla F(x_{p-1})\right\|_{2}^{2}\asymp \sum_{p=1}^{n}\E\|\delta_{p-1}\|_{2}^{4}\lesssim \sum_{p=1}^{n}(p-1)^{-2\a}.
 	\end{equation}
 	Since $\a>1/2$, $\sum_{p=1}^{n}(p-1)^{-2\a}=O(1)$. Then $\E\left\|\bar{s}_{n}\right\|_{2}^{2}\lesssim n^{-1}$. Recall that  $n_{k} = Ck^{\beta-1}$and $n\asymp M^{1/\beta}$. We then have,
 	\begin{equation}\label{ap:16}
 		\left(\sum_{i=1}^{n}l_{i}\right)^{-1}\sum_{i=1}^{n}l_{i}^{2}\E\left\|\bar{s}_{n}\right\|_{2}^{2}\lesssim n^{-1}n_{M}\asymp M^{-1}.
 	\end{equation}		
 	Next we shall prove $\left(\sum_{i=1}^{n}l_{i}\right)^{-1}\sum_{i=1}^{n}\E\left\|\sum_{k=t_{i}}^{i}s_{k}\right\|_{2}^{2}$ is bounded by $O(M^{-1})$. For $t_{i}\le k\le i$, where $t_{i}$ is defined in Section \ref{sec:meth}, we have
 	$$
 	\begin{aligned}
 		s_{k}& = \prod_{p=t_{i}}^{k}\left(\I -\eta_{p}A\right)s_{t_{i}-1} + \sum_{p = t_{i}}^{k}\prod_{i=p+1}^{k}\left(\I - \eta_{i}A\right)\eta_{p}\left(A\delta_{p-1}-\nabla F(x_{p-1})\right)\\
 		& = Y_{t_{i}-1}^{k}s_{t_{i}-1} + \sum_{p=t_{i}}^{k}Y_{p}^{k}\eta_{p}\left(A\delta_{p-1} - \nabla F(x_{p-1})\right),
 	\end{aligned}
 	$$ 
 	and
 	$$
 	\begin{aligned}
 		\sum_{k = t_{i}}^{i}s_{k} = S_{t_{i}-1}^{k}s_{t_{i}-1} + \sum_{p=t_{i}}^{i}\left(\I + S_{p}^{i}\right)\eta_{p}\left(A\delta_{p-1} - \nabla F(x_{p-1})\right).
 	\end{aligned}
 	$$ 
 	Using triangle inequality and Cauchy's inequality ,
 	\begin{equation}
 		\begin{split}
 			\E\left\|\sum_{k = t_{i}}^{i}s_{k}\right\|_{2}^{2}&\lesssim \E\left(\left\|S_{t_{i-1}}^{i}s_{t_{i}-1}\right\|_{2}^{2} + \left(\sum_{p=t_{i}}^{i}\left\|\I + S_{p}^{i}\right\|_{2}\eta_{p}\left\|A\delta_{p-1} - \nabla F(x_{p-1})\right\|_{2}\right)^{2}\right)\\
 			&\lesssim \left\|S_{t_{i-1}}^{i}\right\|_{2}^{2}\E\left\|s_{t_{i}-1}\right\|_{2}^{2} + \left(\sum_{p=t_{i}}^{i}\left\|\I + S_{p}^{i}\right\|_{2}^{2}\eta_{p}^{2}\right)\left(\sum_{p=t_{i}}^{i}\E\left\|A\delta_{p-1} - \nabla F(x_{p-1})\right\|_{2}^{2}\right).
 		\end{split}
 	\end{equation}
 	From Lemma \ref{lemma:S} $\|S_{p}^{i}\|_{2}\lesssim (p+1)^{\a}$, therefore we have $$\sum_{p=t_{i}}^{i}\left\|\I +S_{p}^{n}\right\|_{2}^{2}\eta_{p}^{2}\lesssim  l_{i}.$$
 	According to Taylor's expansion around $x^{*}$, $\left\|A\delta_{p}-\nabla F(x_{p})\right\|_{2} =  O(\|\delta_{p}\|_{2}^{2})$. Then using Lemma \ref{lemma:1} we have,
 	\begin{equation}
 		\sum_{p=t_{i}}^{i}\E \left\|A\delta_{p-1}-\nabla F(x_{p-1})\right\|_{2}^{2}\asymp \sum_{p=t_{i}}^{i}\E\|\delta_{p-1}\|_{2}^{4}\lesssim l_{i}t_{i}^{-2\a}. 
 	\end{equation}
 	Note that $s_{k} = \delta_{k} - U_{k}$. From Lemma \ref{lemma:1} and \ref{lemma:U}, $\E\left\|\delta_{k} \right\|_{2}\asymp\E\left\|U_{k}\right\|_{2} \lesssim k^{-\a}$. So,
 	$$
 	\E\left\|s_{k}\right\|_{2}^{2} \le 2\E\left\|\delta_{k} \right\|_{2}^{2} + 2\E\left\|U_{k}\right\|_{2}^{2}\lesssim k^{-2\a}.
 	$$
 	Thus,
 	$$ 
 	\E\left\|\sum_{k = t_{i}}^{i}s_{k}\right\|_{2}^{2}\lesssim t_{i}^{2\a}t_{i}^{-2\a}+ l_{i}^{2}t_{i}^{-2\a} = 1 + l_{i}^{2}t_{i}^{-2\a}.
 	$$
 	Since $\left(\sum_{i=1}^{n}l_{i}\right)^{-1}\asymp \left(\sum_{m=1}^{M}n_{m}^{2}\right)^{-1}$ and $\alpha > 1/2$, we have
 	\begin{equation}\label{ap:17}
 		\begin{split}
 			\left(\sum_{i=1}^{n}l_{i}\right)^{-1}\sum_{i=1}^{n}\E\left\|\sum_{k=t_{i}}^{i}s_{k}\right\|_{2}^{2}&\lesssim \left(\sum_{m=1}^{M}n_{m}^{2}\right)^{-1}\left(\sum_{m=1}^{M}\sum_{i=a_{m}}^{a_{m+1}-1}(1+ l_{i}^{2}a_{m}^{-2\a})\right)\\
 			&\lesssim n_{M}^{-1} + \left(\sum_{m=1}^{M}n_{m}^{2}\right)^{-1}\left(\sum_{m=1}^{M}n_{m}^{3}a_{m}^{-2\a}\right)\\
 			&\lesssim M^{-1}.
 		\end{split}
 	\end{equation}
 	The claim is proved through (\ref{ap:15}), (\ref{ap:16}) and (\ref{ap:17}).
 	
 	Finally, using the fact  $M = O(n^{1/\beta})$, we can obtain the upper bound in terms of $n$.
 \end{proof}
 
 \subsection{Proof of Theorem \ref{theorem: nonoverlap}}
 The proof for the non-overlapping version is slightly simpler than but almost the same as that for the overlapping version. Instead of writing down the similar long proof, we will provide a high level clarification when changes are needed. 
 
 In the proof of Theorem \ref{theorem2}, we break down the estimation error into several parts in the following form:
 \begin{equation}\label{generalform: O}
 	\left(\sum_{m=1}^{M}\sum_{i=a_{m}}^{a_{m+1}-1}|B_{i}|\right)^{-1}\sum_{m=1}^{M}\sum_{i=a_{m}}^{a_{m+1}-1}T_{i},
 \end{equation}
 where $T_{i}$ is the term associated with batch $B_{i}$, the explicit formula may vary from parts to parts. In the proof for the non-overlapping version, we break down the estimation error into similar parts as above but in the form:
 \begin{equation}\label{generalform: NO}
 	\left(\sum_{m=1}^{M}|B_{a_{m+1}-1}|\right)^{-1}\sum_{m=1}^{M}T_{a_{m+1}-1}.
 \end{equation}
 So, in comparison with the proof for the overlapping version, fewer terms are needed to bound in the non-overlapping version proof. As we can see from previous proof, for large $m$, $T_{i}$'s for $i\in [a_{m}, a_{m+1}-1]$ are usually bounded by the same order, in other words $\sum_{i=a_{m}}^{a_{m+1}-1}T_{i}$ are proportion to $T_{a_{m+1}-1}$. That means the upper bound for $\sum_{m=1}^{M}T_{a_{m+1}-1}$ can be easily generated from the upper bound for $\sum_{m=1}^{M}\sum_{i=a_{m}}^{a_{m+1}-1}T_{i}$. Since $a_{m}$ are polynomially increasing, term in \eqref{generalform: NO} and term in \eqref{generalform: O} are of the same order.
 
 Next, we shall give an example to show how we can leverage pervious proofs. Define 
 \begin{equation}
 	\hat{S}_{n, NOL} = n^{-1}\left[\sum_{m=1}^{M-1}\left(\sum_{k = a_{m}}^{a_{m+1}-1}\epsilon_{k}\right)\left(\sum_{k = a_{m}}^{a_{m+1}-1}\epsilon_{k}\right)^{T} +  \left(\sum_{k = a_{M}}^{n}\epsilon_{k}\right)\left(\sum_{k = a_{M}}^{n}\epsilon_{k}\right)^{T}\right].
 \end{equation}
 We shall follow the same proof of Lemma \ref{lemma2} to show the corresponding result
 \begin{equation}
 	\E\|\hat{S}_{n, NOL}-S\|_{2}\le M^{-\a\beta/2}+ M^{-1/2}.
 \end{equation}
 \begin{proof}
 	We  define
 	$$\hat{S}_{n, NOL}^{*} = n^{-1}\left[\sum_{m=1}^{M-1}\left(\sum_{k = a_{m}}^{a_{m+1}-1}\epsilon_{k}^{*}\right)\left(\sum_{k = a_{m}}^{a_{m+1}-1}\epsilon_{k}^{*}\right)^{T} +  \left(\sum_{k = a_{M}}^{n}\epsilon_{k}^{*}\right)\left(\sum_{k = a_{M}}^{n}\epsilon_{k}^{*}\right)^{T}\right].$$
 	Then we can bound $\E\|{\hat{S}_{n, NOL}} - S\|_{2}$ through triangle inequality 
 	\begin{equation}
 		\E\|{\hat{S}_{n, NOL}} - S\|_{2}\le \E\|\hat{S}^{*}_{n, NOL} - S\|_{2} + \E\|\hat{S}_{n, NOL} - \hat{S}^{*}_{n, NOL}\|_{2}.
 	\end{equation}
 	
 	\bigskip
 	\noindent{\bf Step 1:} Bound $\E\|\hat{S}^{*}_{n, NOL} - S\|_{2}$. 	\\
 	Same as in the proof of Lemma \ref{lemma2}, we have  
 	\begin{equation} 
 		\E\|\hat{S}_{n, NOL}^{*} - S\|_{2} 
 		\le  \sqrt{d\|\E(\hat{S}_{n, NOL}^{*}-S)^{2}\|_{2}}.
 	\end{equation}  
 	Note that by definition of $S$,
 	$$\E(\hat{S}_{n, NOL}^{*}) = n^{-1}\left[\sum_{m=1}^{M-1}\sum_{k = a_{m}}^{a_{m+1}-1}\E(\epsilon_{k}^{*}\epsilon_{k}^{*T}) +   \sum_{k = a_{M}}^{n}\E(\epsilon_{k}^{*}\epsilon_{k}^{*T})\right] = S.$$
 	Then 
 	$$
 	\|\E(\hat{S}_{n, NOL}^{*}-S)^{2}\|_{2} = \|\E\hat{S}_{n, NOL}^{*2}-S^{2}\|_{2}.
 	$$
 	Note  that  $\E(\epsilon_{p_{1}}\epsilon_{p_{2}}^{T}\epsilon_{p_{3}}\epsilon_{p_{4}}^{T})$ is nonzero if and only if for any $r$ there exist $r'\ne r$ such that $p_{r} = p_{r'}$, $r, r' \in \{1,2,3,4\}$. There are two cases we can consider. The first case is $p_{1}=p_{3}\ne p_{2}=p_{4}$ or $p_{1}=p_{4}\ne p_{2}=p_{3}$. This requires $i$ and $j$ in the same block. The second case is  $p_{1}=p_{2}$ and $p_{3}=p_{4}$. So  we can expand $\E{\hat{S}_{n, NOL}}^{2}$ and rewrite it into two parts,
 	\begin{equation}  
 		\E{\hat{S}_{n, NOL}}^{*2} = n^{-2}I + n^{-2} II ,
 	\end{equation}
 	where 
 	$$
 	\begin{aligned} 
 		I = &\E\sum_{m = 1}^{M-1}\sum_{a_{m}\le p_{1}\ne p_{2}\le a_{m+1}-1}\left(\epsilon_{p_{1}}^{*}\epsilon_{p_{2}}^{*T}\epsilon_{p_{1}}^{*}\epsilon_{p_{2}}^{*T} + \epsilon_{p_{1}}^{*}\epsilon_{p_{2}}^{*T}\epsilon_{p_{2}}^{*}\epsilon_{p_{1}}^{*T} \right) + \E\sum_{a_{M}\le p_{1}\ne p_{2}\le n}\left(\epsilon_{p_{1}}^{*}\epsilon_{p_{2}}^{*T}\epsilon_{p_{1}}^{*}\epsilon_{p_{2}}^{*T} + \epsilon_{p_{1}}^{*}\epsilon_{p_{2}}^{*T}\epsilon_{p_{2}}^{*}\epsilon_{p_{1}}^{*T} \right),
 	\end{aligned}
 	$$
 	$$
 	\begin{aligned}
 		II =  \sum_{i \in SET_{n} }\sum_{j  \in SET_{n}}\sum_{p = t_{i}}^{i}\sum_{q = t_{j}}^{j}\E(\epsilon_{p}^{*}\epsilon_{p}^{*T}\epsilon_{q}^{*}\epsilon_{q}^{*T}), (SET_{n} = \{a_2-1, a_3-1, ..., a_M-1 \} \cup \{n\})  .
 	\end{aligned}$$
 	
 	Let  $\|\E(\epsilon_{p_{1}}^{*}\epsilon_{p_{2}}^{*T}\epsilon_{p_{3}}^{*}\epsilon_{p_{4}}^{*T})\|_{2}$ be bounded by constant $C$ for any $p_{r}$,$r\in \{1,2,3,4\}$. Then we can bound $I$ as follows, 
 	\begin{equation}
 		\begin{split}
 			\|I\|_{2}
 			\le &\sum_{m = 1}^{M} \left[\sum_{a_{m}\le p_{1}\ne p_{2}\le a_{m+1}-1}\left(C + C\right) \right]
 			\lesssim   \sum_{m = 1}^{M} n_{m}^{2}.
 		\end{split}
 	\end{equation}
 	Since $n\asymp M^{\beta}$ ,  we have,
 	\begin{equation} 
 		n^{-2}\|I\|_{2}\lesssim \frac{\sum_{m = 1}^{M} n_{m}^{2}}{n^{2}} \lesssim M^{-1}.
 	\end{equation}
 	Next, notice that $\sum_{i \in SET_{n} }\sum_{j  \in SET_{n}}\sum_{p = t_{i}}^{i}\sum_{q = t_{j}}^{j}1 =n^{2}$. Then,
 	\begin{equation} 
 		\begin{split}
 			\left\|n^{-2}II - S^{2}\right\|_{2}=&  n^{-2}\left\|\sum_{i \in SET_{n} }\sum_{j  \in SET_{n}}\sum_{p = t_{i}}^{i}\sum_{q = t_{j}}^{j}(\E(\epsilon_{p}^{*}\epsilon_{p}^{*T}\epsilon_{q}^{*}\epsilon_{q}^{*T})-S^{2})\right\|_{2}\\
 			\le & n^{-2}  \sum_{m = 1}^{M}\sum_{k = 1}^{M} \sum_{p = a_{m}}^{a_{m+1}-1}\sum_{q = a_{k}}^{a_{k+1}-1}\left\|\E(\epsilon_{p}^{*}\epsilon_{p}^{*T}\epsilon_{q}^{*}\epsilon_{q}^{*T}) - S^{2}\right\|_{2} \\
 			= & n^{-2}III + n^{-2}IV.
 		\end{split}
 	\end{equation}				
 	We consider two cases here. One is when $p$ and $q$ are in the same block. Let 
 	$$
 	III = \sum_{m = 1}^{M}\sum_{p = a_{m}}^{a_{m+1}-1}\sum_{q = a_{m}}^{a_{m+1}-1}\left\|\E(\epsilon_{p}^{*}\epsilon_{p}^{*T}\epsilon_{q}^{*}\epsilon_{q}^{*T}) - S^{2}\right\|_{2}.
 	$$ 
 	Here $\|\E(\epsilon_{p}^{*}\epsilon_{p}^{*T}\epsilon_{q}^{*}\epsilon_{q}^{*T})\|_{2}$ is still bounded by constant $C$. Then we have
 	\begin{equation} 
 		\begin{split}
 			n^{-2} III  \le& n^{-2}\sum_{m = 1}^{M}\sum_{p = a_{m}}^{a_{m+1}-1}\sum_{q = a_{m}}^{a_{m+1}-1}\left(C +\left\| S^{2}\right\|_{2}\right)\\
 			\lesssim&n^{-2}\sum_{m=1}^{M}n_{m}^{2}  \lesssim  M^{-1}.
 		\end{split}
 	\end{equation} 
 	The other case is when $p$ and $q$ are in different blocks. Let 
 	$$
 	IV = \sum_{m\ne k} \sum_{q=a_{k}}^{a_{k+1}-1}\sum_{p=a_{m}}^{a_{m+1}-1}\left\|\E(\epsilon_{p}^{*}\epsilon_{p}^{*T}\epsilon_{q}^{*}\epsilon_{q}^{*T}) - S^{2}\right\|_{2}.
 	$$
 	For $p>q$, we have 
 	\begin{equation}
 		\left\|\E(\epsilon_{p}^{*}\epsilon_{p}^{*T}\epsilon_{q}^{*}\epsilon_{q}^{*T}) - S^{2}\right\|_{2} = \left\|\E(\epsilon_{p}^{*}\epsilon_{p}^{*T})\E(\epsilon_{q}^{*}\epsilon_{q}^{*T}) - S^{2}\right\|_{2}= 0 .
 	\end{equation}
 	Therefore $	IV = 0$. 
 	Combining above results, we have
 	\begin{equation} 
 		\left\|n^{-2}II - S^{2}\right\|_{2}\lesssim n^{-2}III  + n^{-2}IV\lesssim M^{-1}.
 	\end{equation}  
 	Thus,
 	$$
 	\E\|\hat{S}^{*}_{n, NOL} - S\|_{2}\le \sqrt{d\E\|\hat{S}^{*2}_{n, NOL} - S^{2}\|_{2}}\lesssim \sqrt{n^{-2}\left\|I\right\|_{2} + \left\|n^{-2}II - S^{2}\right\|_{2}}\lesssim M^{-1/2}.
 	$$
 	
 	\bigskip
 	\noindent{\bf Step 2:} Bound $\E\|\hat{S}_{n, NOL} - \hat{S}^{*}_{n, NOL}\|_{2}$. 	\\ 
 	Let $v_{k} = \epsilon_{k} - \epsilon_{k}^{*}, k\ge 1$.  We can expand $\E\|\hat{S}_{n, NOL} - \hat{S}^{*}_{n, NOL}\|_{2}$ as 		
 	\begin{equation} 
 		\begin{split}
 			&\E\|\hat{S}_{n, NOL} - \hat{S}^{*}_{n, NOL}\|_{2} = \E\left\|n^{-1}\sum_{i \in SET_{n} }\left[\left(\sum_{k = t_{i}}^{i} \epsilon_{k}\right)\left(\sum_{k = t_{i}}^{i} \epsilon_{k}\right)^{T}-\left(\sum_{k = t_{i}}^{i} \epsilon_{k}^{*}\right)\left(\sum_{k = t_{i}}^{i} \epsilon_{k}^{*}\right)^{T}\right] \right\|_{2}\\
 			\le &2\E\left\|n^{-1}\sum_{i \in SET_{n} }\left(\sum_{k = t_{i}}^{i} v_{k}\right)\left(\sum_{k = t_{i}}^{i} \epsilon_{k}^{*}\right)^{T} \right\|_{2} + \E\left\|n^{-1}\sum_{i \in SET_{n} }\left(\sum_{k = t_{i}}^{i} v_{k}\right)\left(\sum_{k = t_{i}}^{i} v_{k}\right)^{T} \right\|_{2}.
 		\end{split}
 	\end{equation} 
 	Apply Cauchy's inequality   
 	\begin{equation} 
 		\E\left\|n^{-1}\sum_{i \in SET_{n} }\left(\sum_{k = t_{i}}^{i} v_{k}\right)\left(\sum_{k = t_{i}}^{i} \epsilon_{k}^{*}\right)^{T} \right\|_{2} \le  \sqrt{\E\|\hat{S}^{*}_{n, NOL}\|_{2}}\sqrt{\E\left\|n^{-1}\sum_{i \in SET_{n} }\left(\sum_{k = t_{i}}^{i} v_{k}\right)\left(\sum_{k = t_{i}}^{i} v_{k}\right)^{T} \right\|_{2}}. 
 	\end{equation}   
 	By triangle inequality and the fact $\|C\|_{2}\le\text{tr}(C)$ for any positive semi-definite matrix $C$,
 	\begin{equation} 
 		\begin{split}
 			&\E\left\|n^{-1}\sum_{i \in SET_{n} }\left(\sum_{k = t_{i}}^{i} v_{k}\right)\left(\sum_{k = t_{i}}^{i} v_{k}\right)^{T} \right\|_{2}
 			\le n^{-1}\sum_{i \in SET_{n} }\E\text{tr}\left(\left(\sum_{k=t_{i}}^{i}v_{k}\right)\left(\sum_{k=t_{i}}^{i}v_{k}\right)^{T} \right)\\
 			=&n^{-1}\sum_{i \in SET_{n} }\E\left\|\sum_{k=t_{i}}^{i}v_{k}\right\|_{2}^{2} \le n^{-1}\sum_{m=1}^{M}\sum_{k=a_{m}}^{a_{m+1}-1}\E\left\|v_{k}\right\|_{2}^{2}\lesssim n^{-1}\sum_{m=1}^{M}n_{m}(a_{m}-1)^{-\alpha}. 
 		\end{split}
 	\end{equation} 
 	The last  inequality comes from the fact $\E\|v_{k}\|_{2}^{2}\lesssim (k-1)^{-\alpha}$. Since $a_{m}\asymp m^{\beta}, n_{m}\asymp m^{\beta-1}$ and $n\asymp M_{\beta}$, 
 	we have 
 	$$
 	\E\left\|n^{-1}\sum_{i \in SET_{n} }\left(\sum_{k = t_{i}}^{i} v_{k}\right)\left(\sum_{k = t_{i}}^{i} v_{k}\right)^{T} \right\|_{2}\lesssim M^{-\alpha\beta}
 	$$	 
 	Then
 	$$
 	\E\|\hat{S}_{n, NOL} - \hat{S}^{*}_{n, NOL}\|_{2}\lesssim M^{-\a\beta/2}.
 	$$
 	
 	Finally,  we  reach the result
 	\begin{equation}
 		\begin{split}
 			\E\|\hat{S}_{n, NOL}-S\|_{2}\lesssim \E\|\hat{S}_{n, NOL}^{*} - S\|_{2} +  \E\|\hat{S}_{n, NOL} - \hat{S}^{*}_{n, NOL}\|_{2}\lesssim M^{-\a\beta/2} + M^{-1/2}.
 		\end{split}
 	\end{equation}
 \end{proof}

 \section{Proof of Proposition \ref{prop:mean}
 }\label{sec: prop}
 Without loss of generality, we assume $x^{*} = 0$. Then in the mean estimation model, the SGD ietrate $x_{i}$ takes the form
 \begin{equation}
 	x_{i} = (1-\eta_{i})x_{i-1}+\eta_{i}e_{i},
 \end{equation} 
 where $\eta_{i}=i^{-\alpha}, 1/2<\alpha<1$. And $e_{i}$ are i.i.d from $N(0, 1)$. Let $x_{0} = 0$, then 
 \begin{equation} 
 	x_{i} = \sum_{p=1}^{i}\prod_{k=p+1}^{i}(1-k^{\alpha})p^{-\alpha}e_{p}.
 \end{equation} 
 Let $W_{k} =\sum_{i=a_{k}}^{a_{k+1}-1}x_{i}$, $1\le k\le M-1$, and $W_{M} = \sum_{i=a_{M}}^{n}x_{i}$ where $M$ satisfies $a_{M}\le n < a_{M+1}$. We can rewrite the covariance of $\sqrt{n}\bar{x}_{n}$ as
 \begin{equation}
 	\text{Var}(\sqrt{n}\bar{x}_{n}) = \frac{\E(W_{1}+...+W_{M})^{2}}{n}.
 \end{equation}
 We can rewrite the estimator as 
 \begin{equation}
 	\hat\Sigma_{n, NOL} = \frac{W_{1}^{2} + W_{2}^{2} + ... + W_{M}^{2}}{n}.
 \end{equation}
 For simplicity, we ignore the $\bar{x}_{n}$ term in the estimator since $\bar{x}_{n}$ converge to $0$ at rate of $O(n^{-1/2})$, which is much faster than the convergence rate of the variance estimator.
 Then 
 \begin{equation}
 	n\text{Bias}(\hat\Sigma_{n, NOL}) = 2\sum_{1\le f < g\le M}\text{Cov}(W_{f}, W_{g}),
 \end{equation}
 and 
 \begin{equation}
 	n^{2}\text{Var}(\hat\Sigma_{n, NOL}) = \sum_{f=1}^{M}\text{Var}(W_{f}^{2}) + 2\sum_{1\le f < g\le M}\text{Cov}(W_{f}^{2}, W_{g}^{2}).
 \end{equation}
 Next, we shall approximate $\text{Var}(W_{f})$ and $\text{Cov}(W_{f}, W_{g})$, $f<g$.
 \begin{equation}\label{eq:var_W_f}
 	\begin{split}
 		\text{Var}(W_{f})& = \sum_{i=a_{f}}^{a_{f+1}-1}\text{Var}(x_{i}) + \sum_{j=a_{f}+1}^{a_{f+1}-1}\sum_{i = a_{f}}^{j-1}\text{Cov}(x_{i}, x_{j})\\
 		& \overset{*}{\asymp} \sum_{i=a_{f}}^{a_{f+1}-1}i^{-\a} + \sum_{j=a_{f}+1}^{a_{f+1}-1}j^{-\a}(1-j^{-\a})\sum_{k=0}^{j-a_{f}-1}(1-j^{-\a})^{k}\\
 		& = \sum_{i=a_{f}}^{a_{f+1}-1}i^{-\a} + \sum_{j=a_{f}+1}^{a_{f+1}-1}(1-j^{-\a})\\
 		& = a_{f+1}-1-a_{f}  .
 	\end{split}
 \end{equation}
 The second line $*$ in \eqref{eq:var_W_f} follows from some simple calculations with $\text{Var}(x_{i})\asymp i^{-\a}$ and $\text{Cov}(x_{i}, x_{j})\asymp j^{-\a}(1-j^{-\a})^{j-i}$ for $ i < j $. Then, 
 \begin{equation}
 	\begin{split}
 		\text{Cov}(W_{f}, W_{g})& =  \sum_{i=a_{f}}^{a_{f+1}-1}\sum_{j = a_{g}}^{a_{g+1}-1}\text{Cov}(x_{i}, x_{j})\\
 		& \asymp \sum_{i=a_{f}}^{a_{f+1}-1}a_{g}^{-\a}\sum_{j = a_{g}}^{a_{g+1}-1}(1-a_{g}^{-\a})^{j-i}\\
 		& = \sum_{i=a_{f}}^{a_{f+1}-1}a_{g}^{-\a}(1-a_{g}^{-\a})^{a_{g}-i}\sum_{l = 0}^{a_{g+1}-1-i}(1-a_{g}^{-\a})^{l}\\
 		& = \sum_{i=a_{f}}^{a_{f+1}-1}(1-a_{g}^{-\a})^{a_{g}-i}  = (1-a_{g}^{-\a})^{a_{g}-a_{f+1}+1}/a_{g+1}^{-a}.
 	\end{split}
 \end{equation}
 Since $W_{f}$ is normal, we have $\text{Var}(W_{f}^{2}) = 2\text{Var}(W_{f})^{2}$ and $\text{Cov}(W_{f}^{2}, W_{g}^{2}) = 2\text{Cov}(W_{f}, W_{g})^{2}$. Then, 
 \begin{equation}
 	\begin{split}
 		n\text{Bias}(\hat\Sigma_{n, NOL}) &\asymp \sum_{g = 2}^{m}\frac{1 - a_{g+1}^{-\a}}{a_{g+1}^{-\a}}\sum_{f=1}^{g-1}(1 - a_{g+1}^{-\a})^{a_{g} - a_{f+1}}\\
 		& = \sum_{g = 2}^{m}\frac{1 - a_{g+1}^{-\a}}{a_{g+1}^{-\a}}O(1) \asymp \sum_{g = 2}^{m}\frac{1 - g^{-\a\beta}}{g^{-\a\beta}} \asymp m^{\a\beta + 1} \asymp n^{\a + 1/\beta}.
 	\end{split}
 \end{equation}
 Also the variance  
 \begin{equation}
 	\begin{split}
 		n^{2}\text{Var}(\hat\Sigma_{n, NOL}) &= \sum_{f=1}^{M}(a_{f+1}-a_{f}-1)^{2} + 2\sum_{1\le f < g\le M}(1-a_{g}^{-\a})^{a_{g}-a_{f+1}+1}/a_{g+1}^{-a}\\
 		&\asymp \sum_{f=1}^{n^{1/\beta}}f^{2\beta - 2}	\asymp n^{2 - 1/\beta}. 
 	\end{split}
 \end{equation}
 Then we have the mean squared error 
 \begin{equation}
 	MSE(\hat\Sigma_{n, NOL}) = \text{Bias}^{2}(\hat\Sigma_{n, NOL})  + \text{Var}(\hat\Sigma_{n, NOL}) \asymp n^{-1/\beta} + n^{2\a + 2/\beta - 2}.
 \end{equation}

 \section{Simulation for stopping rule}\label{sec:stopping_rule_simulation} 
 In this section, we include a simple simulation study applying the fixed-width sequential stopping rule. We set the tolerance $\epsilon_{i} = 0.01$ for $i = 1, ..., d$. The rule is applied to our online approach SGD inference procedure for both linear and logistic regressions with same settings discussed in Section \ref{sec:exper}. We present termination iterations and coverage probabilities at termination in Table \ref{table:stopping_rule}.
 
 \begin{table}[!t]
 	\centering	
 	\caption{Apply fixed-width sequential stopping rule with  the tolerance $0.01$ (discussed in Section \ref{sec:stopping_rule}). We present termination iterations and coverage probabilities at termination. Standard errors are reported in the brackets.}
 	\begin{tabular}{  l | c   c | c c  }
 		\hline
 		&\multicolumn{2}{c}{Linear}&\multicolumn{2}{c}{Logistic}\\
 		\hline	
 		& $d = 5$ &  $d = 20$  & $d = 5$ &  $d = 20$ \\
 		\hline
 		Termination iteration& 47,737 (13,594) & 98,644 (52,424) & 249,962 (63,507) & 446,016 (84,910)\\
 		Coverage probabilities& 0.881 (0.022) & 0.906 (0.020) & 0.865 (0.024) & 0.843 (0.025)\\ 
 		\hline 	    
 	\end{tabular}
 	\label{table:stopping_rule}
 \end{table}

\bibliographystyle{apalike}

\bibliography{reference}
\end{document}